\documentclass[12pt]{article}

\usepackage{arxiv}

\usepackage{graphicx} 
\usepackage{microtype}
\usepackage{graphicx}
\usepackage{thmtools} 
\usepackage{titletoc}
\usepackage{thm-restate}
\usepackage{booktabs} 
\usepackage{hyperref}

\usepackage[page,header]{appendix}
\usepackage{titletoc}
\usepackage[nottoc]{tocbibind}
\usepackage{amssymb}
\usepackage[mathscr]{euscript}
\usepackage{amsthm}
\usepackage{amsmath}
\usepackage{cleveref}
\usepackage{subdepth}
\usepackage{amsmath}
\usepackage[english]{babel}
\usepackage{graphicx}
\usepackage{booktabs}
\usepackage{algorithm}
\usepackage{algpseudocode}
\usepackage[numbers]{natbib}
\usepackage{hyperref}
\usepackage{wrapfig}
\usepackage{tikz}
\usepackage{float}
\usepackage{parskip}

\usepackage{amsmath}

\newcommand{\Ave} {{\rm Ave}}
\newcommand{\I} {{\cal I}}
\newcommand{\E} {{\mathbb E}}

\newcommand{\vvarphi} {\varphi}
\newcommand{\data} {\varphi}

\newcommand{\Z} {{\mathbb Z}}
\newcommand{\R} {{\mathbb R}}
\newcommand{\En} {{U}}

\newcommand{\N} {{\mathbb N}}

\newcommand{\om} {{\omega}}

\newcommand  \aaa {{\alpha}}

\newcommand\Ld{{\bf L^2}}

\newcommand\trans{^{\mathrm{T}}}
\newcommand\Id{\mathrm{Id}}
\newcommand\EE{\mathbb{E}}

\newcommand\expect[2][]{\EE\ifstrempty{#1}{}{_{#1}}\big (#2\big)}

\newcommand\norm[1]{\left\lVert #1 \right\rVert}

\newcommand\paren[1]{\left( #1 \right)}

\DeclareMathOperator{\diag}{diag}
\DeclareMathOperator{\Tr}{Tr}

\def\cistar{\kern.2em\mbox{$\odot\kern-.67em\star\kern .4em$}}

\newcommand {\Hess} {\nabla^2}

\title{Hierarchic Flows to Estimate and Sample High-dimensional Probabilities
}

\newif\ifuniqueAffiliation

\ifuniqueAffiliation 
\else
\usepackage{authblk}

\author[1]{Etienne Lempereur}%

\author[2,3]{Stéphane Mallat}%
\affil[1]{Département d'informatique, Ecole normale supérieure, Paris, France}
\affil[2]{Collège de France, Paris, France}
\affil[3]{Flatiron Institute, New York, USA}
\fi

\hypersetup{
pdftitle={Hierarchic Flows to Estimate and Sample High-dimensional
Probabilities},
pdfauthor={Etienne ~Lempereur, St{\'e}phane ~Mallat},
}

\begin{document}

\newtheorem{theorem}{Theorem}[section]
\newtheorem{lemma}{Lemma}[section]
\newtheorem{proposition}{Proposition}[section]
\newtheorem{definition}{Definition}[section]
\newtheorem{corollary}{Corollary}[section]
\newtheorem{condition}{Condition}[section]
\newtheorem{assumption}{Assumption}[section]
\newtheorem{remark}{Remark}[section]
\newcommand{\z} {{z}}
\newcommand{\pii} {{p}}

\maketitle
\begin{abstract}
Finding low-dimensional interpretable models of complex physical fields such as
turbulence remains an open question, 80 years after the pioneer work of Kolmogorov.
Estimating high-dimensional probability distributions from data samples 
suffers from an optimization and an approximation curse of dimensionality.
It may be avoided by following a hierarchic probability flow from coarse to fine scales. 
This inverse renormalization group is defined by conditional probabilities across scales, renormalized in a wavelet basis. 
For a $\vvarphi^4$ scalar potential, sampling these hierarchic models avoids the critical slowing down at the phase transition.  In a well chosen wavelet basis, conditional probabilities can be captured with low dimensional parametric models, because interactions between wavelet coefficients are local in space and scales.
An outstanding issue is also to approximate non-Gaussian fields having long-range  interactions in space and across scales.
We introduce low-dimensional models of wavelet conditional probabilities
with the scattering covariance.
It is calculated with a second wavelet transform, which defines interactions
over two hierarchies of scales. We estimate
and sample these wavelet scattering models to generate
2D vorticity fields of turbulence, and images of dark matter densities.
\end{abstract}



\section{Introduction}

Estimating models of high-dimensional probability distributions from data is at the heart  of data science and statistical physics. 
For a physical system at equilibrium, the probability distribution of a field 
$\data \in \R^d$ (such as an image) has a density
$p(\data) = {\cal Z}^{-1}\, e^{-\En(\data)}$, where $\En(\data)$ is the Gibbs energy \citep{landau2013statistical}. Learning means approximating and optimizing
an estimation of the high-dimensional energy function $\En$, from $m$ data samples $\{\data^{(i)} \}_{i\leq m}$ resulting from measurements or numerical simulations.
New data can then be generated by sampling the estimated model of $p$, which is also used to estimate solutions
of inverse problems \cite{kaipio2006statistical,aster2018parameter}.
The estimation of a Gibbs energy is particularly difficult when $\data$ has long range dependencies, and its dimension $d$ is large. An outstanding problem is to build probabilistic models of turbulent flows, which dates back to the work of Kolmogorov in 1942 \citep{kolmogorov1941local,kolmogorov1942equations}. 

In statistics, $p$ is estimated by defining an
approximation class $p_\theta$ and by optimizing $\theta$. These approximation and optimization problems are plagued by the curse of dimensionality.
Section \ref{sec:energyestim-sampling} reviews both aspects. It includes
linear approximations of Gibbs energies, 
maximum likelihood estimation versus score matching, and sampling by Langevin diffusions. The optimization curse of dimensionality over $\theta$ is avoided if the log-Sobolev constant of $p$ remains bounded when $d$ increases \citep{gross1975logarithmic,ledoux2000geometry,bakry2014analysis}.
The approximation curse of dimensionality is also avoided if we can define accurate models with a parameter vector $\theta$ whose dimension is bounded or has a slow growth when $d$ increases.
This is possible if interactions are local within $\data$, as in
Markov random fields \cite{geman1984stochastic}. 
Sadly enough, neither of these
two properties is satisfied by complex data such as turbulence fields.
For multiscale fields,
the renormalization group gives a powerful hierarchic approach to
address these approximation and optimization curse of dimensionality.

From a cybernetics perspective, Herbert Simons observed that the architecture of most complex systems is hierarchic, in physics, biology, economic, symbolic languages, or social organizations. He argues that it probably results from their dynamic evolution, where intermediate states 
must be stable \cite{simon1962architecture}. This attractive idea could explain why 
the curse of dimensionality can be avoided when analyzing such systems, but
the notion of hierarchy is loosely defined. The core principle of hierarchic organizations is to build long range interactions from a limited number of local interactions between neighbors in the hierarchy. But, large systems include multiple hierarchies that produce complex long-range interactions in multidimensional matrix organizations, as opposed to a tree. For example, large corporations often include 
horizontal hierarchic organizations dedicated to specific projects, at each vertical hierarchic level
 \cite{turkMatrixOrganisation}. It creates long-range interactions between employees working on the same project.
In physics, the renormalization group theory provides a hierarchic analysis of multiple body interactions.
It computes a flow of probabilities from a fine microscopic scale towards larger macroscopic scales \citep{kadanoff1966scaling,wilson1971renormalization,wilson1972critical,Delamotte_2012}.
 
However, major difficulties have been encountered to take into account non-local interactions in space and across scales. Particularly for non-renormalizable systems such as turbulence, whose degrees of freedom increase with the dimension $d$ of $\data$ \cite{Bohr_Jensen_Paladin_Vulpiani_1998}.

This paper defines hierarchic models to estimate and sample high-dimensional 
non-Gaussian processes having non-local interactions. Section \ref{sec:prob-sep} considers data 
$\data \in \R^d$ defined on graphs, or images. A first hierarchic organization is constructed with coarse graining approximations $\data_j$ of $\data$,
of progressively smaller sizes as the scale $2^j$ increases. 
Figure \ref{fig:intro} gives an illustration of the vorticity field of a 2D turbulence.
The probability density $p(\data)$ is progressively mapped into
densities $p_j (\data_j)$ from fine to coarse scales $2^j$.
This renormalization group transformation is computed by marginal integrations of the
high frequency degrees of freedom, which progressively disappear as $j$ increases.
There is no difficulty to estimate and sample $p_J (\varphi_J)$ if $\data_J$ has a low dimension.
From this estimation, 
the high-dimensional model of $p$ can be 
estimated and sampled with a reverse
Markov chain. It transforms $p_J$ into $p$ by iteratively
computing $p_{j-1}$ from $p_j$, as shown
by Figure \ref{fig:intro}.
Each transition kernel $\bar p_j$ of this hierarchic flow 
is the conditional probability of $\data_{j-1}$
given $\data_j$.  The main difficulty is to understand on what conditions one can estimate and sample these transition kernels,
without suffering from the curse of dimensionality.

\begin{figure}
\label{fig:intro}
\centering
\includegraphics[width=0.25\textwidth]{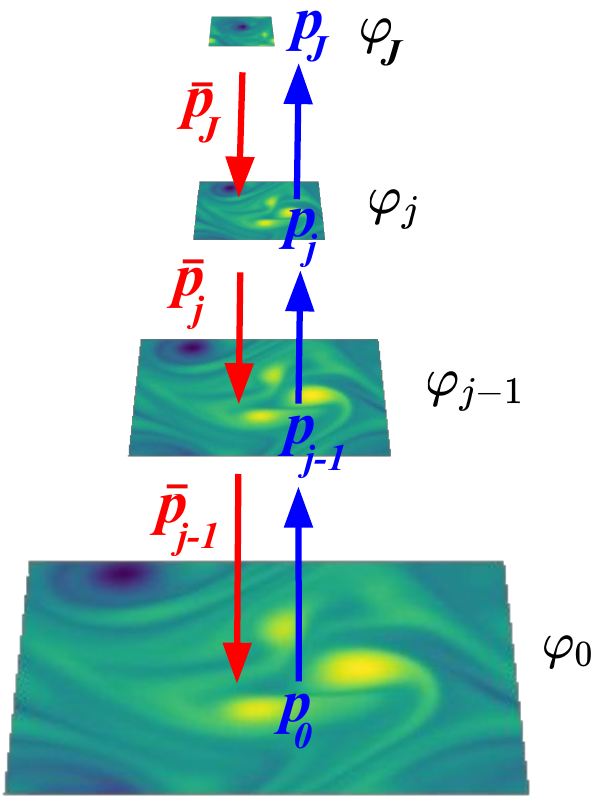}
\caption{The renormalization group \textit{(illustrated in blue)} computes the probability distributions $p_j (\data_j)$ of images $\data_j$ at progressively larger scales $2^j$, with marginal integrations of high-frequency degrees of freedom, up to a maximum scale $2^J$. A hierarchic flow \textit{(illustrated in red)} is a top-down inverse renormalization group which recovers $p$ from $p_J$ by estimating each transition probability $\bar p_j$ from $p_j$ to $p_{j-1}$ \cite{marchand_wavelet_2022}. Difficulties arise when $\bar p_j$ is non-local, as in turbulent flows.} 
\end{figure}

A hierarchic probability flow across scales is
an inverse to Wilson's renormalization group. If we represent high frequencies in wavelet
bases then the transition kernels can be written as conditional probabilities of wavelet coefficients \cite{marchand_wavelet_2022}.
Renormalizing wavelet coefficients
is a strategy to control log-Sobolev constants of transition probabilities.
For the $\vvarphi^4$ model of ferromagnetism at phase transition, it was shown in
\cite{marchand_wavelet_2022} that a renormalization in a wavelet basis 
eliminates the "critical slowing down" of Langevin
sampling algorithms \citep{10.1093/oso/9780198834625.001.0001}.
This is verified and analyzed for different types of 
wavelet bases in Section \ref{sec:scal-pot-ene}. The critical slowing down is beaten thanks to the localization of wavelets in the Fourier domain.
Moreover, since wavelets have a localized spatial support,
the conditional probabilities $\bar p_j(\bar\data_j |\data_j)$ can be approximated with a local scalar potential, which defines a low-dimensional parametric
model. The choice of wavelet is thus a trade-off between spatial
and frequency localizations, which affect the log-Sobolev constant and the
model dimension.

The energy of $\vvarphi^4$ is local in space and "renormalizable," which means that it can be approximated with a number of coupling parameters, which does not depend on the field dimension $d$, at all scales.
This property does not apply to
complex systems such as fluid turbulence, which have progressively more degrees of
freedom as the dimension $d$ increases \citep{Bohr_Jensen_Paladin_Vulpiani_1998}.
To face this issue, we introduce hierarchic potential models, whose dimensions increase when the scale $2^j$ decreases. The hierarchy preserves a coupling flow equation, which relates the coupling parameters of energies from one scale to the next.

In physics and statistics, Gibbs energies of non-Gaussian probability distributions are often approximated with polynomials of degrees larger than $2$, typically $3$ or $4$ \cite{landau2013statistical}.
For a stationary field of dimension $d$, it involves $d^2$ or $d^3$ approximation terms, whose estimations have a large variance. Moreover, for multiscale fields the parametrization of Gibbs
energies is often unstable, leading to brutal phase transitions as in the $\varphi^4$ model.
Section \ref{sec:scat-cov} introduces stable hierarchic parameterizations, with multiscale
interaction energy models of dimension $O(\log^3 d)$. High order polynomials are replaced
by scattering covariance coefficients. 
They are computed with a second wavelet transform, applied to 
the modulus
of the first wavelet transform. It defines a second hierarchy, with a second scale parameter. The resulting scattering coefficients \cite{mallatscat}
capture long-range non-Gaussian interactions across space and scales  
\cite{morel2023scale,cheng2023scattering}.
It defines a renormalization group representation of non-renormalizable systems,
with $O(\log^3 d)$ degrees of freedom. Gibbs energy models of two-dimensional turbulent vorticity fields and dark matter density fields
are estimated and sampled with this hierarchic approach.

\section{Probability Models, Estimation and Sampling}
\label{sec:energyestim-sampling}

We denote by $\data \in \R^d$ a data vector of dimension $d$.
We suppose that it has a probability density $p(\data) = {\cal Z}^{-1} e^{-\En(\data)}$ with respect to the Lebesgue measure, with
a Gibbs energy $U$ which is twice differentiable. 
Our goal is to estimate an accurate model of $p$ from 
$m$ samples $\{ \data^{(i)} \}_{i \leq m}$, and to generate new data by sampling this model. 
This section is a brief introduction to
sampling of high-dimensional probabilities with Langevin diffusions,
and to the estimation of parametric models 
by maximum likelihood and score matching.

\subsection{Langevin Sampling and Log Sobolev Inequalities}
\label{sec:langevin-sampling-logSobolev}

If $\En(\data)$ is known, 
one can sample $p$ with the Langevin dynamics, which is a Markov process that iteratively updates a field $\data_t$ with the stochastic differential equation
\[
d \data_t = - \nabla_\data \En(\data_t)\, dt + \sqrt 2 \, dB_t~,
\]
where $B_t$ is a Brownian motion. It is a gradient descent on the energy, which is perturbed by the addition of a Gaussian white noise. Let $\data_0$ be a sample of a density $p_0$ at time $t=0$. At time $t$, $\data_t$ is a sample of a density $p_t$ which is guaranteed to converge to the density $p$ of Gibbs energy $\En(\data)$ \citep{Lelièvre_Stoltz_2016}. The unique invariant measure of this Markov chain is $p$. We can thus sample $p$ by running the Langevin equation over samples of an initial measure $p_0$, for example Gaussian, but the convergence may be extremely slow.

The convergence of $p_t$ towards $p$ is defined with a KL divergence
\[
KL(p_t , p) = \int p_t (\data) \, \log \frac{p_t(\data)} {p(\data)}\, d\data \geq 0 .
\]
De Bruin identity \cite{bauerschmidt2023stochastic} proves that
\begin{equation}
\label{deBruin}
\frac{d KL(p_t , p)} {dt} = - \I (p_t , p) ,
\end{equation}
where $\I$ is the relative Fisher information
\[
\I(q , p) = \int q(\data)\, \|\nabla_\data \log \frac{ q(\data)} {p(\data)}\|^2\, d \data .
\]
The exponential convergence of $p_t$ towards $p$ is guaranteed 
if $p$ satisfies a log-Sobolev inequality, which is recalled.

\begin{definition}
The log-Sobolev constant $c(p)$ of $p$ is the smallest constant so that for any probability density $q$
\begin{equation}
\label{log-sobolev}
KL(q , p) \leq c(p)\, {\cal I} (q , p) .
\end{equation}
\end{definition}

Log-Sobolev constants relate entropy and gradients of smooth normalized functions $f(\data)$ in
the functional space ${\bf L^2}(p \,d\data)$. Indeed,  $f^2 = {q}/ p$ satisfies
$\|f\|_2 = 1$ in  ${\bf L^2}(p \,d\data)$ because $\int q(\data) d \data = 1$, 
and one can verify that the log-Sobolev inequality (\ref{log-sobolev}) is equivalent to 
\begin{equation}
\label{log-sobolev2}
\int f^2(\data)\, \log f^2(\data) \,p(\data) d\data \leq 4\, c(p) \int \|\nabla f(\data) \|^2 p(\data)d\data .
\end{equation}
If $p$ is a normal Gaussian then $c(p) = 1/2.$

The log-Sobolev constant gives an exponential rate of convergence of a Langevin diffusion. Indeed, inserting the log-Sobolev definition (\ref{log-sobolev}) in (\ref{deBruin}) proves that
\begin{equation}
\label{exp-dec}
KL(p_t , p) \leq e^{ {-t} /{c(p)}} KL(p_0 , p)~.
\end{equation}
The time it takes for Langevin diffusions to reach a fixed precision is thus at most
proportional to the log-Sobolev constant $c(p)$. 
The trouble is that $c(p)$ typically grows exponentially with the dimension $d$ of $\data$.

Upper bounds of 
the log-Sobolev constant can be computed in two cases \citep{bauerschmidt2023stochastic}. Under independence conditions,
if $\data = (\data_k)_k$ and $p$ can be separated as a tensor product of densities $p(\data) = \prod_k p_k (\data_k)$ then 
\begin{equation}
\label{product}
c(p) = \max_k \big(c(p_k) \big)_k .
\end{equation}
In other words, the log-Sobolev constant of independent random variables is the maximum log-Sobolev constant of each variable. 
The second case, when $\En$ is strongly convex,
is the Bakry-Emery theorem \citep{bakry2014analysis}, which proves that
\begin{equation}
\label{lower-norm}
\forall \data~~,~~\Hess_\data \En(\data) \geq \aaa Id~~\Rightarrow~c(p) \leq \frac 1 {2 \aaa} .
\end{equation}
The maximum constant $\aaa$ satisfying this equality is the infimum over $\data$ of all eigenvalues  of the Hessian matrices $\Hess_\data \En(\data)$.
The following proposition gives
a lower bound of the log-Sobolev constant from the covariance of $p$.

\begin{proposition}
 \label{prop:covar-lower}
Let 
$\mu_{\max}$ be the largest
eigenvalue of the covariance of $p$. 
The log-Sobolev constant satisfies
\begin{equation}
\label{sbol-covariance}
c(p) \geq \mu_{\max} / 2~.
\end{equation}
\end{proposition}
This proposition is proved in Appendix \ref{app:pointcarre}.
It comes from the inequality between log-Sobolev and Poincaré constants 
\cite{ledoux2000geometry}.

The Langevin diffusion is numerically calculated with the Euler-Maruyama discretization. If $\nabla_\varphi \En(\varphi)$ is uniformly $L-$Lipschitz, then
the discretization time step can be smaller or equal to
$L^{-1}$ \cite{vempala2022rapid}.
Since $L$  is the supremum of all eigenvalues of the
Hessians $\nabla^2_\varphi \En(\varphi)$,
it results from (\ref{exp-dec}) that the number of 
Langevin diffusion steps to achieve a given precision is proportional to the log-Sobolev constant multiplied by this eigenvalues' supremum.  It is a normalization of the log-Sobolev constant, which specifies the computational complexity of the Langevin sampling algorithm.

Numerically, the convergence rate of a Langevin diffusion is estimated from 
the relaxation time $\tau$ of the auto-correlation
\begin{equation}
\label{relax-time}
A(t) =\expect[p_{t_0}]{\big(\data_{t+t_0}-\E[\data_{t+t_0}]\big)\trans\big(\data_{t_0}-\E[\data_{t_0}]\big)} \propto  e^{-\frac{t}{\tau}} ,
\end{equation}
for $t_0$ big enough \citep{Sokal:1990tr,sokal1997monte}. 
In high dimension $d$, we cannot compute numerically the Kulbach Liebler divergence along the Langevin dynamic, from which we could derive a lower bound on the log-Sobolev constant, using equation (\ref{exp-dec}). A quantitative numerical estimation of the log-Sobolev constant is therefore out of reach. Instead of measuring the exponential decay of the  Kulbach Liebler divergence, we estimate the exponential decay of the autocorrelation.
For a gaussian distribution, $\tau$ is equal to the largest eigenvalue of the covariance of $p$, which is its log-Sobolev constant. For more general distributions $p$, this relaxation time, like the log-Sobolev constant, reflects the multi-modality of $p$ and the conditioning of its covariance.  The relaxation time $\tau$ gives only a partial information on the log-Sobolev constant.

To eliminate the bias introduced by the discretization step of Langevin diffusion, syntheses are generated using a Metropolis-Adjusted Langevin Algorithm (MALA) \cite{grenander1994representations}. It iterates over discretized Langevin dynamics proposals that are accepted or rejected using the Metropolis-Hastings algorithm.

\subsection{Approximation and Learning Gibbs Energies}
\label{sec:estim-gibbs}

The Gibbs energy $\En$ of a probability density $p$ is usually not known a priori.
It is approximated by a parametrized model $p_{\theta} = {\cal Z}_\theta^{-1} e^{-\En_{\theta}}$
where $\theta$ is optimized by minimizing $KL(p,p_\theta)$. Learning $\theta$ 
from a dataset of $m$ samples $\{ \data^{(i)} \}_{i \leq m}$ is 
usually done with a
maximum likelihood estimation. We shall see that it can be replaced by score matching estimations, which require much fewer calculations, but whose precision depends upon the log-Sobolev constant $c(p)$.

\paragraph{Exponential models}
We concentrate on linear approximations $\En_\theta$ of the Gibbs energy $\En$,
over a family of $m'$  functions $\Phi = \{\phi_k  \}_{k \leq m'}$
weighted by $\theta = \{\theta_k \}_{k \leq m'}$
\begin{equation}
\label{linear-decomp}
\En_\theta(\data)  =  \theta\trans \Phi(\data)  = \sum_{k \leq m'} \theta_k \, \phi_k (\data) .
\end{equation}
It defines an exponential family of probability densities
\[
p_\theta(\data) = {\cal Z}_\theta^{-1}\, e^{-  \theta\trans  \Phi(\data) } .
\]
A major issue is to find a family $\Phi$ which provides an accurate approximation $\En_\theta$ of $U$ in high dimension $d$, despite the curse of dimensionality.
In statistics, each $\phi_k$ is called a moment generating function because the maximum likelihood estimators of $\theta$ is specified by
the vector of moments $\E_p [\Phi]$ as we shall see.
The $\phi_k$ can be polynomials of $\data$.
However, if we use all high order polynomials of a fixed degree, then the dimension of $\Phi$ has a polynomial growth in $d$. The resulting estimation has a high variance, which
requires a large number $m$ of training samples \cite{higher-order}.
Markov random fields \cite{geman1984stochastic} provide a powerful framework where all interactions are supposed to be local over native variables, which defines low-dimensional models. However, it 
does not apply to data having long-range interactions, which is the case for many image textures \cite{portilla2000parametric}. 
In physics, $\En_\theta$ is called an Ansatz. The $\phi_k$ are considered as interaction potentials. Physical models  often rely on multilinear functions of derivatives of $\data$. The resulting Ansätze are usually local and low-dimensional, but are unable to capture complex systems such as turbulence.

\paragraph{Maximum Likelihood Estimation}
The next step is to optimize $\theta$ in order to best approximate $p$ by $p_\theta$.
The maximum likelihood estimator $\theta$ maximizes $\E_p (\log p_\theta)$, which
is equivalent to minimize the KL-divergence $KL(p , p_\theta)$.
If $\En_\theta(\data) =  \theta\trans \Phi (\data) $ then
this KL divergence is  a convex function of $\theta$. The maximum likelihood
estimator can thus
be calculated with a gradient-descent algorithm of step size $\epsilon$, which computes
\begin{equation}
\label{likelihood-grad}
\theta_{k+1} - \theta_k = \epsilon \Big(
\E_{p_{ \theta_k}}(\Phi) - \E_{p}(\Phi) \Big) .
\end{equation}
The maximum likelihood parameter $\theta$
satisfies $\E_{p_{\theta}}(\Phi) = \E_{p}(\Phi)$. The resulting $p_\theta$ 
is called a moment projection of $p$ \cite{bishop2006pattern}.
If the potentials of $\Phi$ are linearly independent, then it is uniquely defined if
it exists. One can verify that this moment projection is the
distribution $p_\theta$ which maximizes
the entropy ${\cal H}(p_\theta) =-\int p_\theta(\data)\,\log p_\theta(\data) \, d \data$ 
subject to the moment condition $\E_{p_{\theta}}(\Phi) = \E_{p}(\Phi)$ 
\cite{jaynes1957information}.

The expected value $\E_{p}(\Phi)$ is estimated from the $m$ samples $\data^{(i)}$ with a Monte Carlo sum $m^{-1} \sum_i \Phi(\data^{(i)}).$
The estimation of $\E_{p_{\theta_k}}(\Phi)$ with a Monte Carlo sum requires computing enough samples of $p_{\theta_k}$.
It can be done with a Langevin diffusion, but
it requires a considerable amount of computations since it must be run for a sufficiently long time so that the Langevin algorithm converges, and it must be repeated enough so that Monte Carlo sum converges. This is unfeasible if the Langevin convergence rate is too low.

\paragraph{Score Matching Estimation}
Optimizing  $p_{\theta} = {\cal Z}_\theta^{-1} e^{-\En_\theta}$ by minimizing the  $\mathrm{KL}$ divergence is computationally expensive because the gradient descent on $\theta$ includes the term $\nabla_{\theta} \log Z_{\theta} = -\expect[\pii_{\theta}]{\Phi(\data)}$.
Score matching \citep{hyvarinen2005estimation} is an 
appealing alternative. It eliminates the normalization constant ${\cal Z}_\theta$ by replacing the $\mathrm{KL}$  minimization by a minimization of the relative Fisher information, which depends on $\nabla_\data \log p$
\begin{align*}
    \I(p , p_\theta)
    &= \expect[p]{\frac12\| {\nabla_{\data} \log p_\theta (\data)} - \nabla_{\data}
      \log p (\data)\|^2}.    
\end{align*}
With an integration by parts, \citep{hyvarinen2005estimation} proves that it is equivalent to minimize a quadratic loss in $\theta$:
\begin{align}
\nonumber
    \ell(\theta) & = \expect[p]{\frac12\|\nabla_{{\data}} \En_{\theta} \|^2 - \Delta_{{\data}} \En_{\theta}} \\
    \label{eq:loss-score}
    & =
    \expect[p]{\frac12\| \theta\trans \nabla_\data \Phi(\data)  \|^2 -  \theta\trans \Delta_{{\data}}\Phi(\data)  } ~,
\end{align}
whose solution is obtained without sampling $p_\theta$
\begin{equation}
  \label{theta-calcul}
  \theta =  M^{-1}
  \expect[p]{ \Delta_{{\data}}\Phi(\data)  } ~~\mbox{with}~~
  M =  \expect[p]{ \nabla_\data \Phi(\data) \nabla_\data \Phi(\data)\trans }.
\end{equation}
Minimizing the score matching loss is thus much faster than a minimization of the KL divergence.
Expected values are estimated by en empirical sum over the $m$ samples
$\data^{(i)}$ of $p$. Estimation errors introduced by the inversion of $M$ must often be regularized,
which is done by adding $\epsilon \Id$ to $M$, where $\epsilon$ decreases like $m^{-1}$ for an estimation with $m$ samples.

Score matching is a consistent estimator, which means that if there exists a unique $\theta^*$ such that $p = p_{\theta^*} > 0$ then when $m$ goes to infinity, the minimizer $\theta$ of the score matching loss converges to $\theta^*$ \citep{hyvarinen2005estimation}. However, 
the relative precision of score matching relatively to the maximum likelihood estimation depends on if the KL divergence can be controlled by the relative Fisher information. This is captured by the log-Sobolev constant $c(p)$ defined in (\ref{log-sobolev}).
For exponential models $p_\theta$, Theorem 2 in \cite{koehler2022statistical} bounds the covariance of the score matching estimation  with the covariance of the maximum likelihood estimator multiplied by $c(p)^2$.  The bound involves a multiplicative constant which also depends upon the
regularity of $U_\theta$. Its amplitude can be approximated by the largest eigenvalue squared of the Hessian of $U_\theta$. 
Similarly to the Langevin relaxation time, 
the log-Sobolev constant needs to be normalized by this largest eigenvalue amplitude. 
A score matching thus achieves a comparable accuracy as a maximum likelihood estimator if the number $m$ of samples
is multiplied by this normalized log-Sobolev constant, which may be very large.

\paragraph{Multiscale Gaussian and non-Gaussian densities}  
If $p = {\cal Z}^{-1} e^{-U}$ is a zero-mean gaussian distribution, then its Gibbs energy is quadratic
\begin{equation}
\label{Gauss-energy}
\En(\data) =  \frac 1 2  \data\trans K \data ~~\Rightarrow~~~\Hess_\data \En(\data) = K \geq 0.
\end{equation}
The matrix $K$ is symmetric positive and its inverse $K^{-1} = C$ is the covariance matrix of $p$. The variance of $\data$ is normalized by imposing that ${\rm Trace}(C) = d.$
Let $\mu_{\max}$ be the largest eigenvalue of $C$.
The Bakry-Emery upper bound of the log-Sobolev constant in (\ref{lower-norm}) together with 
the lower bounds (\ref{sbol-covariance}) proves that 
\begin{equation}
\label{Gaussian-logSoblev}
c(p) =  \frac 1 2 \, \mu_{\max}  .
\end{equation}
We explained that the number of steps of a Langevin diffusion is proportional to a
normalized log-Sobolev constant, which is multiplied by the supremum of the eigenvalues of $\nabla_\varphi^2 U = K$. 
It is thus divided by the smallest eigenvalue $\mu_{\min}$ of the covariance $C = K^{-1}$, 
and hence proportional to the condition number $\mu_{\max}/\mu_{\min}$ of $C$. 
The inefficiency of score matching relative to maximum likelihood estimation also
depends on this normalized log-Sobolev constant and hence on the condition number of $C$.

If $p$ is stationary or has stationary increments, then $C$ is 
diagonalized in the Fourier basis. Multiscale fields are defined as random
fields having a power spectrum (covariance eigenvalues) 
which is of the order of $|\omega|^{-\eta}$ at each frequency $\om \neq 0$,
for some $\eta > 0$.
This is true for natural images \cite{Burton1987ColorAS,PhysRevLett.73.814} or physical fields such as $\varphi^4$. Numerically, $\eta$ can be estimated by a regression over the eigenvalues of the empirical covariance.
For example, if $K = - \Delta$ then $\data$ is a Brownian motion and $\eta = 2.$
The growth of eigenvalues at low frequencies means that
$\data$ has long range correlations.
In two dimensions, $\sqrt{2}\pi \geq |\om| \geq \pi d^{-1/2}$ so
$c(p) = \mu_{\max} / 2 \propto d^{\eta/2}$ grows with $d$,
which is not modified by the normalization.
If $p$ is non-Gaussian then, using Proposition \ref{prop:covar-lower}, we still have $c(p) \geq \mu_{\max} / 2$ which means
that it grows at least like $d^{\eta/2}.$
To eliminate the growth due to this bad-conditioning of the covariance, we must
separate different frequency bands where eigenvalues have different amplitudes. This is a key idea which motivates the use of wavelet transforms to compute a
hierarchical probability factorization with the renormalization group.

\section{Approximation, Learning and Sampling with Hierarchic Flows in Wavelet Bases}
\label{sec:prob-sep}

Sampling a probability density $p$ with a Langevin diffusion or estimating its Gibbs energy by score matching becomes unfeasible when the log-Sobolev constant grows with the dimension $d$. 
This is typically the case for large multiscale fields $\data$. 
This section shows that this difficulty may be avoided
if we decompose $p$ into a hierarchic flow of probabilities across scales with a wavelet transform, and if we renormalize wavelet coefficients to reduce log-Sobolev constants.

Section \ref{wavelet} reviews multiresolution approximations and wavelet transforms.
Section \ref{sec:cond-proba} explains that a hierarchic flow of probabilities 
 computes an inverse renormalization group introduced in
\cite{marchand_wavelet_2022}.
Sections \ref{sec:cond-estim} and \ref{sec:sampling}
reviews the estimation and sampling of the resulting probability models, with
score matching and Metropolis adjusted Langevin diffusions.
For stationary probabilities, model parameters satisfy a coupling flow equation given in Section \ref{sec:energy-model}.
Section \ref{sec:basis-choice} relates wavelet properties to the log-Sobolev constants of
transition probabilities.

\subsection{Multiresolution Approximations and Wavelets}
\label{wavelet}

Multiresolutions define coarse-graining approximations whose evolutions across scales depend upon decomposition
coefficients in a wavelet basis \cite{mallat1989theory,meyer_1993}. 
They can be extended to arbitrary data defined on a graph.
We begin from the construction of multiresolution approximation on graphs to show that hierarchic flows can be applied to general data structures. We then concentrate 
on images where we perform numerical experiments.

\paragraph{Multiresolution and wavelets on graphs}

We consider $\data \in \R^d$ defined by its $d$ values $\data(n)$ on the 
nodes of a graph. Multiresolution approximations compute coarse graining
approximations of $\data$ of progressively smaller dimensions, 
by iterating over coarse graining operators \cite{Kondor2022multiresolution}. 
Let us write $\data_0 = \data$.
For each $j > 0$, $\data_j$ is calculated from $\data_{j-1}$ with a coarse-graining averaging operator $G_j$
\begin{equation}
\label{dec1}
\data_j = G_j \data_{j-1}~.
\end{equation}
The rows of $G_j$ sum to $1$. It is a sparse matrix  
which averages groups of neighbor coefficients of $\data_{j-1}$. The dimension of $\data_j$ is proportional to $2^{-r j}$ where $r$ is the dimension of the graph. 
If $\data$ is an image and hence defined on a two-dimensional graph, then $r=2$. At a level $j$, each value of $\data_j$ is computed by
averaging values of $\data_0$ over groups of neighbor nodes in the graph, whose sizes are proportional to $2^{r j}$. It provides an approximation at the scale $2^j$.
The graph topology is the prior information allowing to build these
multiscale groups. In the simplest cases, the averaged groups are non-overlapping and the coarse graining defines a computational tree
\cite{gavish2010multiscale}. 
The operator $G_j$ can also be defined as a diagonal operator in the
orthogonal basis which diagonalizes the graph Laplacian \cite{hammond2011wavelets}. 
It can then be interpreted as a convolution on the graph,
which projects $\varphi$ on the lower-frequency eigenvectors. 

The coarse graining is inverted by also
computing the high-frequency variations of $\data_{j-1}$ which have been eliminated by the averaging operator $G_j$. For this purpose, we define a complement operator $\bar G_j$ whose rows are sparse with nearly
the same support as $G_j$, and such that
$\binom{G_j}{\bar G_j}$ is an invertible square matrix.
Wavelet coefficients are the high frequencies $\bar \data_j$ of
$\data_{j-1}$ computed by this complement
\begin{equation}
\label{dec12}
\bar \data_j = \bar G_j \data_{j-1}.
\end{equation}
They measure the variations of $\data_{j-1}$ over local neighborhoods where $G_j$ averages $\data_{j-1}.$
Let $(H_j , \bar H_j)$ be the inverse matrix of $\binom{G_j}{\bar G_j}$:
\begin{equation}
\label{conditions}
H_j \,G_j +  \bar H_j \bar G_j = Id ,
\end{equation}
It results from (\ref{dec1}) and (\ref{dec12}) that
$\varphi_{j-1}$ can be recovered from $(\varphi_j, \bar \varphi_j)$
\begin{equation}
\label{dec2}
\data_{j-1} = H_j \data_j + \bar H_j \bar \data_j .
\end{equation}
If $\left( {G_j} \atop {\bar G_j} \right)$ is an orthogonal matrix then
$H_j = G_j\trans$ and $\bar H_j = \bar G_j\trans$ and this decomposition is orthogonal. 

Cascading $(G_j,\bar G_j)$ from $\data$ computes a multiscale averaging
$\data_j = A_j \data$ and wavelet coefficients $\bar \data_j = W_j \data$ with
\begin{equation}
\label{wave-def22}
A_j = G_j\,G_{j-1}...\, G_1 ~~\mbox{and}~~ W_j = \bar G_j G_{j-1}...G_1  .
\end{equation}
The matrix $W = (A_J\,, W_J\,,\,...\, W_1)$ is invertible and computes a wavelet transform
at all scales. It is orthogonal if each $\left( {G_j} \atop {\bar G_j} \right)$ is orthogonal.

\paragraph{Multiresolution of Images}
Images are sampled on a uniform grid which is a Euclidean graph of dimension $r =2.$
Fast wavelet transforms \cite{Mallat1989ATF} are calculated with convolutions and subsampling operators on this graph,
which do not depend on $j$. The operators 
$G_j  = G$ and $\bar G_j = \bar G$ are defined by
\begin{equation}
\label{convol-subsampl}
G \data (n) = \data * g (2n)~~\mbox{and}~~\bar G \data (n) = \data * \bar g (2n)~,
\end{equation}
$g$ being a low-pass filter which averages neighbor pixels in the image.
The complement $\bar g = (\bar g_k )_{1 \leq k \leq 3}$ is composed of
$3$ separable high-pass filters which compute the image variations over the same neighborhood. 
The inverse operators $H$ and $\bar H$ inserts zeros between each coefficient 
of $\data_j$ and $\bar \data_j$ before computing convolutions with dual bi-orthogonal filters $h$ and $\bar h$ \cite{doi:10.1137/1.9781611970104}.
This fast wavelet transform is illustrated in Figure \ref{WaveletImage}. Appendix \ref{app:wavelets} reviews the
properties of conjugate mirror filters
$(g,\bar g)$ which define an orthogonal matrix $\left( {G} \atop {\bar G} \right)$. It implies that $h(n) = g(-n)$ and $\bar h(n) = \bar g(-n).$ 
All numerical applications are computed with conjugate mirror filters, but this is not a required condition to define a hierarchic probability flow.

\begin{figure}
\centering
\includegraphics[width=0.65\textwidth]{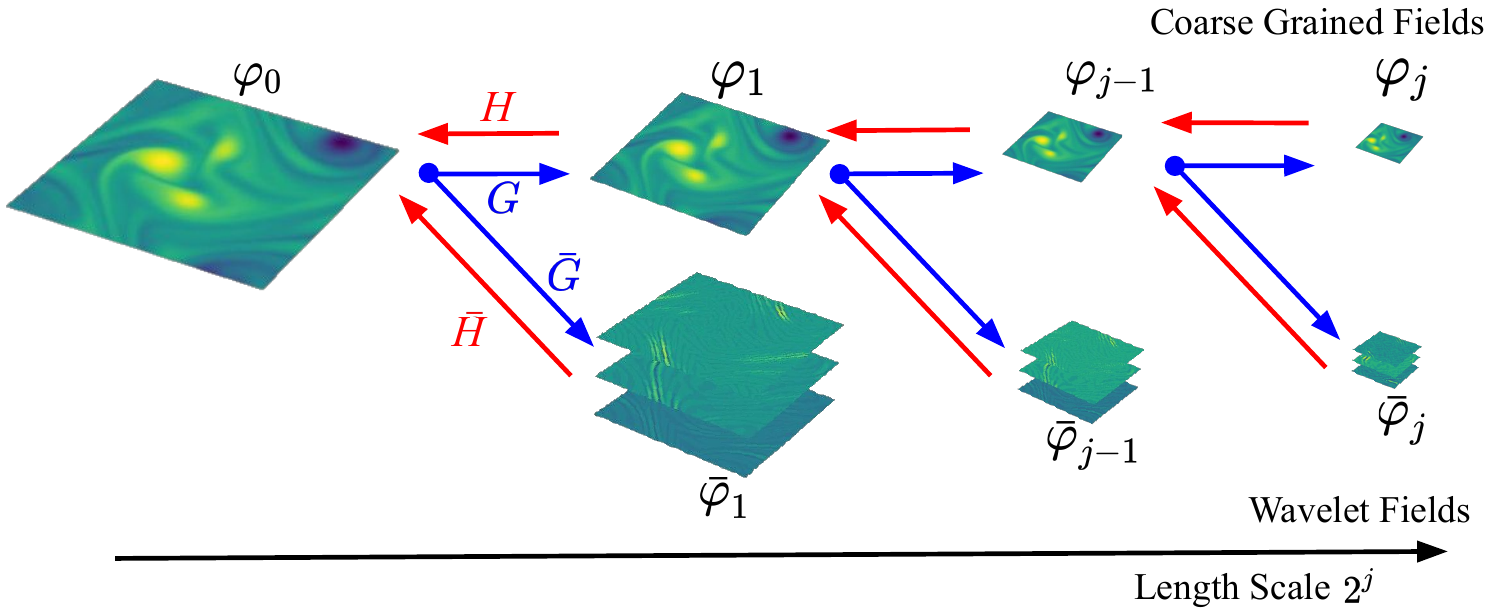}
\caption{A fast wavelet transform iteratively decomposes an image approximation
$\data_{j-1}$ into a coarser approximation $\data_{j}$ with a sub-sampled
low-pass filtering $G$,  and  $3$ wavelet coefficient images $\bar \data_{j}$. They are calculated by $\bar G$ with subsampled convolutions with $3$ band-pass filters along different orientations. A finer scale image $\data_{j-1}$ is reconstructed from $(\data_j , \bar \data_j)$
with the inverse operator $(H,\bar H)$. }
\label{WaveletImage}
\end{figure}

The operators $A_j$ and $W_j$ in (\ref{wave-def22}) are a cascade of $j$ convolutions 
and sub-samplings by $2$. They are thus convolutional operators, followed by a subsampling by $2^j$. Coarse-grained images and wavelet coefficients can therefore
be written as convolutions with a scaling filter $\phi_j$ and wavelets $\psi_{j,k}$ subsampled by $2^j$:
\begin{equation}
\label{filters}
\data_j = \big(\data * \phi_j (2^j n)\big)_{n}~~\mbox{and}~~
\bar \data_j = \big( \data * \psi_{j,k} (2^j n) \big)_{k \leq 3, n}.
\end{equation}
These scaling filters and wavelets are specified by
the filters $(g,\bar g)$ as explained
in Appendix \ref{app:wavelets}.
The support width of $\phi_j$ and $\psi_{j,k}$ is proportional to  $2^j$.
The Fourier transform $\hat \psi_{j,k}$ of each wavelet $\psi_{j,k}$
is dilated by $2^{-j}.$ These Fourier transforms are essentially localized in 
frequency annuli illustrated in Figure \ref{fig:wavelet-subd}(a), around the lower frequencies covered by $\hat \phi_j$.

The rows of the wavelet transform $W =(A_J\,, W_J\,,\,...\, W_1)$
are translated wavelets at all scales, which define a basis of $\R^d$. It is an 
orthogonal basis if $G$ and $\bar G$ are conjugate mirror filters. 
To study asymptotic properties of wavelet coefficients 
when $d$ goes to $\infty$, we need to control the convergence of these bases.
If we renormalize the support of $\data$ to $[0,1]^2$ then
discrete orthonormal wavelet bases converge to wavelet orthonormal 
bases of ${\bf L^2}([0,1]^2)$.
When the image size $d$ goes to $\infty$ then normalized discrete wavelets $\psi_{j,k}$ converge to wavelet functions
$\psi_{j,k}(x) = 2^{-j} \psi_k(2^{-j} x)$, such that
$\{ x\mapsto \psi_{j,k} (x- 2^j n) \}_{j \leq 0,~n \leq 2^{-j},~k\leq 3}$ 
is an orthonormal basis of ${\bf L^2}([0,1]^2)$ \cite{Mallat}.

\begin{figure}
\centering
\includegraphics[width=0.7\textwidth]{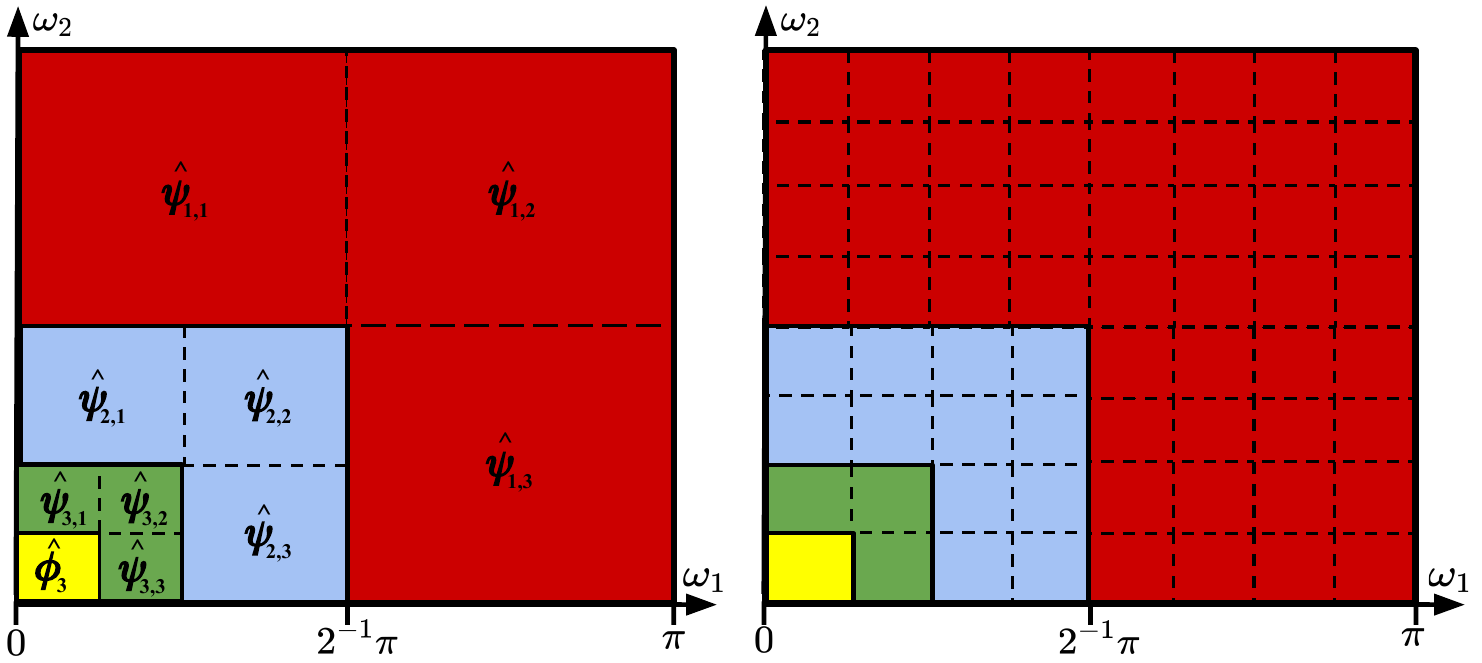}

~~~~~(a)~~~~~~~~~~~~~~~~~~~~~~~~~ ~~~~~~~~~~~~~~~~~~~~~~~~~~~~~~~~~~~(b)~~~~~~
\caption{(a): Frequency supports of Fourier transforms of two-dimensional wavelets $\hat \psi_{j,k}(\omega_1,\omega_2)$ for $1 \leq k \leq 3$ over $3$ scales $2^j$.
  (b): Frequency subdivisions of wavelet packets over $a_j$ levels at
  each scale $2^j$, with $a_1 = 2$, $a_2 = 1$ and $a_3 = 0$. }
\label{fig:wavelet-subd}
\end{figure}

\paragraph{Wavelet packets} 
The frequency resolution of wavelets 
can be improved with wavelet packet bases introduced in \cite{coifman1992wavelet}.
A wavelet packets transform sub-decomposes the frequency support of each 
wavelet. 
This is done by iteratively applying $a_j$ times the convolutions and subsampling $\left( G \atop \bar G \right)$, 
after applying $\bar G$ to $\data_{j-1}$. It computes
$\bar \varphi_j = \bar G_j \varphi_{j-1}$ with 
\begin{equation}
\label{waveletpacksn}
\bar G_j =  \left( G \atop \bar G\right)^{a_j}\, \bar G .
\end{equation}
The matrix $\left( G \atop \bar G_j \right)$ is 
orthogonal if $\left( G \atop \bar G \right)$ is orthogonal. It then computes decomposition coefficients in an
orthogonal basis of wavelet packet vectors.
The filter (\ref{waveletpacksn}) performs 
a frequency subdivision of each wavelet frequency
band into $2^{2 a_j}$ bands, illustrated in Figure \ref{fig:wavelet-subd}(b).
These wavelet packets thus 
have a Fourier transform which are concentrated
on square domains, which are $2^{a_j}$ times more narrow than for wavelets. However, 
the spatial support of these wavelet packets is $2^{a_j}$ times larger than for wavelets. 
Properties of wavelet packets are studied in \citep{coifman1992wavelet}.

\subsection{Hierarchic Flows and Renormalization Group}
\label{sec:cond-proba}

We now consider a random vector $\data$ defined on a graph whose probability
density is $p(\data)$. If $p$ has a large log-Sobolev constant, we avoid estimating and sampling $p$ directly. A hierarchic representation of $p$ is defined 
as a product of conditional probabilities of wavelet coefficients  \cite{marchand_wavelet_2022}, whose log-Sobolev
constants are reduced by renormalizing wavelet coefficients. 
This hierarchic representation is calculated as an inverse of the renormalization group
transformation.

\paragraph{Forward renormalization}
The renormalization group of Kadanoff \cite{kadanoff1976variational} and Wilson \cite{wilson1971renormalization} computes $p_j (\data_j)$ 
at each scale $2^j$, given $p(\data)$ at the finest scale. It iteratively computes $p_j$ from $p_{j-1}$
with a marginal integration over the high frequency 
degrees of freedom.
The use of an appropriate coordinate system to compute
this marginal integration
has been thoroughly studied \citep{Delamotte_2012,battle1999wavelets}.
Although they were not introduced with wavelets, Kadanoff's and Wilson's renormalization group have been revisited from this perspective by Battle and others \cite{battle1987ondelettes,battle1999wavelets}.
Kadanoff renormalization group \citep{kadanoff1976variational} computes $\data_j$ from $\data_{j-1}$ with a block averaging, 
which amounts to define $\bar \data_j$ as orthogonal wavelet
coefficients in the Haar basis, which has a minimal spatial support but is discontinuous.   
In an early publication \citep{wilson1971renormalization}, Wilson
introduces a decomposition, later identified as a decomposition in a Shannon
wavelet orthogonal basis. Shannon wavelets have a Fourier transforms
which is discontinuous, with a minimum frequency support. These wavelets
thus have a slow spatial decay. Their properties are reviewed in Appendix \ref{app:wavelets}. 
Other wavelet bases adapted to specific classes of Hamiltonians have been designed by Battle \citep{battle1999wavelets}, but are difficult to compute. 

Following \cite{marchand_wavelet_2022}, 
we represent the high frequencies of $\varphi_{j-1}$
with wavelet coefficients $\bar \data_j$.
Wavelet coefficients $\bar \data_j$ are normalized 
by dividing each coordinate $\bar \data_j(n)$ 
by its standard deviation $\sigma_{j,i}$.
Let $D_j = \diag (\sigma_{j,i}^{-1})_i$ be the corresponding
diagonal matrix. Normalizing $\bar \data_j$
is equivalent to replace
$\bar G_j$ by $D_j \bar G_j$ and $\bar H_j$ by $\bar H_j D_j^{-1}$.

Since $\data_{j-1} = H_j \data_j +  \bar H_j \bar \data_j$ we have
$d \data_{j-1} = w_j d \data_j d \bar \data_j$ 
where $w_j = |{\rm det}(H_j , \bar H_j)|$ is the Jacobian modulus.
The marginal integration of high-frequencies in this wavelet basis can be written
\begin{equation}
\label{margin-int}
p_j(\data_j) = w_j \,\int p_{j-1}(\data_{j-1})\, d \bar \data_j .
\end{equation}

\paragraph{Hierarchic flow as an inverse renormalization}
At a large scale $2^J$, 
$\data_J$ is of sufficiently low dimension so there is no difficulty
to estimate $p_J(\data_J)$ from data, and to sample this probability density.
A hierarchic flow is an inverse renormalization group transformation which 
maps $p_J$ into $p$ with a Markov chain. 
It computes $p_{j-1} (\data_{j-1})$ by multiplying $p_j (\data_j)$ with
a conditional probability of $\data_{j-1}$ given $\data_j$, which is also equal to
 the conditional density of $\bar \data_j$ given $\data_j$
\begin{equation}
\label{probcond}
p_{j-1}(\data_{j-1}) = w_j^{-1} \,p_j (\data_j)\, \bar p_j (\bar \data_j | \data_j) .
\end{equation}
Cascading (\ref{probcond}) transforms $p_J$ into
$p$ with transition kernels defined
by these conditional probabilities
\begin{equation}
\label{factor-proba}
p (\data) = w^{-1}\, p_J (\data_J) \prod_{j=1}^J \bar p_j (\bar \data_j | \data_j ) ~~\mbox{with}~~w = \prod_{j=1}^J w_j ~.
\end{equation}
This inverse wavelet renormalization group \cite{marchand_wavelet_2022} computes
$p$ starting from $p_J$.

\subsection{Estimation of a Conditional Probability Model}
\label{sec:cond-estim}

Given $m$ samples $\{ \data^{(i)} \}_{i \leq m}$ of $p$, 
a model of $p$ is estimated with the hierarchical factorization
(\ref{factor-proba}), from exponential models of each conditional probability. 
Each conditional probability model is estimated from data by score matching, whose precision
depends upon its log-Sobolev constant.

\paragraph{Hierarchic model}
An exponential model $p_\theta$ of $p$ is defined from exponential models 
$p_{\theta_J}$ and $\bar p_{\bar \theta_j}$ 
of $p_J$ and $\bar p_j$:
\begin{equation}
\label{final-model}
p_{\theta} (\data) = w^{-1}\, p_{\theta_J} (\data_J) \prod_{j=1}^{J} \bar p_{\bar \theta_{j}}(\bar \data_{j}|\data_j) .
\end{equation}

An exponential model of $p_J$ is defined as in (\ref{linear-decomp}) by
\begin{equation}
\label{internasf2}
p_{\theta_J}(\data_J)  = {\cal Z}_J^{-1} e^{-\theta_J\trans \Phi_J(\data_J)} .
\end{equation}
We choose $J$ large enough so that $\data_J$ is sufficiently low-dimensional to easily compute this estimation. For any $j \leq J$, an exponential model of $\bar p_j$ is defined by 
\begin{equation}
\label{internasf}
\bar p_{\bar \theta_j} (\bar \data_j | \data_j) = e^{F_j (\data_j) - \bar \theta_{j}\trans \Psi_j (\data_{j-1})}~,
\end{equation}
where $F_j$ is a free energy which normalizes the conditional probability:
\begin{equation}
\label{normalisat-eq}
\int \bar p_{\bar \theta_j} (\bar \data_j | \data_j)\,  d \bar \data_j = 
  e^{F_j (\data_j)} \int e^{- \bar \theta_{j}\trans \Psi_j  (\data_{j-1})}\, d \bar \data_j  = 1.  
\end{equation}
Each free energy $F_j$ is specified by $\bar \theta_j$, but it does
need to be computed to estimate $\bar \theta_j$ or sample $\bar p_{\bar \theta_j}$.

The model $p_\theta = {\cal Z}_\theta^{-1} e^{-U_\theta}$ 
has a Gibbs energy 
\begin{equation}
\label{modelwithfree}
\En_{\theta} = \theta_J\trans \Phi_J + \sum_{j=1}^{J} \big(\bar \theta_{j}\trans \Psi_j - F_j \big).
\end{equation}
Section \ref{sec:energy-model} explains that it can be calculated with a coupling
flow equation, which requires regressing each $F_j$.

\paragraph{Score Matching}
A maximum likelihood estimation computes
$\theta = (\theta_J , \bar \theta_j)_{j\leq J}$ by minimizing
$KL(p,p_{\theta})$.
It results from the factorization (\ref{factor-proba}) of $p$
and (\ref{final-model}) of $p_\theta$ that
\begin{equation}
\label{model-error2}
KL(p , p_{\theta}) = KL (p_J , p_{\theta_J}) + \sum_{j=1}^{J} \, \E_{p_j} \big(KL (\bar p_{j} , \bar p_{\bar \theta_{j}})\big) .
\end{equation}

The minimization of $KL(p , p_{\theta})$ is thus obtained by minimizing 
each $\E_{p_j} \big(KL (\bar p_{j} , \bar p_{\bar \theta_{j}})\big)$.

As explained in Section \ref{sec:estim-gibbs},
minimizing a $KL$ divergence with a gradient descent is computationally
expensive because it requires computing normalization constants.
We thus optimize $\{\theta_J , \bar \theta_j\}_j$ by score matching, which
replaces the $KL$ divergence by a relative Fisher information. 
The parameters $\theta_J$ are calculated 
by minimizing the Fisher information $\I(p_J , \bar p_{\theta_J})$,
which amount to minimize a quadratic loss (\ref{eq:loss-score}).
Interaction parameters $\bar \theta_j\trans$ are calculated 
by minimizing the averaged Fisher information 
\[
\expect[p_j] {\I (\bar p_{\bar \theta_j} , \bar p_j  )} =  \E_{p_{j-1}}
\Big(\|\nabla_{\bar \data_{j}} \log \bar p_{j} ( \bar \data_{j}|\data_{j}) - \nabla_{\bar \data_{j}} \log \bar p_{\bar \theta_{j}}
( \bar \data_{j} | \data_{j})\|^2 \Big).
\]
The score gradient is computed on $\bar \data_j$ for $\data_j$ fixed,
and the Fisher information does not depend on the free energy $F_j$ .
Appendix \ref{app:score-matching} shows that 
$\bar \theta_j$ is also a solution of a quadratic minimization.

\paragraph{Conditional log-Sobolev constants}
An upper bound of
$\E_{p_j} (KL (\bar p_{j} , \bar p_{\bar \theta_{j}}))$ computed from
$\E_{p_j} ( \I(\bar p_j , \bar p_{\bar \theta_j}))$  depends on a
log-Sobolev constant $c(\bar p_j )$ averaged over $p_j$.
The conditional log-Sobolev constant $c(\bar p_j)$ of $\bar p_j(\bar \data_j | \data_j)$ is defined as
the smallest constant so that for any conditional probability $\bar q(\bar \data_j | \data_j)$
\begin{equation}
\label{log-sobolev2}
\E_{p_j} \big( KL(\bar q , \bar p_j ) \big)\leq c(\bar p_j)\, \E_{p_j} \big({\cal I} (\bar q , \bar p_j ) \big).
\end{equation}

One can apply Theorem 2 in \cite{koehler2022statistical} to prove
that the number $m$ of samples needed for a score matching to achieve a comparable accuracy as a Kullback-Leibler divergence minimization 
is asymptotically multiplied by this conditional log-Sobolev constant.
Section \ref{sec:estim-gibbs} shows that for exponential models, 
minimizing a KL divergence is equivalent to match moments, and hence that
minimizing $\E_{p_j} \big(KL (\bar p_{j} , \bar p_{\bar \theta_{j}})\big)$
is equivalent to find $\bar \theta_j$ such that
\[
\E_{p_j} \E_{\bar p_{\bar \theta_j }} (\Psi_j) = \E_{p_{j-1}} (\Psi_j ).
\]
We can thus evaluate the numerical precision of
score matching estimators from this moment matching condition.

\subsection{Hierarchic Sampling}
\label{sec:sampling}

\paragraph{Sampling with a discrete stochastic equation}
A sample $\data $ of a hierarchic model $p_{\theta}$ 
of $p$ is calculated from coarse to fine scales, by first sampling $p_{\theta_J}$ and then iteratively sampling each conditional probabilities
$\bar p_{\bar \theta_j}$. These probability densities are sampled with a Metropolis Adjusted Langevin diffusion \cite{grenander1994representations,roberts1998optimal}, and we relate the rate of convergence of the unadjusted dynamic to log-Sobolev constants. This is further developed in Appendix \ref{sec:mixingtimelangevin}.

A hierarchic sampling computes a sample $\data$ of $p_{\theta}$ as follows.\\
\begin{itemize}
\item Initialization: compute a sample $\data_J$ of $p_{\theta_J}$. 
\item For $j$ from $J$ to $1$, given
$\data_j$ compute a sample $\bar \data_j$ of $\bar p_{\bar \theta_j}( \cdot | \data_j)$
and set  $\data_{j-1} = H_j \data_j + \bar H_j \bar \data_j$.\\
\end{itemize}

The sample $\data = \data_0$ of $p_{\theta}$ is thus calculated by iterating on a discrete stochastic equation, which recovers
$\data_{j-1}$ from $\data_{j}$ by sampling 
the conditional probabilities $\bar p_{\bar \theta_j}( \cdot| \data_j)$. 

\paragraph{Langevin sampling of conditional probabilities}
The wavelet conditional probability can be sampled with a
Langevin (or MALA) algorithm, which does not depend upon the normalization free energy $F_j$.
Langevin diffusions have an exponential convergence if their log-Sobolev constants are uniformly bounded.
To simplify explanations, 
we neglect the model approximation error, so $p_{\theta_J} = p_J$ and $\bar p_{\bar \theta_j} = \bar p_j$.
For a fixed $\data_j$, a Langevin diffusion approximates $\bar p_j(.|\data_j)$ by $\bar p_{j,t}$ after time $t$. The product $p_{t} = w^{-1} p_{J,t} \prod_{j=1}^J \bar p_{j,t}$ defines
an approximation of $p$.
The KL divergence error between $p$ and $p_{t}$ can be decomposed as a sum
\begin{equation}
\label{model-error0}
KL(p_{t} , p) = KL (p_{J,t} , p_{J}) + \sum_{j=1}^{J} \E_{p_{j,t}} \big(KL (\bar p_{j,t} , \bar p_{j})\big) .
\end{equation}

The decay of $KL(p_{J,t},p_{J})$ is driven by the log-Sobolev constant $c(p_J)$,
according to (\ref{exp-dec}). 
The same result applies to the conditional probabilities $\bar p_j $ if we incorporate
the expectation in $p_{j}$.

Similarly to (\ref{exp-dec}),
De Bruin identity (\ref{deBruin}) implies an exponential convergence
of the expected KL divergence:
\begin{equation}
\label{exp-dec22}
\E_{p_j} \big(KL(\bar p_{j,t} , \bar p_j)\big) \leq e^{-t/c(\bar p_j)} \E_{p_j} \big( KL(\bar p_{j,0} , \bar p_j) \big)~.
\end{equation}
Notice that the expected value is in $p_j$ and not in $p_{j,t}$ as in 
(\ref{model-error0}) but they converge to the same value because 
$p_{j,t}$ converges to $p_j$ when $t$ increases.

\paragraph{Sampling with a stochastic differential equation}
This discrete hierarchic sampling is defined from an inverse renormalization group computed in a wavelet basis, over discrete 
dyadic scales. Continuous hierarchic sampling algorithms can  also be
calculated with a renormalization group resulting from a coarse graining over a continuum of scales. Such a coarse graining defines a continuous transport of the probability distribution. It 
can be inverted with a Stochastic Differential Equation (SDE),
which sample $\varphi$ given $\varphi_J$
\cite{cotler2023renormalizingdiffusionmodels}.
The continuous coarse graining may be defined in a Fourier basis by a progressive elimination of the highest frequencies as defined
by the Polchinski's renormalization \cite{polchinski1984renormalization}.
The inverse SDE computes a sample of $p$ from $\varphi_J$
by updating high frequency Fourier coefficients.
However, the drift term of this SDE depends upon the interactions of
high frequencies with lower frequencies that are usually 
non-local and hence difficult to estimate. We shall see that
this is the case for simple models such as $\varphi^4$, which include
scalar potentials. Scalar potentials define interactions that are local in space and  hence delocalized across frequencies. 

One may also define a continuous renormalization group with a diffusion operator which computes a non-orthogonal continuous wavelet transform over a continuous range of scales \cite{Carosso_2020}.
The inverse transport is also computed with an SDE
\cite{cotler2023renormalizingdiffusionmodels}. However, the drift
term of this equation requires to 
estimate the coupling parameters of the field
at each intermediate scale. The
SDE computations must also be calculated on the fine spatial grid of $\varphi$, because there is no underlying orthogonal basis.
The use of such redundant representations
become computationally prohibitive for large size fields.

\paragraph{Score diffusion sampling}
More efficient computational strategies have been developed with
a discrete wavelet hierarchic sampling, 
by replacing the Langevin sampling of $\bar p_{\bar \theta_j} (.|\varphi_j)$ with score diffusion \cite{song2020score} or a stochastic interpolant \cite{albergo2023stochasticinterpolantsunifyingframework}. 
Score diffusion samples are calculated
by inverting an Ornstein-Uhlenbeck SDE which transports $\bar \varphi_j$
into a Gaussian white noise. All calculations are performed
on the reduced grid spatial of $\bar \varphi_j$.
It requires to estimate the gradient of the log probability 
of wavelet coefficients $\bar \varphi_j$ contaminated by
Gaussian white noises.
It has been shown that score functions can then be estimated with deep neural networks having local receptive fields
\cite{kadkhodaie-local-conditional-models}, even for complex
random processes such as human faces.
This combination of the
inverse renormalization group  with score diffusion also
accelerates computations of standard score diffusion models
\cite{guth_wavelet_2022}. State of the art image generation results
have been obtained with such multiscale score diffusion 
algorithms \cite{masuki2025generativediffusionmodelinverse,batzolis2022nonuniformdiffusionmodels,sheshmani2024renormalizationgroupflowoptimal}.

To simplify the mathematical analysis through log-Sobolev constants,
in this paper we shall use a Langevin diffusion to sample
$\bar p_{\bar \theta_j}$. It also avoids learning
a deep neural network, which requires a large training data-set.

\subsection{Coupling Flow Equation of Stationary Energy Models}
\label{sec:energy-model}

In some applications, the Gibbs energy model $U_\theta$  \cref{modelwithfree} of $p_\theta$ needs to be explicitly
calculated. For example, to compute high-dimensional integrals with
Monte Carlo reweighing  \cite{gabrie2022adaptive}, or to analyze the interaction properties of a physical system. It then requires regressing the free
energies $F_j$ over predefined potential vectors.
We define hierarchical potentials, allowing to build energy
models whose dimensions increase with the field size. This is needed to approximate complex fields having progressively more degrees of freedom when their resolution increases.
For stationary probabilities, we prove that $U_\theta$ can then be calculated
with a discrete coupling flow equation.

\paragraph{Energy calculation} 
Each conditional probability model
$\bar p_{\bar \theta_j} (\bar \data_j | \data_j) = e^{F_j (\data_j) - \bar \theta_{j}\trans \Psi_j (\data_{j-1})}$ involves a free
energy $F_j(\data_j)$ that is approximated by a linear regression
$\alpha_j\trans\Phi_j (\data_j)$. 
Inserting this approximation in the energy model 
 (\ref{modelwithfree}) gives
\begin{equation}
\label{modelwithfree2}
\En_{\theta} = \theta_J\trans \Phi_J + \sum_{j=1}^{J} \big(\bar \theta_{j}\trans \Psi_j - \alpha_j\trans \Phi_{j} \big).
\end{equation}
The  coefficients 
$\alpha_j$ are calculated from the conditional probability normalization (\ref{normalisat-eq}), up to an additive constant
\begin{equation}
\label{normalisat-eq2}
 e^{\alpha_j\trans \Phi_j (\data_j)}\int e^{- \bar \theta_{j}\trans \Psi_j  (\data_{j-1})}\, d \bar \data_j  \approx \mbox{cst} .  
\end{equation}
Appendix \ref{app:free} shows that calculating the gradient along $\data_j$ 
allows us to compute $\alpha_j$ by minimizing a quadratic form.

\paragraph{Stationary models}
If $p(\data)$ is stationary and hence has a Gibbs energy $\En(\data)$ which is
invariant to translations on the sampling grid of $\data$
then  $p_j (\data_j)$ is invariant by translation on the coarser
sampling grid of $\data_j$. Hierarchic models in wavelet  bases are not strictly stationary because 
a wavelet orthonormal wavelet basis is not invariant by translations. To define a stationary model, we iteratively project the hierarchic model over translation invariant functions.

 Let $T_\tau$ be a translation of $\data_{j-1}$ by $\tau$ (modulo periodic boundary conditions).  A projection of $f(\varphi_{j-1})$ over translation invariant functions of $\data_{j-1}$ is computed by averaging its values over all the translations of $\data_{j-1}$
on its grid ${\cal G}_{j-1}$ of size $|{\cal G}_{j-1}|$
\begin{equation}
\label{proj-fun}
    (\Ave_{j-1} f) \,(\data_{j-1}) = \frac{1}{|{\cal G}_{j-1}|}
    \sum_{\tau \in {\cal G}_{j-1}} f( T_\tau \data_{j-1}).
\end{equation}
If $f(\data_j)$ is a function of $\data_j = G \data_{j-1}$ then 
we write $\Ave_{j-1} f = \Ave_{j-1} f_G$ with
$f_G (\data_{j-1}) = f(G \data_{j-1})$.
If $f(\data_j)$ is invariant to translations of $\data_j$  on its
grid ${\cal G}_j$ then the sum (\ref{proj-fun})
can be reduced to the $4$ translations $\tau \in {\cal G}_{j-1} / {\cal G}_{j}$.
The following theorem proves that this translation invariant projection of Gibbs energies reduces the Kullback-Leibler error on stationary densities.

\begin{proposition}
\label{th:translat}
Let $p(\data_{j-1})$ be a stationary density. If $q(\data_{j-1})$ is a density of energy $\En$ and if $\tilde q(\varphi_{j-1})$ is the density of energy $\Ave_{j-1} \En$ then
\begin{equation}
\label{KL-diff}
KL(p , \tilde q) \leq KL(p , q) .
\end{equation}
\end{proposition}

The proof is in Appendix \ref{app:translat}.
This proposition proves that energy models of stationary probabilities are improved by the projection (\ref{proj-fun}) on translation invariant functions.

\paragraph{Coupling flow equation with hierarchic potentials}
We introduce
hierarchic stationary models where coupling parameters
can be calculated with a  flow equation from coarse to fine scales, using data to estimate each term. It inverts the renormalization group coupling flow equation, which goes from fine to coarse scales \cite{Delamotte_2012}.

At the largest scale $2^J$, we have computed a model $\En_{\theta_J}$
of the Gibbs energy of $p_J.$ At each scale $2^j$, we can compute an approximation $\En_{\theta_{j-1}}$ of the Gibbs energy of $p_{j-1}$
from an approximation $\En_{\theta_{j}}$ of the Gibbs energy of $p_{j}$,
by adding the interaction energy model ${\bar \theta_j}\trans \Psi_j -  \alpha_j\trans \Phi_j$ of $\bar p_j = p_{j-1} / p_j$. Proposition \ref{th:translat} proves that the projection $\Ave_{j-1}$ reduces the approximation error.
A translation invariant Gibbs energy model having a reduced error is thus
\begin{equation}
\label{iterations10}
\En_{\theta_{j-1}}  = \Ave_{j-1}(\En_{\theta_j}  + {\bar \theta_j}\trans \Psi_j - \alpha_j\trans \Phi_j ). 
\end{equation}
The following definition imposes a hierarchic condition on $\Phi_{j-1}$ so that $\theta_{j-1}$ can be calculated from $(\theta_j, \bar \theta_j, \alpha_j)$ with a linear equation.

\begin{definition}
\label{def:embedded-potential}
We say that $\{\Phi_j \}_{0 \leq j \leq J}$ are hierarchic stationary with interactions
$\{\Psi_j \}_{1 \leq j \leq J}$ if all $\Phi_j (\data_j)$ are invariant to translations of $\data_j$ and 
 if there exists a linear operator $Q_{j}$ such that
\begin{equation}
\label{multis-ansatz}
 \Ave_{j-1} (\Phi_j , \Psi_j) = Q_{j} \Phi_{j-1} ~.
\end{equation}
\end{definition}

This definition generalizes renormalizable models which are self-similar and
have the same dimension at all scales. hierarchic stationary potentials are constructed by
progressively incorporating new interaction potentials $\Psi_j$ for each $j$. The dimension
of $\Phi_j$ is therefore increasing as the scale $2^j$ decreases.
This generalization will allow us 
to build potential vectors that can approximate energies of complex fields in Section \ref{sec:scat-cov}.

For hierarchy stationary models, the following proposition derives that
the parameter vector $\theta_{j-1}$ of 
$ \En_{\theta_{j-1}}$ satisfies a linear 
equation, which relates it to $(\theta_{j},\bar \theta, \alpha_j)$.

\begin{proposition}
\label{prop:couplingflow}
If $\{ \Phi_j \}_{j\leq J}$ are hierarchic stationary 
satisfying (\ref{multis-ansatz}) then the Gibbs energy $\En_{\theta_{j-1}}$ in
(\ref{iterations10}) is given by $\En_{\theta_{j-1}} =  \theta_{j-1}\trans \Phi_{j-1}$
with
\begin{equation}
\label{prop-eq3}
\theta_{j-1} =   Q_{j}\trans(\theta_j - \alpha_j , \bar \theta_{j-1} ) .
\end{equation}
\end{proposition}

\begin{proof}
This property is proved by induction on $j$. It is valid for $j=J$ where
$\En_{\theta_J} = \theta_J\trans \Phi_J.$
Suppose that it is valid for $j \leq J$. Inserting 
$\En_{\theta_j} = \theta_j \trans \Phi_j$
in (\ref{iterations10})  implies with (\ref{multis-ansatz}) that 
\[
 \En_{\theta_{j-1}} =  ( \theta_j - \alpha_j \, ,\,  \bar \theta_j ) \trans 
 (\Phi_j , \Psi_j) =
\theta_{j-1}\trans \Phi_{j-1}~,
\] 
where  $\theta_{j-1}$ satisfies (\ref{prop-eq3}). 
\end{proof}

This proposition computes energy models with a discrete coupling flow equation in a wavelet basis, from coarse to fine scales.
It inverts the renormalization group equation, which goes from fine to coarse scale. In a Fourier basis, this renormalization group equation can be written as a differential equation, which defines a Polchinski flow \cite{polchinski1984renormalization,bauerschmidt2023stochastic}. The inverse equation involves the parameters of conditional probabilities, which specifies fine scale probabilities from coarser scales. 
The dimension of the coupling flow vector $\theta_j$ also increases as the scale decreases, which is necessary to obtain
accurate models of complex fields which are not exactly self-similar. At the finest scale $j=0$, we obtain a translation invariant Gibbs energy $\En_{\theta_0} = \theta_0\trans \Phi_0$ of a  stationary model $p_{\theta_0}$ of $p = p_0$.

\subsection{Log-Sobolev Constants and Wavelet Choice}
\label{sec:basis-choice}

We want to decompose $p$ having a large log-Sobolev constant into conditional probabilities $\bar p_j$ having smaller log-Sobolev constants, which can therefore
be learned and sampled more easily. To do so,
we study the choice of
basis and of the hierarchical projectors $G_j$ and $\bar G_j.$
In the following, we suppose that $\left( G_j \atop \bar G_j \right)$ 
is an orthogonal matrix.

\paragraph{Selection of eigenvectors}
The Bakry-Emery theorem 
gives an upper bound of $c(\bar p_j)$ from the inverse of the smallest eigenvalue of the Hessians $\Hess_{\bar \data_j} \bar \En_j$, if it is positive.
This suggests choosing 
$G_j$ and $\bar G_j$ so that it maximizes this minimum eigenvalue. In the orthogonal case,
$\data_{j-1} = \bar G_j\trans \bar \data_j + G_j\trans  \data_j  $ so
\begin{equation}
\label{Hessian-inequality}
\Hess_{\bar \data_j} \bar \En_{j}  = \bar G_j (\Hess_{\data_{j-1}} \En_{j-1} ) \, \bar G_j\trans  .
\end{equation}
The Hessian eigenvalues of $\bar \En_j$ are thus obtained by selecting the Hessian
eigenvalues of $\En_{j-1}$ with the orthogonal operator $\bar G_j$.
To minimize the log-Sobolev constant, $\bar G_j$ must eliminate small or negative eigenvalues
of $\Hess_{\data_{j-1}} \En_{j-1} (\data_{j-1})$.
If the
high amplitude eigenvectors of $\Hess_{\data_{j-1}} \En_{j-1}(\data_{j-1})$ are concentrated in a fixed linear space with high probability, then the range of
$\bar G_j\trans$ should  be included in this space.

\paragraph{Renormalized log-Sobolev lower-bound}
The renormalization of $\bar \data_j$ by $D_j$ aims at preconditioning
the covariance of $\bar p_{j}$ to avoid creating a large log-Sobolev constant.
If $\bar \mu_{{\max},j}$ is the largest eigenvalue of the covariance
$\bar C_j$ of $\bar \data_j$ then (\ref{sbol-covariance}) proves that 
\begin{equation}
    \label{log-ccrensdf}
c(\bar p_j ) \geq  \frac 1 2 \, \bar \mu_{{\max},j}.
\end{equation}
If $p$ is Gaussian then $\bar p_j (. |\data_j)$ is then also Gaussian so $c(\bar p_j ) = \bar \mu_{{\max},j} / 2$. The log-Sobolev normalization amounts to multiply by the largest eigenvalue of
$\nabla^2_{\bar \varphi_j} \bar U_j$. In the Gaussian case it divides by the smallest eigenvalue of the covariance and is
thus equal to the covariance condition number.
The covariance $C_{j}$ of
$\data_{j}$ is computed iteratively
from $C_0$ with (\ref{dec1}), which implies that
$C_j = G_j C_{j-1} G_j\trans$. 
The covariance $\bar C_j$ of $\bar \data_j$ is computed from
$\bar G_j$, which includes the renormalization. It  gives
\begin{equation}
  \label{Jacobi-preconditioning}
\bar C_j = \bar G_j C_{j-1} \bar G_j\trans ~~\mbox{with}~~\diag(\bar C_j) = Id . 
\end{equation} 
The covariance $ C_{j-1}$ is projected and renormalized by $\bar G_j$,
which sets the diagonal values of $\bar C_j$ to $1$.
The maximum eigenvalue and the condition number of
$\bar C_j$ do not grow with the dimension $d$
if $\bar G_j$ represents $C_{j-1}$ over a basis of nearly
eigenvectors, so that all eigenvalues remain of the order of $1$. This necessary condition to control the log-Sobolev constant is not sufficient. 
Non-convex Gibbs energies may have
much larger log-Sobolev constants. This issue is studied numerically in Sections \ref{subsec:scalarpots} 
and \ref{sec:applications} for scalar potential energies and two-dimensional turbulence data.

\paragraph{Choice of wavelet and wavelet packet basis}
For multiscale stationary processes, we explain that the largest eigenvalues and the condition numbers of all 
$\bar C_j$ remain bounded if 
the wavelet has a Fourier transform which is sufficiently well localized. We do not give mathematical details but quote the main results. It gives necessary conditions to control the growth of $c(\bar p_j)$ and its normalization, with the dimension $d$. 

Multiscale stationary fields have a density $p$ whose covariance is diagonalized in a Fourier basis, with eigenvalues which grow like $|\omega|^{-\eta}$ at low-frequencies. In an image of width $d^{1/2}$, if the smallest eigenvalue is of the order of $1$ then 
the largest eigenvalue is of the order of 
$d^{\eta/2}$, which implies that the log-Sobolev constant
$c (p)$ increases at least like $d^{\eta/2}$.
A hierarchical factorization tries to avoid this growth 
by renormalizing $\bar \data_j$ so that its
covariance does not grow with $d$.
Normalizing the covariance is necessary but not sufficient
to guarantee that
the normalized log-Sobolev constant remains bounded for all $d$.
This critical slowing down also appears for non-convex energies. 
For example, Multigrid Monte Carlo \cite{PhysRevD.40.2035} algorithms
perform such a normalization
without calculating conditional distributions across scales,
and they do not eliminate the
critical slowing down of the $\varphi^4$ model which has a non convex energy.

The condition number of the normalized covariance of $\bar \data_j$
is computed when $d$ goes to infinity by studying the 
representation of the limit covariance operator
of $\data$ in a wavelet orthonormal basis of ${\bf L^2}([0,1]^2)$.
The renormalization sets the diagonal of $\bar C_j$ to $1$.
The covariance is said to be preconditioned by its diagonal if the lowest and largest eigenvalues then remain of the order of $1$.
Classes of linear singular operators, preconditioned by their diagonal in a wavelet basis, have been thoroughly studied in harmonic analysis \cite{meyer_1993}. 
It has been applied to the fast resolution of elliptic problems \cite{jaffard1992elliptic}.
For appropriate wavelets, this preconditioning is valid for
classes of pseudo-differential operators, and classes of
singular homogeneous operators diagonalized in a Fourier
basis.

Suppose that the covariance eigenvalues are proportional to $|\omega|^{-\eta}$.
Preconditioning requires to use regular wavelets with enough
vanishing moments.
At low frequencies, each wavelet $\psi_k$ must have a Fourier transform
$\hat \psi_k$ satisfying
$|\hat \psi_{k} (\omega)| = O(|\omega|^{\eta/2})$, to avoid being contaminated by the explosion of the largest covariance eigenvalues at low-frequencies $\omega$. If a wavelet has $m$ vanishing moments, then Appendix \ref{app:wavelets} shows that $|\hat \psi_{k} (\omega)| = O(|\omega|^m)$. We thus choose a wavelet with $m \geq \eta/2$ vanishing moments.
At high frequencies, $|\hat \psi_k (\om)|$
must have a decay faster than $|\omega|^{-\eta/2}$, which is satisfied if
$\psi_k$ has $m \geq \eta/2$ derivatives in ${\bf L^2} (n[0,1]^2)$.
If the $\psi_k$ have a compact support, more than $\eta/2$ vanishing moments and
$\eta/2$ bounded derivatives, then
one can prove \cite{meyer_1993,marchand_wavelet_2022} that the maximum eigenvalue and the condition number 
of the covariance of all
$\bar \data_j$ are uniformly bounded for all $j$ and $d$.

Despite normalization, a bad conditioning may be produced by the smallest eigenvalues of the covariance, if their decay is
faster than a power law at high frequencies. 
Indeed, the frequency resolution of wavelets is 
not sufficient to renormalize this fast decay.
To ensure that the condition number of the 
covariance does not increase with $d$, one can represent the high-frequency  $\bar \varphi_j$ in a wavelet packet basis having a better frequency resolution. Each wavelet packet must have a Fourier transform concentrated in a sufficiently narrow frequency domain,
where the covariance eigenvalues vary by a bounded multiplicative factor.
The frequency width of these wavelet packets is $2^{a_j}$ times more narrow than wavelets if computed with a wavelet packet filter $\bar G_j$ defined in (\ref{waveletpacksn}). The constant $a_j$ can be adjusted to define wavelet packet coefficients $\bar \varphi_j$ whose normalized covariance
$\bar C_j$ has a condition number of the order of $1$.
Improving wavelet packet frequency resolutions by $2^{a_j}$ also increases their spatial support by a factor of $2^{a_j}$. We thus choose $a_j$ to be as small as possible. This is applied in Section \ref{sec:applications} to improve the Langevin mixing time for the vorticity fields of two-dimensional turbulent flows.

\section{Hierarchical Models of Local Scalar Potentials}
\label{sec:scal-pot-ene}

Scalar potential models introduced in this section are local interaction models often
used in statistical physics. We study the particular case of the $\varphi^4$ model to illustrate the estimation
and sampling properties of hierarchical probability flows at phase transitions.

\subsection{Scalar Potential Energies and $\vvarphi^4$ Model}
\label{sec:scalar-pot}
Physical systems at equilibrium have a probability density
$p = {\cal Z}^{-1} e^{-U}$ with an energy $U$ decomposed in a quadratic term corresponding to two-point interactions and a non-linear potential $V(\data)$ which specifies all other interactions:
\begin{equation}
\label{physics-ener}
\En(\data) = \frac {- \beta} 2  \data\trans \Delta \data  + V(\data) .
\end{equation}
The Laplacian $\Delta$ is discretized over the grid of $\data$ and defines the kinetic energy at a temperature $1/\beta$. 
Some physical systems \citep{ramond2020field} have a potential $V$ which is
reduced to a sum of scalar potentials at all locations $n$
\begin{equation}
\label{physics-ener1}
V(\data) = \sum_n v(\data(n)) .
\end{equation}
It enforces no interactions between different sites $n$ but
favors values of $\data(n)$ where $v$ is nearly minimum.
Ferromagnetism and second order phase transitions are
captured by the $\vvarphi^4$ model \cite{10.1093/oso/9780198834625.001.0001}.
Its scalar potential $v(t) = t^4 - (1+2\beta)t^2$ is non-convex, with  a double-well which pushes the values of each $\data(n)$ towards $+1$ or $-1$ \citep{10.1093/oso/9780198834625.001.0001}.

\begin{figure}
\centering
\includegraphics[width=0.7\textwidth]{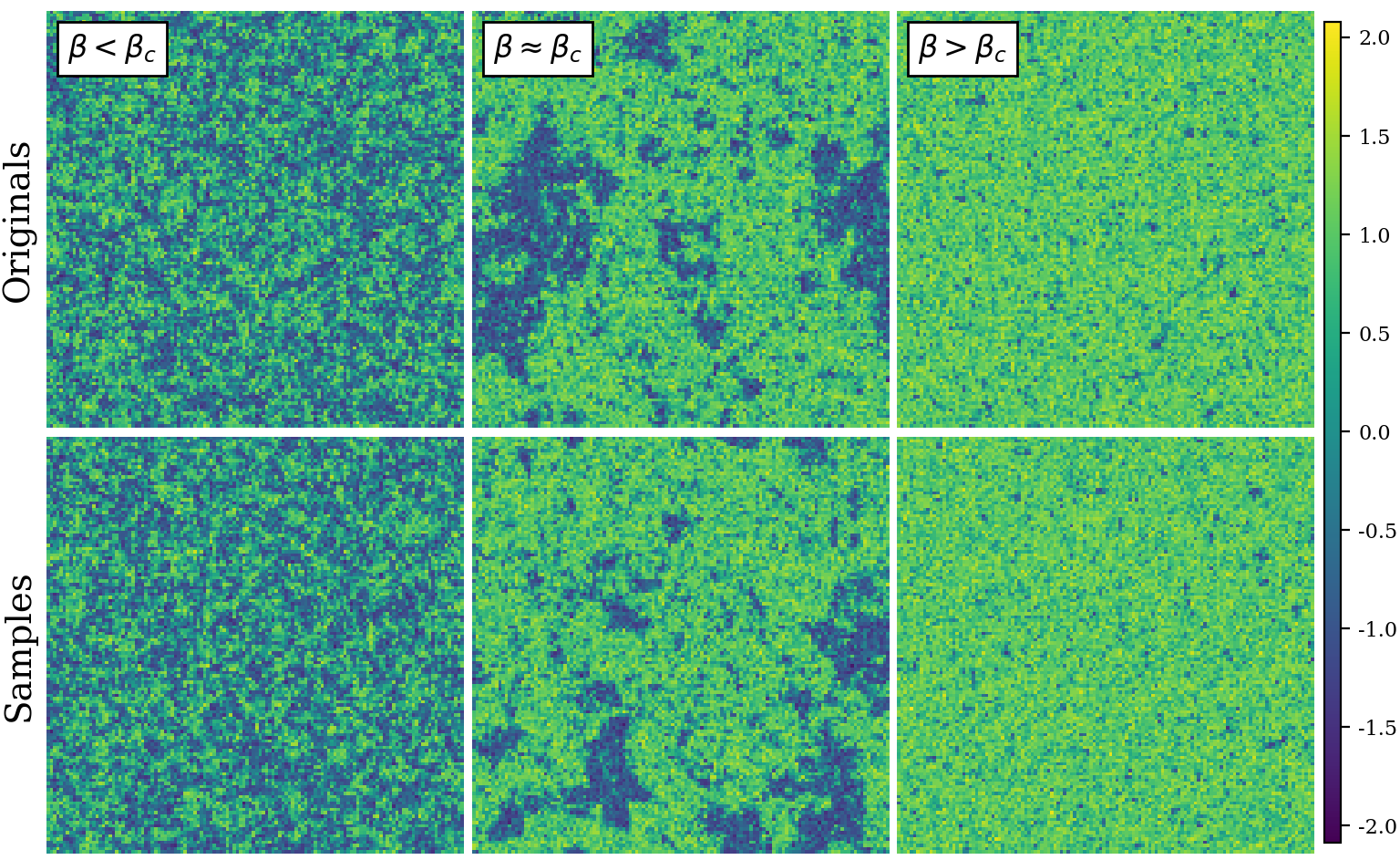}
\includegraphics[width=0.25\textwidth]{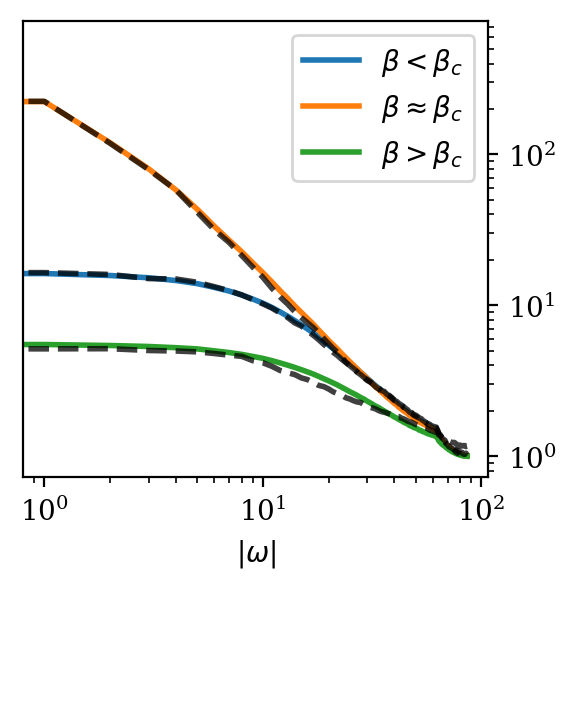}

~~~~~~(a)~~~~~~~~~~~~~~~~~~~~~~~~~~~~~~~~~~~~(b)~~~~~~~~~~~~~~~~~~~~~~~~~~~~~~~~~~~(c)~~~~~~~~~~~~~~~~~~~~~~~~~~~~~~~~~~~~~~~~~~~~~~(d)~~~~~~~
\caption{Top Row : Realisations of $\vvarphi^4$ fields at temperature $1/\beta$, and system size $d=128^2$. (a): For $\beta =0.5 < \beta_c$ , the system is disordered with short range correlations. (b): At the phase transition, $\beta = 0.68\approx  \beta_c$, the field is
  self-similar, with long range correlations. (c): For $\beta = 0.76 > \beta_c$, the system is in a ferromagnetic phase, with a non-zero mean (here positive). Bottom row: samples generated with a hierarchic factorization in a Haar
  wavelet basis with the same $\beta$ in (a,b,c).
  (d) : The graph shows the covariance eigenvalues (power spectrum) in these  $3$ cases, as a function of the two-dimensional frequency radius $|\omega| = (|\om_1|^2 + |\omega_2|^2)^{1/2}$. For $\beta = \beta_c$, eigenvalues have a power-law decay and develop a singularity at low frequencies, which correspond to long-range correlations. We superimposed in dashed line the covariance eigenvalues of a hierarchic model estimated by score matching. For visualization, the spectrum at different temperatures are multiplied by a constant which aligns their minimum eigenvalue.}
\label{fig:varphi4synthesis}
\end{figure}

If $\beta = 0$ then 
$\En(\data) = \sum_n v(\data(n))$. Each $\data(n)$ are then 
i.i.d independent random variables of density $\tilde p (t) = \eta\, e^{-v(t)}$. The power spectrum is constant. 
We saw in (\ref{product}) that the independence implies 
that the log-Sobolev constant satisfies $c(p) = c(\tilde p)$, and thus does not depend upon the dimension of $\data$. 
If $\beta > 0$ the Laplacian correlates pixels over a progressively larger neighborhood as $\beta$ increases. It increases the power spectrum at low frequencies. 
In the thermodynamic limit $d \to \infty$ of infinite system size, the $\vvarphi^4$ energy has a phase transition at $\beta_c\approx 0.68$ \citep{doi:10.1142/S0129183116501084}.
The power spectrum then has a power $|\om|^{-\eta}$ with $\eta = 1.75$ as shown
in Figure \ref{fig:varphi4synthesis}(d). It is singular at low-frequencies, which corresponds to a field having long range correlations. 
Figure \ref{fig:varphi4synthesis} shows realizations of $\vvarphi^4$ fields for $\beta < \beta_c$, $\beta = \beta_c$ and $\beta > \beta_c$.
For $\beta > \beta_c$ (low-temperature), there are two phases where the average field value is strictly positive or
strictly negative, which explains ferromagnetism.
Figure \ref{fig:varphi4synthesis}(d) corresponds to one phase where
most field values are close to
$1$.

\paragraph{Hessians eigenvalues}
The probability density $p$ has a non-convex energy $\En$ whose Hessian is
\[
\Hess_{\data} \En(\data) = - \beta \Delta + {\rm diag} (\mu_n)_n ~~\mbox{with}~~\mu_n = v''(\data(n)) . 
\]
The Laplacian $\Delta$ is diagonal in the Fourier basis with eigenvalues
$|\omega|^2.$
The scalar potential is diagonal in a Dirac basis with positive and negative eigenvalues. These eigenvalues are much larger than the Laplacian eigenvalues at low-frequencies and produce
eigenvectors of the Hessian $\Hess_{\data} \En$ having negative eigenvalues.

To define a hierarchic model with conditional probabilities $\bar p_j (\bar \varphi_j | \varphi_j)$ having small
log-Sobolev constants, Section \ref{sec:basis-choice} explains that
we may choose projectors $\bar G_j$ that select eigenvectors having high amplitude positive  eigenvalues, and discard negative eigenvalues. At the finest scale, 
this can be done \cite{guth2023conditionally} with
a first high-frequency filter $\bar G_1$ which eliminates low frequencies and 
selects high-frequency variables $\bar \data_1$. The resulting interaction
energy $\bar \En_1$ has a projected Hessian
\[
\Hess_{\bar \data_1} \bar \En_{1}(\data) = -\beta \bar G_1 
\Delta  \bar G_1 \trans +  \bar G_1 \, ( \diag (\mu_i)_i
) \, \bar G_1 \trans~ . 
\]
Figure \ref{fig:compareeigenvalue-histo}(a)
compares the histograms of the eigenvalues of $\Hess_{\data_{0}} \En_{0}$ 
and the Hessian  $\Hess_{\bar \data_1} \bar \En_1$ (without normalization for comparison purposes). It is computed with a Symlet-4 wavelet for $\vvarphi^4$ at critical temperature $\beta = \beta_c$. 
As expected, $\Hess_{\bar \data_{1}} \bar \En_{1}$, which is equal to $\bar G_1\Hess_{\data_{0}} \En_{0}\bar G_1\trans$,  has fewer negative eigenvalues
than $\Hess_{\data_{0}} \En_{0}$, but some still remain. These negative
eigenvalues can be almost everywhere eliminated by an orthogonal $\bar G_1$, which selects a more narrow high-frequency band  than wavelets.
This can be done with wavelet packets \cite{guth2023conditionally}. Since $\vvarphi^4$ has a self-similar probability distribution at the phase transition $\beta = \beta_c$, the same result is obtained at all other scales $2^j$.

\begin{figure}
\centering
\includegraphics[width=0.405\textwidth]{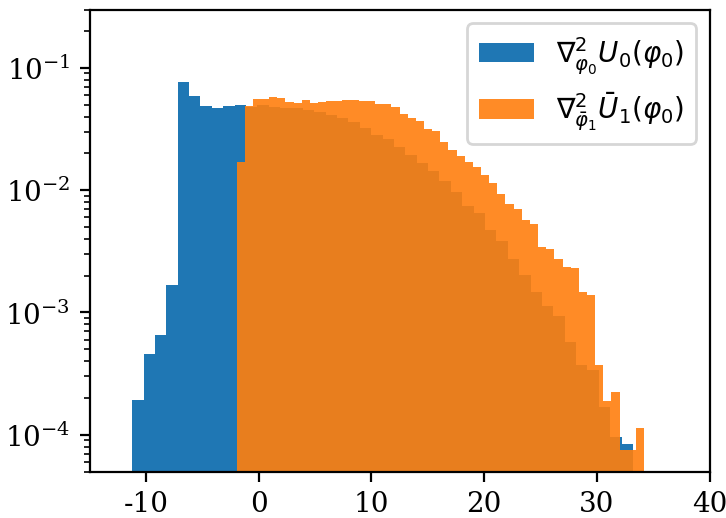}
\includegraphics[width=0.49\textwidth]
{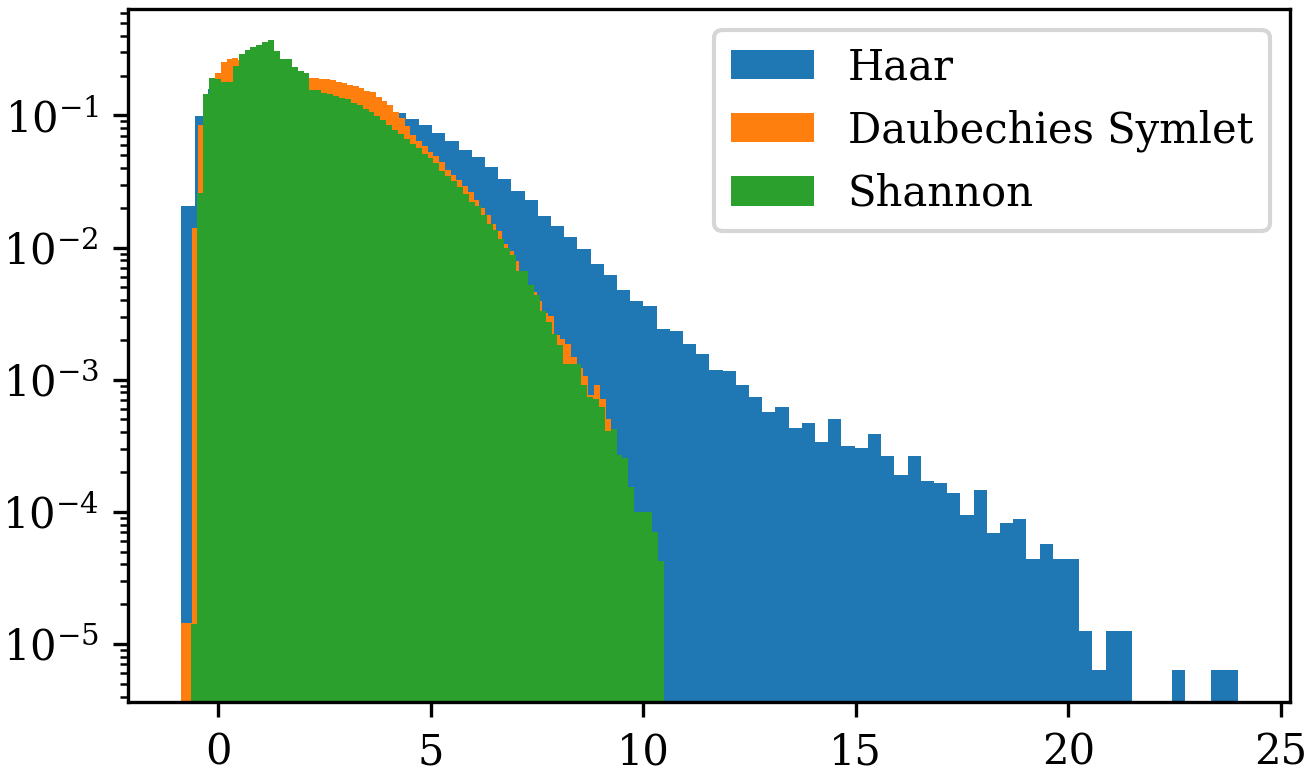}\\
~~~~~~~~~~~~~~~(a)~~~~~~~~~~~~~~~~~~~~~~~~~~~~~~~~~~~~~~~~~~~~~~~~~~~~~~~~~~~~~~~~~~~~~~~~~~~~~~~~~~~~(b)~~~~~~~~~~~~
\caption{(a): We consider  images $\data_0$ of the  $\vvarphi^4$ model of dimension $d=128^2$, for a critical $\beta = \beta_c$. The approximations $\data_1$ are calculated with a Symlet-4 filter at the finest  scale $2^1$. The graph shows the distributions of eigenvalues of the Hessian
  $\Hess_{\data_{0}} \En_{0}$
  in blue, and of $\Hess_{\bar \data_1} \bar \En_1=
  \Hess_{\bar \data_1}  \En_0 $ (without renormalization) in orange.  The most negative eigenvalues of $\Hess_{\data_{0}} \En_{0} $
  correspond to low frequency eigenvectors. They do not appear in
  $\Hess_{\bar \data_1} \bar \En_1$.
  (b): Distributions of eigenvalues of Hessians
  $\Hess_{\bar \data_j} \bar \En_{\theta_j}$ at all scales $2^j \leq 2^J$, for $d=32^2$, with $J=3$, for samples of
  hierarchical models of $\vvarphi^4$ at phase transition. They are computed for Haar (blue), Symlet-4 (orange) and Shannon wavelets (green). Eigenvalues are more concentrated when the wavelet has a better frequency
 localization.}
\label{fig:compareeigenvalue-histo}
\end{figure}

Figure \ref{fig:compareeigenvalue-histo}(b)
displays the distribution of eigenvalues of
$\Hess_{\bar \data_j} \bar \En_{\bar \theta_j} (\data_{j-1})$,
for samples $\data_{j-1}$ of a hierarchic model, computed 
at all scales $2^j \leq 2^J$. These distributions are calculated
for hierarchical factorization computed
with Haar (\textit{blue}), Symlet-4 (\textit{orange}) and Shannon (\textit{green}) wavelets. The Hessian eigenvalues are nearly the same for Shannon wavelets and Symlet-4. For Haar wavelets there are much more high amplitude eigenvalues. Indeed, Haar wavelets are not as well localized in frequency. For a Haar wavelet, $|\hat \psi(\om)|^2$ decays like $|\om|^{-2}$ at high frequencies because $\psi(x)$ is discontinuous. The high amplitude eigenvalues are due to this slow high-frequency decay, which compensate, with slow vanishing, for the growths of the Hessian eigenvalues proportional to $|\omega|^\eta$ for $\eta = 1.75$ \cite{tauber2014critical}.

The existence of negative Hessian eigenvalues prevents using the Bakry-Emery theorem to compute an upper bound on the log-Sobolev constant of wavelet conditional probabilities. However, we shall see in the next section that these remaining negative eigenvalues do not prevent the Langevin diffusion from exponential convergence, even at the phase transition.
These numerical results are an indication that log-Sobolev constant of
wavelet conditional probabilities do not depend upon the scale.

\subsection{Learning and Sampling hierarchic Scalar Potential Energies}
\label{subsec:scalarpots}

This section reviews the estimation of hierarchic models of scalar potentials 
introduced in \citep{marchand_wavelet_2022}, 
and its application to the estimation and sampling of the
$\vvarphi^4$ model at critical temperature. It is shown in \citep{marchand_wavelet_2022} that the critical slowing down
disappears when sampling
the conditional probabilities of a hierarchic factorization, although the
Hessians still have negative eigenvalues as presented in Figure \ref{fig:compareeigenvalue-histo}.  We compare different wavelets 
for learning and sampling the $\vvarphi^4$ model
at critical temperature. Model errors 
decrease when decreasing the wavelet support because
conditional distributions of wavelet coefficients $\bar\data_j$ have shorter
dependencies in space and scales. The precision of low dimensional models
versus the convergence rate of sampling introduces a trade-off between spatial
versus frequency localization. It justifies the use of a wavelet basis as opposed to a Fourier basis, and the Haar wavelet basis is the winner.

\paragraph{Hierarchic scalar potentials}
A hierarchic model is defined with a coarse scale model
$\En_{\theta_J}  = \theta_J\trans \Phi_J$ and 
interaction energy models
$\bar \En_{\bar \theta_{j}}   =  \bar \theta_j \trans \Psi_j$ at each scale $2^j \leq 2^J$.
We define $\Phi_J$ and each $\Psi_j$ for scalar potential energies, and prove that it defines a stationary hierarchic model, whose coupling parameters are
computed with a coarse to fine coupling flow equation.

At the coarsest scale, $U_{\theta_J} = \theta_J\trans \Phi_J$
includes a two-point interaction matrix and a parametric scalar potential
\begin{equation}
\label{inter-param}
\theta_J \trans \Phi_J(\data_J) =\frac 1 2 \data_J\trans  K_J \data_J
+  V_{\gamma_J} (\data_J)    ,
\end{equation}
 where $V_\gamma (\data)$ has a scalar potential $v_{\gamma}(t) = \sum_k \gamma_k\, \rho_k (t)$ decomposed over a finite approximation family $\{ \rho_k (t) \}_k$ with coefficients $\gamma = (\gamma_k )_k$. It results that
\begin{equation}
\label{scalar-pot-vec}
V_\gamma (\data) = \gamma\trans \Gamma(\data)~~\mbox{with}~~
\Gamma(\data) = \Big( \sum_n \rho_k (\data(n)) \Big)_k .
\end{equation}
We  use translated sigmoids which do not grow at infinity:
$\rho_k (t) = (1+\exp((t-t_k)/\sigma_k))^{-1}$.
In numerical calculations there are $40$ evenly spaced translations $t_k$, on the support of the distribution of each $\data_J(n)$, and $\sigma_k=\frac{3}{2}(t_{k+1}-t_k)$. 
Defining a model of $p_J$ which is stationary is equivalent to 
imposing that $K_J$ is a convolutional operator.

The interaction Gibbs energy $\bar U_{\bar \theta_j}$ of $p_j( \bar \data_j | \data_j)$
includes two-point interactions within the 
high frequencies $\bar \data_j$, between high frequencies $\bar \data_j$
and the lower frequencies $\varphi_j$,
with convolution matrices $\bar K_j$ and $\bar K'_j$, 
plus a scalar potential
\begin{equation}
\label{scalar-interaction}
\bar \En_{\bar \theta_{j}} (\data_{j-1}) =
 \bar \data_j\trans   \bar K_{j}  \bar \data_{j} +
\bar \data_j\trans   \bar K'_{j}  \data_{j} +
\bar V_{\bar \gamma_j} ( \data_{j-1}) = \bar \theta_j \trans\Psi_j (\data_{j-1})~.
\end{equation}
It defines
\begin{equation}
\label{scalar-interaction00}
\bar \theta_j = \left(
\begin{array}{c} 
\bar K_j\\
\bar K_j'\\ \bar \gamma_j
\end{array} \right)
~~
\mbox{and}~~
\Psi_j (\data_{j-1}) = 
\left(
\begin{array}{c}
\bar \data_{j}  \bar \data_j\trans \\
\bar \data_j \data_j\trans \\
\Gamma(\data_{j-1})
\end{array} \right).
\end{equation}
The stationary interaction $\bar K_j'$ between high and low-frequencies 
has an energy contribution which is typically much smaller than the interaction $\bar K_j$ within high frequencies, because $\varphi_j$ and $\bar \varphi_j$ are computed over frequency domains having a small overlap.  The following theorem defines a stationary
hierarchic model from this scalar potential interactions $\Psi_j$. 
With an abuse of notation, we write $\varphi * \varphi\trans$ the convolution between $\varphi$ and  $\varphi\trans(n) = \varphi(-n)$.
We recall from (\ref{filters2}) that $\varphi_j(n) = \varphi * \phi_j (2^j n)$.

\begin{theorem}
\label{pro:scalar-embedded-potential}
For any $j \geq J$,
\begin{equation}
\label{potential-equations}
\Phi_j (\data_j) = \left(
\begin{array}{c}
\data_j* \data_j\trans \\
\Gamma ( \data_j * \phi_\ell )
\end{array} \right)_{J- j \geq \ell \geq 0} 
\end{equation}
defines stationary hierarchic  potentials with interactions $\Psi_j$ in (\ref{scalar-interaction}).
Regressing each free energy $F_j$ over $\Phi_j$ defines a fine scale Gibbs stationary model for $j=0$
\begin{equation}
\label{fine-snsdof}
 \En_{\theta} (\data) = 
\frac 1 2 \data\trans K \data + \sum_{j=0}^J   V_j (\data)~~\mbox{with}~~
 V_j (\data) = 
\gamma_j\trans \Gamma(\data * \phi_j) ~,
\end{equation}
where $K$ is a convolution matrix and $(\gamma_j)_{0 \leq j \leq J}$
are scalar potential parameters computed by a coupling flow equation.
\end{theorem}

The proof is in Appendix \ref{app:pro:scalar-embedded-potential}.
At each scale $2^j$, this Gibbs energy has a different scalar potential
$ V_j (\data) = \gamma_j\trans \Gamma(\data * \phi_j)$. The convolution
with $\phi_j$ averages $\varphi$ over a domain proportional to $2^j$. As the scale $2^j$ increases, it becomes more and more non-local. 
Scalar potential energies (\ref{physics-ener}), such as the
$\vvarphi^4$ model, have a single potential $V_0(\data) = \sum_n v_0(\data(n)) $
at the finest scale $j=0$ and are thus local. They correspond to a particular case where 
$\gamma_j = 0$ for $j < 0$.
However, a single fine scale scalar potential is not always sufficient.
This is the case of cosmological weak-lensing fields, 
which can be approximated by incorporating
different scalar potentials $V_j$ at different scales. In this case,
numerical results show that hierarchic scalar potential models
provide accurate approximations of $\En$ \citep{guth2023conditionally,marchand_wavelet_2022}.

Computing the fine scale energy $U_\theta$  in (\ref{fine-snsdof}) is useful to carry a qualitative physical analysis of the estimated energy \cite{brossollet2024effectiveenergyinteractionsequilibrium}. 
However, this calculation is unstable.
Section \ref{sec:estim-gibbs} explains that the number of samples needed to achieve a given accuracy is proportional to the normalized log-Sobolev constant. Reducing the log-Sobolev constant with a hierarchical flow allows to learn precise approximations of coupling parameters $\bar \theta_j$ at each scale,
with less data. The inverse renormalization equation (\ref{prop-eq3})
involves the calculation of free energy parameters $\alpha_j$ and errors
are amplified by
the inverse renormalization matrices $Q_j$ which progressively
computes the aggregated model parametrized by $\theta$.
The resulting fine scale energy model $U_\theta$ can therefore have important errors, which may define a non-integrable measure.
Numerical experiments demonstrate that sampling directly this
Gibbs energy is computationally slow and may diverge \cite{marchand_wavelet_2022}. 
The renormalization group representation based on interaction parameters
$\bar \theta_j$ provides a more accurate parametric model, a faster and much more reliable sampling of the estimated probability distribution.

\begin{figure}
\centering

\includegraphics[width=0.8\textwidth]{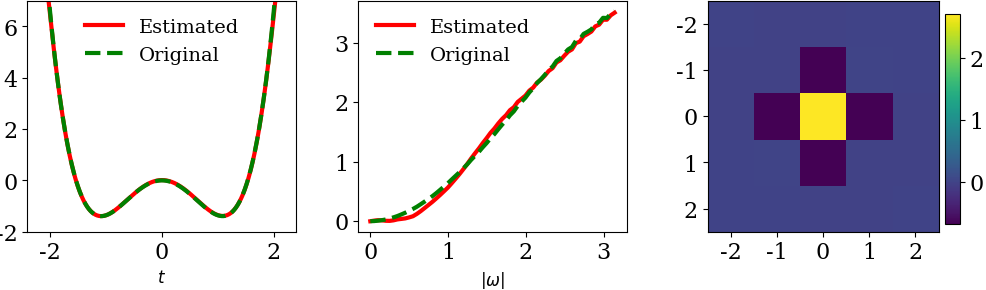}

(a)~~~~~~~~~~~~~~~~~~~~~~~~~ ~~~~~~~~~~~~~~~~~~~~(b)~~~~~~~~~~~~~~~~~~~~~~~~~~~~~~~~~~~~~~~~~~~~~~(c)

\caption{Original and estimated energy of $\vvarphi^4$ at critical temperature, for images of size $d=128 \times 128$.
  (a): Superposition of the scalar potential $v(t)$ of the $\vvarphi^4$ model for $\beta = \beta_c$ and the estimated scalar potential $v_{0} (t)$ of the
  hierarchical stationary energy model $\En_{\theta}$. (b): Superposition of the Laplacian eigenvalues (in the Fourier basis) and the eigenvalue of 
  the estimated 2-point interaction matrix $K$ of $\En_{\theta}$. (c): Estimated convolution kernel of $K$. These result show that
  the hierarchical stationary model calculated in a Haar basis gives a precise
  approximation of the $\vvarphi^4$ energy.}
\label{fig:MultiscaleEnergies}
\end{figure}

\paragraph{Hierarchic model estimation and sampling}
The parameters $\{\theta_J , \bar \theta_j\}_{j \leq J}$ are
estimated from $m$ samples $\{ \data^{(i)} \}_{i\leq m}$ of $p$, with the
conditional score matching algorithm of Section \ref{sec:cond-estim}.
Samples $\data$ of $p_{\theta} =  w^{-1} p_{\theta_J} \prod_j \bar p_{\bar \theta_j}$ are computed with the hierarchical sampling algorithm of Section \ref{sec:sampling}, which does not require the knowledge of the free energies, or normalizations, of the $\bar p_{\theta_j}$.
The MALA algorithm includes a rejection of Langevin diffusion propositions. The scalar potential also rejects proposals outside a high probability interval. 
Figure \ref{fig:varphi4synthesis} compares original samples from $p$ computed with exact $\vvarphi^4$ energies at different temperatures, and samples of a hierarchical model 
$\bar p_{\theta}$ estimated in a Haar
wavelet basis. Generated images have textures which
cannot be distinguished visually from the original image textures. 

The model precision can be evaluated by computing the resulting
stationary energy and by comparing it with the true $\varphi^4$ energy.
A hierarchic stationary model (\ref{fine-snsdof}) of $\vvarphi^4$ 
is calculated in the Haar basis from
the estimated interaction energy parameters $\bar \theta_j$ and the
free energy parameters $\alpha_j$.
The only non-zero scalar potential is at the finest scale $j=0$.
It implies that the scalar potential of the free energy $F_j$ cancels the
scalar potential of the energy at the previous scale. 
Figure \ref{fig:MultiscaleEnergies} compares the estimated
energy $U_\theta$ and original $\vvarphi^4$ energy. 
Figure \ref{fig:MultiscaleEnergies}(a) compares the estimated $v_0(t)$ and original
scalar potential function $v(t)$.
Figure \ref{fig:MultiscaleEnergies}(b) compares
the eigenvalues of the estimated two-point interaction
matrix $K$ and of a Laplacian, which is the two-point interactions of the $\varphi^4$ model. 
Figure \ref{fig:MultiscaleEnergies}(c) shows that the convolution
kernel of $K$ is indeed close to a Laplacian.
It shows that the hierarchic stationary model in a Haar wavelet basis gives an accurate
approximation of the $\vvarphi^4$ energy model at the phase transition.
We shall see at the end of this section that the estimation error becomes larger with Daubechies and Shannon wavelets, which have a spatial support larger than Haar wavelets.

\begin{figure}
\centering
\includegraphics[width=0.78\textwidth]{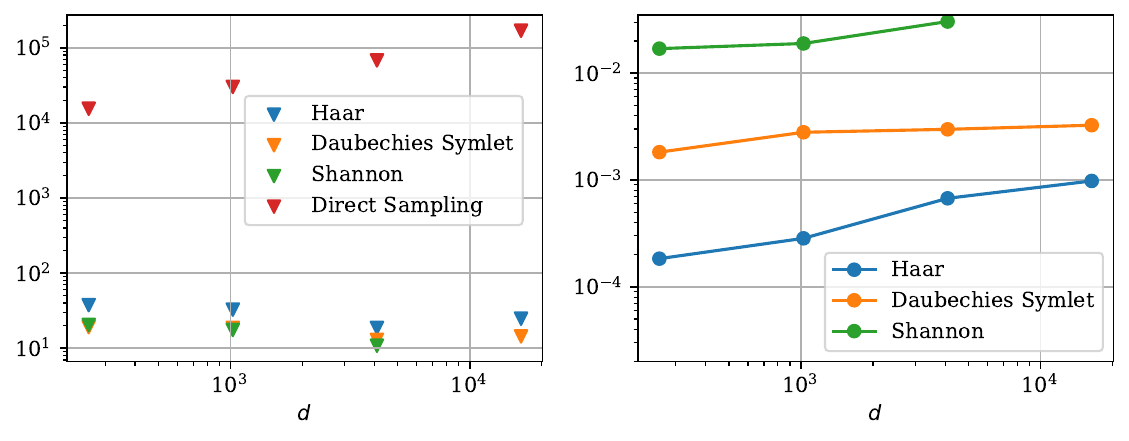}

~~~~~~~~~~(a)~~~~~~~~~~~~~~~~~~~~~~~~~~~~~~~~~~~~~~~~~~~~~~~~~~~~~~~~~~~~~~~~~~~~~~(b)~~

\caption{ (a): Normalized Langevin auto-correlation relaxation time  
of the $\vvarphi^4$ model at critical temperature and for hierarchical models computed in different wavelet bases. In red is shown the normalized relaxation time of a Langevin applied directly on the energy
  of the $\vvarphi^4$ model, which grows with the image dimension $d$. It illustrates the  critical slowing down. On the contrary, it remains constant for  all wavelet hierarchical models. This constant depends upon the log-Sobolev constant of wavelet conditional probabilities, which depends upon the wavelet choice. The hierarchic normalized relaxation time decreases when the wavelet has a Fourier transform which is better concentrated. It is maximum for Haar and smaller for Daubechies-4 Symlet and Shannon wavelets.
  (b): Approximation error of hierarchical models computed in the same bases as in (a). This error measures the difference between the marginal and second order moments of samples of $\vvarphi^4$ and the same moments computed from samples of hierarchical models in each wavelet basis. It is quantified with a KL divergence calculated in (\ref{moment-error}). The error decreases when the wavelet support decreases because it is dominated by the estimation error of the non-convex scalar potential which is local in space. It is much smaller for a Haar wavelet than for Symlet$-4$ wavelets. The Shannon wavelet which has a non-compact support with a slow spatial decay yields a much larger error than Haar and Symlet$-4$ wavelets.}
\label{fig:langevin-mixing}
\end{figure}

\paragraph{Critical slowing down at phase transition}
Sampling $p = {\cal Z}^{-1} e^{-\En}$ with an unadjusted Langevin algorithm has a computational complexity proportional to the number of iterations after discretization. It is proportional to the log-Sobolev constant $c(p)$
multiplied by the largest eigenvalue of the Hessian of $U$, as explained
in Section \ref{sec:langevin-sampling-logSobolev}.
For $\beta = \beta_c$, it suffers from a critical slowing down
due to a growth of the log-Sobolev constant when the system size $d$ increases. The convergence of Langevin diffusion is 
estimated by the auto-correlation relaxation time  defined
in (\ref{relax-time}). The computational complexity is evaluated by the normalized relaxation $\tau$ equal to 
auto-correlation relaxation time divided by the discretization time step.
Figure \ref{fig:langevin-mixing}(a)
gives the evolution of this normalized relaxation time $\tau$ as a function of the system size $d$ for $\beta = \beta_c$.
It grows like $d^{\eta_0/2}$ with $\eta_0 = 2$ \citep{Podgornkik96,Sethna2021StatisticalME,tauber2014critical}.
This behavior is partly explained by the log-Sobolev lower-bound 
$c(p) \geq \mu_{\max} / 2$ where $\mu_{\max}$ is the largest eigenvalue
of the covariance. It grows like $d^{1.75/2}$ 
for $\beta = \beta_c$, as shown by
the power spectrum in Figure \ref{fig:varphi4synthesis}.
However, this explanation is not complete since the covariance is only a lower-bound. Furthermore, the log-Sobolev constant is equal to the autocorrelation relaxation time only if the distribution is gaussian. The autocorrelation relaxation time has a faster growth exponent where $1.75$ is replaced by $2$ which gives $d$. Indeed, it also
suffers from the non-convexity of the scalar potential, which is not
captured by the covariance lower-bound. 

It has been shown in \cite{marchand_wavelet_2022} that a hierarchical factorization in
a wavelet basis avoids this critical slowing down. To compute the normalized auto-correlation relaxation time of the hierarchical sampling algorithm, 
for each $j$ we numerically estimate
the relaxation time $\bar\tau_j$ of probability $\bar p_{\bar\theta_j}(\bar \data_j |\data_j)$, like in (\ref{relax-time}). The hierarchic normalized auto-correlation relaxation time $\tau$, is defined by
\begin{equation}
\label{eq:autocorrelation}
\tau = \sum\limits_{j=1}^{J} \frac{{\bar{d}_j}}{d}\, \bar{\tau}_{j} +\frac{d_J}{d}\, \tau_{J}~,
\end{equation}
where $(\bar d_j,d_j)$ are the dimensions of $(\bar\data_j,\data_j)$.  Appendix \ref{sec:mixingtimelangevin} explains how
to estimate numerically
the relaxation time of each conditional probability,
with an Unadjusted Langevin Algorithm.
Each normalized relaxation time $\bar\tau_j$ is divided by the discretization time-step.  To evaluate the overall computational complexity, each $\bar \tau_j$ is multiplied by the relative size $\bar d_j / d$ of the gradient $\nabla_{\bar \data_j} \bar U_{\bar \theta_j}$.
Figure \ref{fig:langevin-mixing}(a) gives the hierarchic normalized auto-correlation relaxation time, depending on the system size $d$, for different wavelet basis.
For Haar, Daubechies Symlets and Shannon wavelets,
Figure \ref{fig:langevin-mixing}(a) shows that hierarchic normalized auto-correlation relaxation times do not increase with
the dimension $d$, what verifies that they do not suffer from the phase-transition critical slowing down. It reproduces the
absence of critical slowing down observed in \citep{marchand_wavelet_2022}
with a Metropolis-Hasting sampling, as we know that the MCMC mixing time
tends to an unadjusted Langevin diffusion in the continuum time limit \cite{andreanov2006field,gelman1997weak}.

These experiments give
a strong indication that log-Sobolev constants of wavelet conditional
probabilities are uniformly bounded independently of $d$, despite the fact that the energy Hessians have negative eigenvalues. This is a mathematical conjecture which has not been proved.
Calculations of the log-Sobolev constant of $\vvarphi^4$ have been carried for temperatures above the critical temperature \citep{bauerschmidt2022log,bauerschmidt2023stochastic,chen2022localization,bauerschmidt2019very}, but have not been extended up to the phase transition.
As expected, Figure \ref{fig:langevin-mixing}(a) also shows that hierarchic normalized auto-correlation relaxation time becomes smaller when improving the frequency localization of wavelets. Shannon wavelets have more vanishing moments and are more regular
than Daubechies Symlets which are themselves better localized in frequency
than Haar wavelets. Because the Haar wavelet has a poor frequency localization, the  coarse graining does not eliminate all the high frequency from $ \varphi_{j-1}$, which are responsible for big eigenvalues in $\nabla^2_{\bar\varphi_j}\bar U_j$. This tail, observed Figure \ref{fig:compareeigenvalue-histo}(b) requires reducing the time sampling step of the Langevin dynamic, and it increases the normalized relaxation times.

\paragraph{Energy estimation error}
Approximating scalar two-point potential energies requires to accurately approximate
both the kinetic energy term and the scalar potential
in (\ref{scalar-interaction}). The Hessian of the kinetic energy is nearly a discretized Laplacian, which is diagonal in a Fourier basis with positive eigenvalues. The wavelet representation of a Laplacian is not diagonal but preconditioned by its diagonal in a wavelet basis. It is sufficient to eliminate the bad conditioning with a renormalization, and obtain
a low-dimensional approximation with a band matrix. The scalar potential  is non-convex, with a Hessian which is diagonal in a Dirac basis and has negative eigenvalues. To build a low-dimensional model we must
preserve this spatial localisation by using wavelet
filters having a small spatial support.
A Fourier representation of this scalar potential introduces long range dependencies between Fourier coefficients, which leads to higher dimensional coupling models that introduce more estimation errors.

Hierarchical models can be sampled without estimating the free energies of
conditional probabilities. To evaluate the model precision without estimating
the free energies, we measure estimation errors by computing the 
model errors on a set of sufficient statistics for the 
$\varphi^4$ energy.
These moments are estimated by Monte Carlo, over samples
of $\varphi^4$. They are compared
with the moment computed from samples of hierarchic models.
The sufficient statistics are defined by 
second order moments and by the marginal distribution of $\data(n)$.
Appendix \ref{app:phi4-metric} computes 
in (\ref{moment-error}) a Kullback divergence 
error $e(p,p_\theta)$ which adds a Kullback divergence error
calculated from second order moments and a
divergence calculated from marginal distributions.
Figure \ref{fig:langevin-mixing}(b) gives the value of the moment error
for hierarchical models computed with Haar, Symlet-4 and Shannon wavelets.
These models are learned with large enough data sets so that the variance of statistical estimators are negligible.
The error is minimum for Haar wavelets because they have a support of
minimum size.
It is much larger for a Symlets-4 wavelet, whose support is $7$ times larger. For a Shannon wavelet, which has a slow spatial decay, the estimation error of marginal densities becomes extremely large. 
This error comes from longer range dependencies among wavelet coefficients when they are calculated with wavelets having a wider support, produced by the scalar potential. 

Wavelet bases seem to have a near-optimal trade-off to estimate the probability distribution of $\varphi^4$ at the phase transition while avoiding the critical slowing down. The Haar wavelet corresponding to Kadanoff's renormalizing group appears to be the best wavelet choice, thanks to its shortest support. It minimizes the model estimation error
in a low-dimensional parametric class, while avoiding the critical slowing down.

\section{Robust Multiscale High Order Interactions}
\label{sec:scat-cov}

Multiscale non-Gaussian random fields have
long range interactions across space and scales. In natural images, it appears through the existence of sharp transitions which propagate along piece-wise regular curves such as filaments or edges of objects. Non-Gaussian properties may be captured by higher order polynomials, but it typically leads to high-dimensional models and high variance estimators. 
Section \ref{sec:matrix-orga} introduces low dimensional models of multiscale
probability interactions. It defines robust approximations of high order  models.
Section \ref{sec:applications} studies numerical applications to
modeling and generation of dark matter densities and two-dimensional turbulent vorticities.

\subsection{Interactions over Multiple Hierarchies by Wavelet Scattering}
\label{sec:matrix-orga}

Hierarchic models decompose $p$ into a cascade of 
conditional probabilities across scales, which are rewritten
as the conditional probabilities of wavelet coefficients. 
In the following, we build models of such conditional probabilities,
by computing a complex wavelet transform which explicitly provides
a complex phase.
Non-Gaussian properties are captured with a second wavelet transform,
on the complex modulus of the first wavelet transform, leading to
low-dimensional models of long-range spatial dependencies and dependencies across scales.

\paragraph{Complex wavelet transform}
To model the  probability distribution of $\varphi_{j-1}$ conditioned on $\varphi_j$, we compute a complex wavelet transform of $\varphi_{j-1}$. 
The complex wavelet coefficients calculated from $\varphi_{j-1}$ can 
also be written as convolutions of $\varphi$ with complex wavelets $\tilde \psi_{j',k}$ at scales $2^{j'} \geq 2^j$. They have
$Q$ orientations indexed by $k$, sampled on the grid of $\varphi_{j-1}$ at intervals $2^{j-1}$
\[
\Big( \varphi * \tilde \psi_{j',k} (2^{j-1} n) \Big)_{ n }.
\]
Let $g$ be the low-pass filter of the coarse-graining operator $G$ which computes
$\varphi_{j-1}$ from $\varphi.$ Appendix \ref{app:wavelets} shows that
$\varphi * \tilde \psi_{j',k}$
is calculated from $\varphi_{j-1}$ with an a-trous algorithm. It
is a cascade of $j'-j -1$ convolutions of $g$, followed by convolutions with a family of complex band-pass filter
$\tilde g = (\tilde g_{k} )_{k \leq Q}$. These filters
are dilated by introducing zeros in between their coefficients.
In numerical applications, $\tilde g$ has $Q = 4$ Morlet filters specified in Appendix \ref{app:wavelets}. At each scale $2^j$, they define $4$ wavelets $\tilde \psi_{j,k}$ whose support is proportional to $2^j$. They
are approximately rotated by $0$, $\pi/4$, $\pi/2$ and $3 \pi / 4$.
Figure \ref{fig:complexfilter} shows the real and imaginary parts
of the wavelets $\psi_{j,k}$ for $j=3$ computed with these Morlet filters and the Symlet-4 filters $g$. These complex wavelets have a Hermitian symmetry $\tilde \psi_{j,k}(-n) = \tilde \psi_{j,k}^*(n)$. Their real and imaginary parts are therefore symmetric and antisymmetric. 
The lowest frequencies are retained by
the scaling filter $\phi_J$, that we write
$\phi_J = \tilde \psi_{J+1,k}$ to simplify notations.

\begin{figure}
    \centering
    \includegraphics[width=0.35\textwidth]{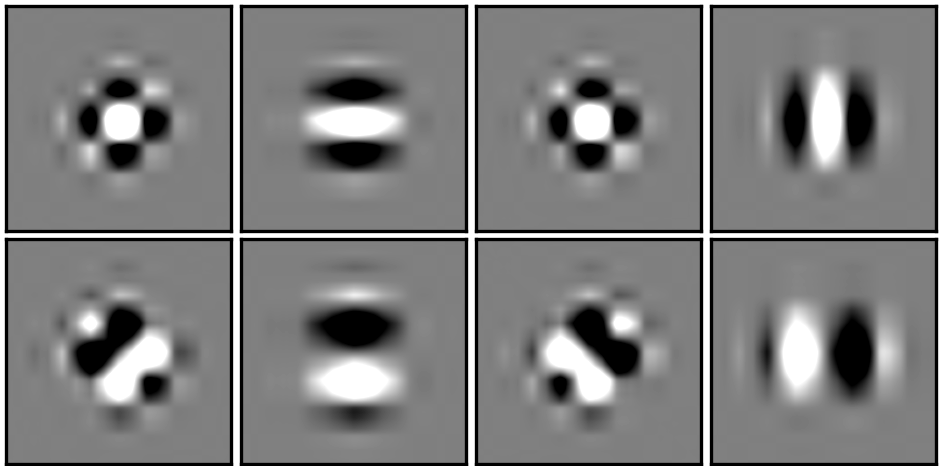}
    \caption{ 
    Complex wavelet $\tilde \psi_{j,k}$ computed with a 2D Symlet-4 low-pass filter $g$,  and $4$ oriented Morlet filters $(\tilde g_k)_{k\leq 4}$, at the scale $2^j=8$. The upper and lower rows show respectively the
      real and imaginary parts of $\tilde \psi_{j,k}$ , for $k=1,2,3,4$.}
    \label{fig:complexfilter}
\end{figure}

\begin{figure}
\centering
\includegraphics[width=0.17\textwidth]{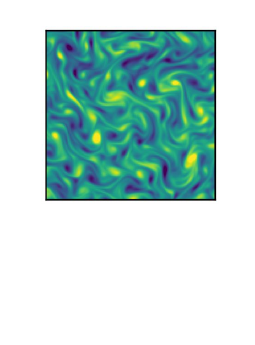}
\includegraphics[width=0.56\textwidth]{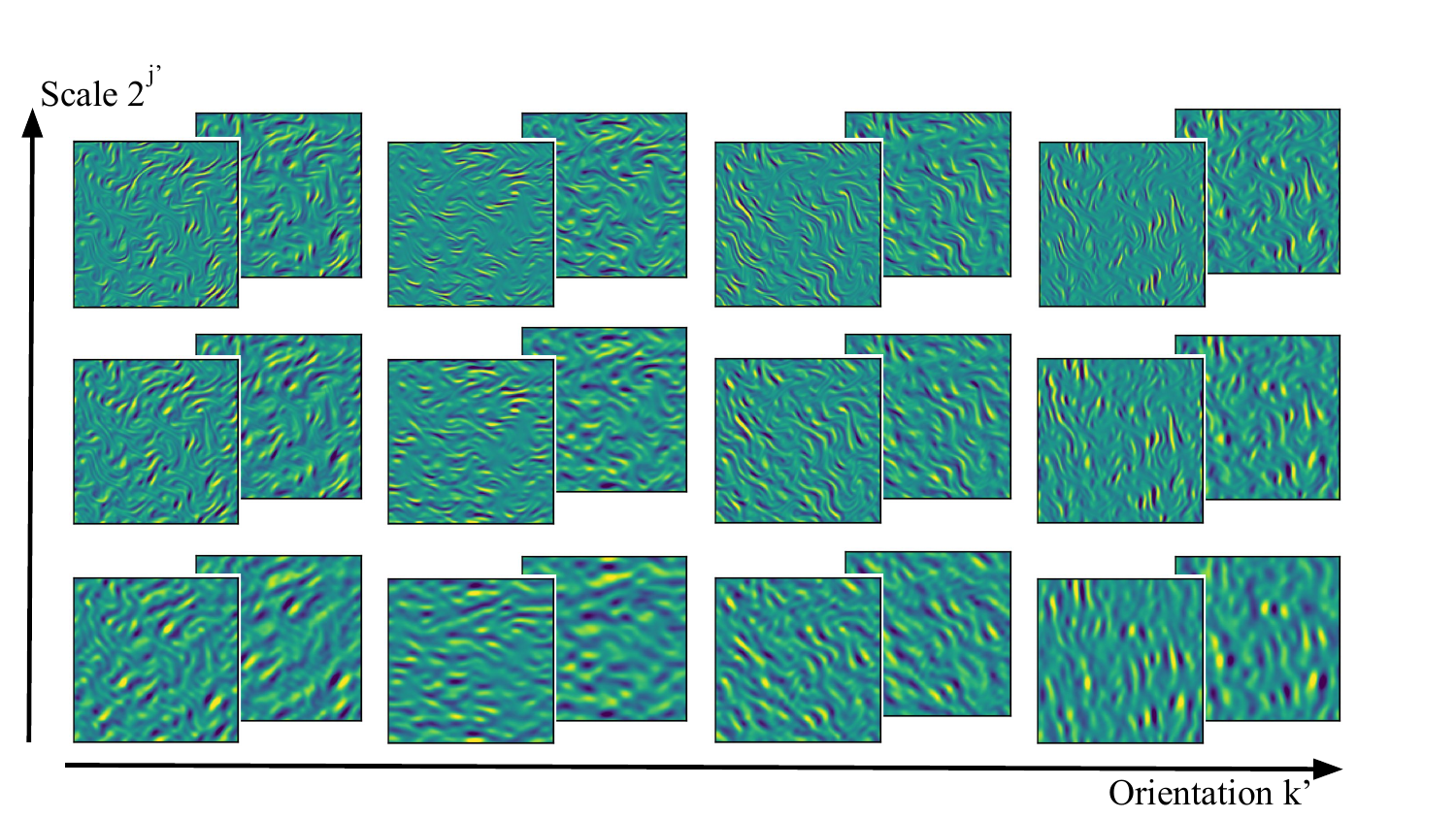}

~~~(a)~~~~~~~~~~~~~~~~~~~~~~~~~~~~~~~~~~~~~~~~~~~~~~~~~~~~~~~~~~~~~~~~~~~(b)~~~~~~~~~~~~~~~~~~~~~~~~~~~~~~~~~~~~~~

\includegraphics[width=0.42\textwidth]{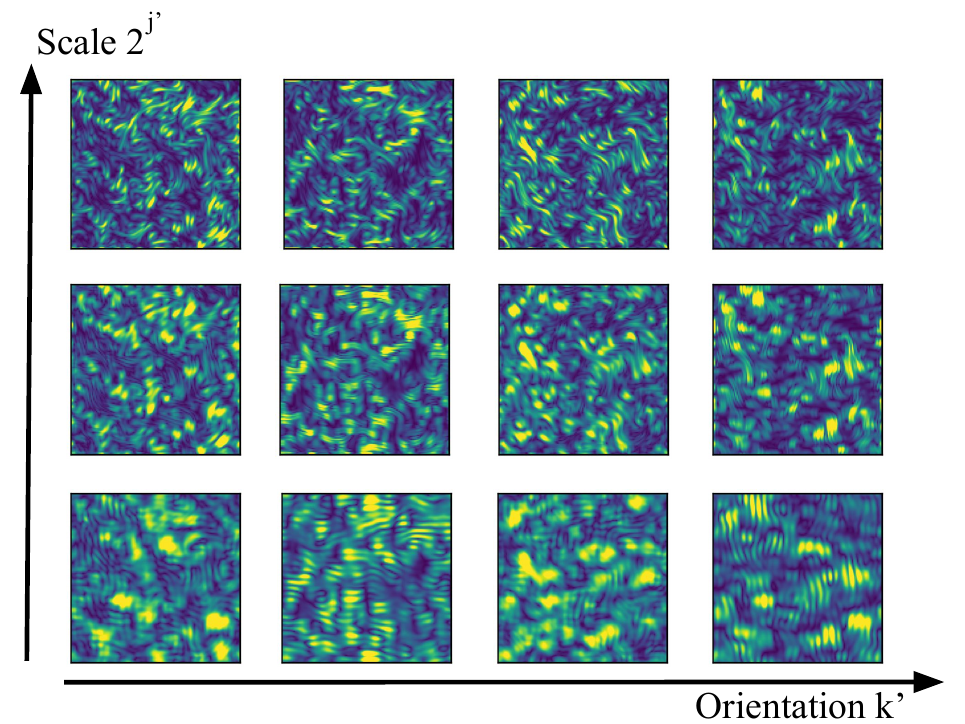}
\includegraphics[width=0.56\textwidth]{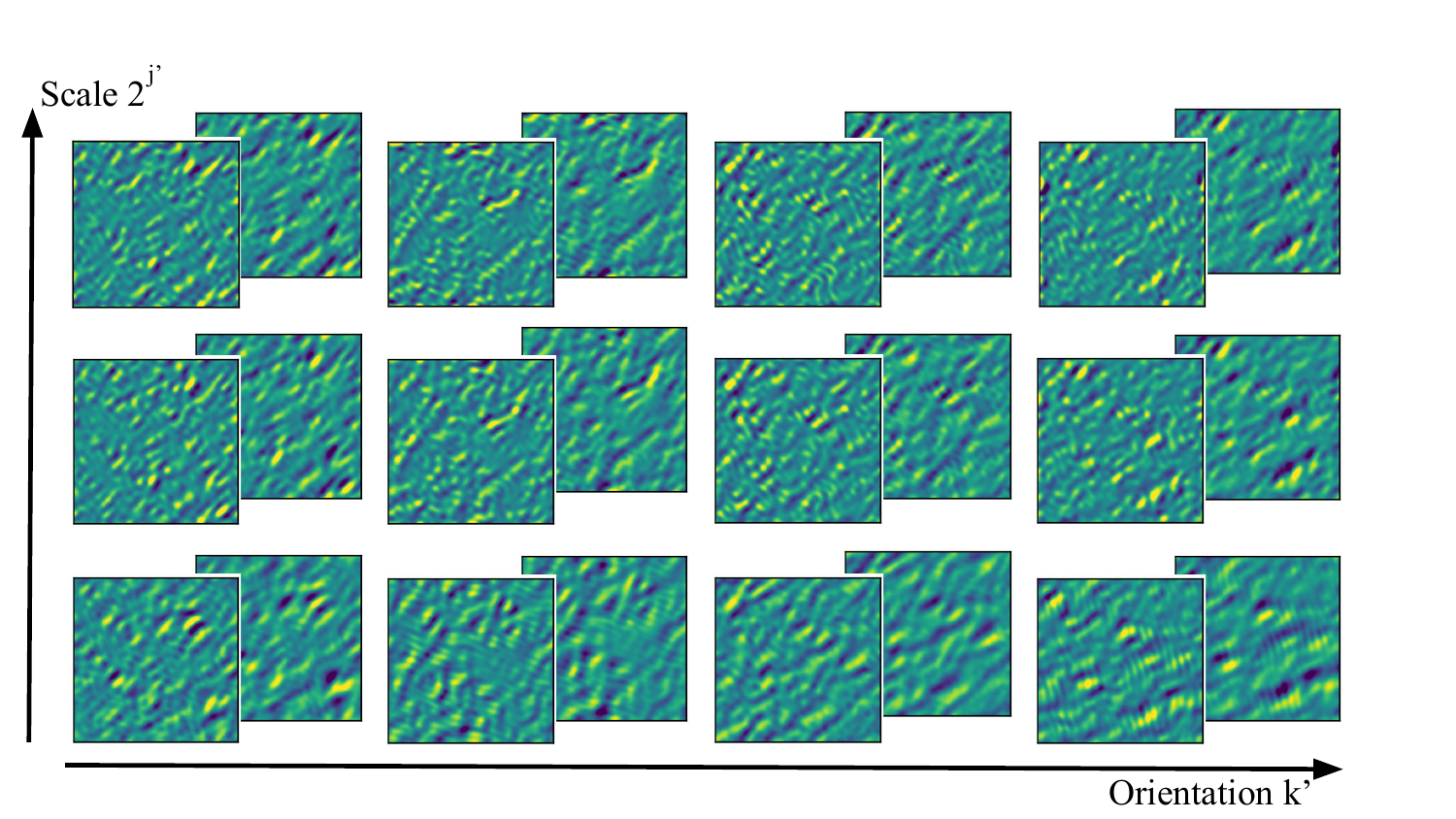}

(c)~~~~~~~~~~~~~~~~~~~~~~~~~~~~~~~~~~~~~~~~~~~~~~~~~~~~~~~~~~~~~~~~~~~~~~~~~~~~~~~~~~~~~~~~(d)~~~~~~~~~~~~~

\caption{ (a): $\varphi_{j-1}$ is a vorticity field of a 2D turbulence.  (b): Complex wavelet transform at scales $2^{j'} \geq 2^j$ computed from $\varphi_{j-1}$ without subsampling. Foreground and background images show respectively the real and imaginary parts of $\data* \tilde \psi_{j',k'}$, computed with
Morlet type wavelets. Each row corresponds to a scale $2^{j'} \geq 2^j$ 
and each column to the $Q = 4$ orientation indices $k'$.
(c) Modulus $| \data* \tilde \psi_{j',k'}|$ of wavelet coefficient images.
They have a large amplitude at sharp transitions, 
with long-range dependencies in space and across scales and orientations.
(d) The foreground and background images are the real and imaginary
parts of $|\varphi * \tilde \psi_{j',k'}| * \tilde \psi_{\ell,k}$ for different $j'$ and $k'$, and for a fixed $\ell = 3$ and $k=1$. 
These images look-alike, which shows that wavelet coefficient modulus $|\varphi * \tilde \psi_{j',k'}|$ have strong dependencies across scales $2^{j'}$ and orientations $k'$.}
\label{fig:complex-wavelet-transf}
\end{figure}

Figure \ref{fig:complex-wavelet-transf}(b) shows the 
wavelet coefficients of a vorticity image of 2D turbulence field.
The wavelet coefficients of such a stationary field are almost not correlated across scales, because wavelets have a Fourier transform
localized in different frequency bands. Wavelet coefficients in different frequency bands have phases which oscillate at different rates or along different orientations, which cancel correlations.
However, the amplitudes $|\varphi * \tilde \psi_{j',k}|$ of wavelet coefficients are strongly correlated across scales, as shown by 
Figure \ref{fig:complex-wavelet-transf}(c). 
The vorticity field has sharp variations which create
large amplitude wavelet coefficients at the same positions over multiple scales and orientations. 

\paragraph{Wavelet scattering} 
Most wavelet modulus
$|\varphi * \tilde \psi_{j',k'}|$ have long range spatial 
dependencies.
This long range  dependency is represented by local interactions
with a second hierarchic decomposition, computed with a second wavelet transform of $|\varphi * \tilde \psi_{j',k'}|$.

The wavelet coefficients $\varphi * \psi_{j',k'}$ and the second wavelet transform of
$|\varphi * \psi_{j',k'}|^q$ for $q=1$ or $q = 2$ are incorporated into a vector of scattering coefficients:
\begin{equation}
\label{complex-definition}
S_{j'}(\varphi_{j-1}) = \left(
\begin{array}{c}
\data * \tilde \psi_{j',k'} (2^{j-1}n)\\
|  \data * \tilde \psi_{j',k'}|^q * \tilde \psi_{\ell,k} (2^{j-1} n)
\end{array}
\right)_{\ell \geq j' , k,k' \leq Q,n} .
\end{equation}
If $q=2$  then upper and lower terms of $S_{j'}$
are polynomials of degree $1$ and $2$ of the values $\varphi(n)$. If $q=1$ then each term remains Lipschitz, but the complex modulus may create singularities when complex wavelet coefficients vanish,  which is addressed by replacing $|z|$ by $(|z|^2+\epsilon)^{1/2}$, for a small $\epsilon$. 
Scattering coefficients 
$| \varphi * \tilde \psi_{j',k'}|^q * \tilde \psi_{\ell,k} $ are indexed by two scales $2^{j'} \geq 2^j$ and $2^\ell \geq 2^{j'}$. This double hierarchy measures the variations of $| \varphi * \tilde \psi_{j',k'}|^q$
over neighborhoods of sizes proportional to $2^{\ell}$, 
in a direction indexed by $k$. 
Figure \ref{fig:complex-wavelet-transf}(d) shows these scattering coefficients
for a fixed $(\ell,k).$

\paragraph{Robust scattering interaction energies}
Probability models $p_\theta$ have been defined as maximum entropy 
distributions conditioned by the covariance values of scattering coefficients \cite{morel2023scale,cheng2023scattering}.
Such distributions are sampled with a microcanonical algorithm,
which avoids the calculation of the underlying Gibbs energy.
In the following we introduce a factorization of scattering energy models with the inverse wavelet renormalization group. It enables us to
  compute the Gibbs energy  parameters at each scale, with a coupling flow equation. Computing these Gibbs energies has important applications to analyze the properties of multiscale physical systems \cite{brossollet2024effectiveenergyinteractionsequilibrium}.

At the largest scale $2^J$, the Gibbs energy model 
$\En_{ \theta_J}  = \theta_J\trans \Phi_J$ of $p_{\theta_J} (\varphi_J)$
is defined with a quadratic term and a scalar potential, as in (\ref{inter-param}).
At scales $2^j \leq 2^J$,
we define a Gibbs energy model $\bar U_{\bar \theta_j}$ of
$p(\bar \varphi_j | \varphi_j)$  from interactions of scattering coefficients.
Let $z\trans$ be the complex conjugate transpose of the complex valued vector $z$.
The scattering energy model includes two-point interactions between scattering coefficients $S_j$ at the scale $2^j$ with 
$S_{j'}$ for $2^{j'} \geq 2^j$, plus a scalar potential:
\begin{equation}
\label{scalar-interaction3}
\bar \En_{\bar \theta_{j}}  =
\sum_{j' = j}^{J+1} S_j\trans   K_{j,j'} S_{j'} +
\bar V_{\bar \gamma_j}  = \bar \theta_j\trans \Psi_j  ~, 
\end{equation}
where
\begin{equation}
\label{scalar-interaction30}
\Psi_j  = 
\left(
\begin{array}{c}
S_{j} S_{j'}\trans \\
\Gamma 
\end{array}
\right)_{j' \geq j}~~
\mbox{and}~~
\bar \theta_j = 
\left(
\begin{array}{c}
K_{j,j'}\\\bar \gamma_j 
\end{array}
\right)_{j' \geq j}.
\end{equation}
The dimensionality of this model can be reduced from known symmetries
of $p$. If $p$ is stationary, then $p_{j-1} (\varphi_{j-1})$ is invariant to translation. The translation invariance of
$\bar U_{\bar \theta_j}$ is equivalent to imposing 
that each $ K_{j,j'}$ is a convolutional operator over the spatial grid of $\varphi_{j-1}$. Symmetries to actions of other groups can be enforced  with other conditions discussed in the next section.

If $q=2$, then the coordinates of each $S_{j'}$ are polynomials of degree $1$ and $2$ of the values $\varphi_{j-1}(n)$, so each interaction term $S_j\trans K_{j,j'} S_{j'}$ is
a polynomial of degree of $4$. If $q=1$, then the complex wavelet phase
is treated similarly but $\bar U_{\bar \theta_j}$ has a quadratic growth in
$\varphi$ as opposed to a degree $4$. It is obtained by replacing $|z|^2$ by $|z|$.
The polynomial of degree $4$ is changed into a polynomial of degree $2$ with
phase harmonics, whose properties are studied in \cite{zhang2021maximum}. 
It improves statistical robustness because $|z|$ is Lipschitz as opposed to $|z|^2$. 
Similarly to a linear rectifier used in neural networks, $|z|$
is homogeneous and can be computed as a linear combination of rectifiers  \cite{zhang2021maximum}. Models computed with $q=1$ have similar properties to models
computed with $q=2$, but are often more accurate because of their statistical robustness studied in the next section.

\paragraph{Local scattering spectrum interactions}
A hierarchic organization aims at creating the needed long-range interactions through local interactions in the hierarchy. Similarly,
a scattering spectrum defines long-range models of stationary fields with local low-dimensional interactions among scattering coefficients.

As in  \cite{morel2023scale,cheng2023scattering}, 
a scattering spectrum interaction model is constructed by eliminating
the interaction terms
in (\ref{scalar-interaction3}) which are a priori negligible. 
For stationary probabilities, the interaction matrices $K_{j,j'}$ are convolution operators.
Images of $S_{j'}$ are complex wavelet coefficients 
computed with $\tilde \psi_{j',k'}$ or with $\tilde \psi_{\ell,k}$ for 
$2^{\ell} \geq 2^{j'}$. Two images of $S_j$ and $S_{j'}$ 
have an energy mostly concentrated in disjoint frequency domains if
they are computed with different wavelets.
Their interaction thus has a negligible contribution to the energy $\bar U_{\bar \theta_j}$ because
each $K_{j,j'}$ in (\ref{scalar-interaction3}) is convolutional.

Two images of $S_j$ and $S_{j'}$ computed with the same $\tilde \psi_{\ell,k}$ have a priori non-zero interaction.
A scattering spectrum model further assumes that wavelet coefficient interactions are local in space. This assumption is valid for multiscale stationary processes which do not produce oscillatory phenomena. A scattering spectrum
model keeps only the interaction
of pairs of scattering coefficients in $S_j$ and $S_{j'}$ which have the same spatial position.

The resulting scattering spectrum model assumes
that interactions are local in space and over the scattering scales. 
However, it defines long-range interactions
in space and across the wavelet coefficients of the original image $\varphi$.
Appendix \ref{app:scattering} shows that this scattering spectrum model  reduces $\Psi_j $ to a vector of dimension $O(\log^2 d)$ if $\varphi$ has a dimension $d$.
The full hierarchic energy model $U_\theta$ aggregates the interaction models $\bar U_{\bar \theta_j}$ at all scales
$d \geq 2^j \geq 1$. It is thus defined by a coupling vector $\theta$ of dimension
$O(\log^3 d).$

\paragraph{Stationary Scattering Energy Model}
We give an analytical formulation of the stationary scattering energy model
$U_\theta$ obtained from interaction energy models $\bar U_{\bar \theta_j}$ 
at all scales $2^j$.
It is calculated with 
hierarchic stationary potentials defined from the
scattering interaction energies $\Psi_j$ in (\ref{scalar-interaction3}).
Let us write the complex wavelet transform of $\varphi$
\begin{equation}
\label{compln-wav-tn}
 W\data = \Big(  \data * \tilde \psi_{j,k} (n) \Big)_{0 < j \leq J+1,k,n} ~,
\end{equation}
This wavelet transform without subsampling is computed with the a-trous algorithm
of Appendix \ref{app:wavelets}.
From $\varphi_j = \varphi * \phi_j (2^j n)$,
Appendix \ref{app:wavelets} shows that the a-trous algorithm
computes wavelet coefficients of $\varphi$ at scales $2^J \geq 2^{j'} \geq 2^j$ subsampled at intervals $2^j$. We write
\[
R (\varphi_j) = \left(
\begin{array}{c}
\varphi * \phi_j (2^j n)\\
|\varphi * \tilde \psi_{j',k'} (2^j n)|^q 
\end{array}
\right)_{j' \geq j ,k',n} .
\]
If $f = (f_k )_k$ is a family of real valued fields, we also write
$f * f\trans = (f_k * f_{k'}\trans )_{k,k'}$ with $f_{k}\trans(n) = f_{k}(-n).$

\begin{theorem}
\label{prop:scat-covr-embeddge}
The potentials
\begin{equation}
\label{potential-equations2}
\Phi_j (\data_j) = \left(
\begin{array}{c}
R(\varphi_j) * R(\varphi_j)\trans\\
\Gamma ( \data_j * \phi_\ell ) 
\end{array}
\right)_{J-j \geq \ell \geq 0} 
\end{equation}
are stationary hierarchic  with interactions $\Psi_j$ defined in (\ref{scalar-interaction3}).
Regressing each free energy $F_j$ over $\Phi_j$ defines a fine scale Gibbs stationary model
\begin{equation}
\label{scat-enn-mod}
 \En_{\theta} (\data) = 
\frac 1 2 \data\trans K \data + 
 \sum_{j=0}^J   V_j (\data) + V_{\rm int} (\varphi) ~,
\end{equation}
with $V_j (\data) = 
\gamma_j\trans \Gamma(\data * \phi_j)$  and
\begin{equation}
\label{interapot}
 V_{\rm int} (\varphi) = 
\data\trans L\, (|W \data|^q) + 
\frac 1 2 (|W \data|^q)\trans M \,(|W \data|^q) ,
\end{equation} 
where $\gamma_j$ and the convolution operators $(K,L,M)$
are computed with a linear coupling flow equation.
\end{theorem}

The proof is in Appendix \ref{app:th:scat-covr-embeddge}.
When $q=2$, the interaction potential is thus a sum of third order and fourth order polynomials which capture interactions between wavelet coefficients at the scale $2^j$ and larger scales. Setting $q=1$ gives a robust approximation, where all terms have a quadratic growth.
The convolutional matrices
$L$ and $M$ have non-local kernels in space. A local scattering spectrum model
only keeps interaction coefficients between pairs of scattering
coefficients computed with the same wavelet, at the same position. It implies 
that $K = W\trans K' W$, $L = W\trans L' W$
and $M = W\trans M' W$, where $K'$ is diagonal and $L'$ and $M'$ are block diagonal convolutional matrix, with a total of $O(\log^3 d)$ non-zero interaction
coefficients.

Computing the fine scale Gibbs stationary model (\ref{scat-enn-mod}) is useful for a physical analysis of the energy \cite{brossollet2024effectiveenergyinteractionsequilibrium}. However, as in Theorem \ref{pro:scalar-embedded-potential}, let us emphasize that it is numerically not reliable to sample $p_\theta$ because this parametrization is often unstable. The parametrization $\bar \theta_j$ of conditional probabilities provide a more stable representation of the
Gibbs energy. It gives a faster and much more reliable sampling of the estimated probability distribution.

\subsection{Numerical Applications to Dark Matter and Turbulence Fields}
\label{sec:applications}

We evaluate the precision of hierarchical wavelet scattering models $ \En_{\theta}$ for two types of non-Gaussian physical fields having
coherent geometric structures, studied in \cite{morel2023scale,cheng2023scattering}. We numerically estimate the evolution with system size $d$ of the Langevin normalized auto-correlation relaxation time, from which we obtain a partial information on the evolution of the normalized log-Sobolev constants.

Numerical experiments are carried on $2$D fields of dark matter images,
which are the logarithm of 2D slices of simulated 3D large-scale distribution of dark matter in the Universe~\cite{Villaescusa_Navarro_2020} shown in Figure \ref{plot:QuijoteTurbuSynth}(a).
We also model 2D turbulence vorticity fields of incompressible 2D fluids stirred at the scale of size 32 pixels, at a fixed time, simulated from 2D Navier-Stokes equations \cite{schneider2006coherent}, shown in Figure \ref{plot:QuijoteTurbuSynth}(a). The time evolution of 2D turbulence is
an inverse cascade which transfers the energy towards the lowest scales \cite{kraichnan1967inertial}.
The 2D Navier-Stoke simulation initialized with a Gaussian white noise at $t=0$
defines a transport of this Gaussian white distribution into a stationary distribution at a fixed $t$. The vorticity field  is stationary in space at a fixed time, but it does come from a system at equilibrium in time. There is no closed formula for Hamiltonians of such out of equilibrium systems, except 
in particular 1D cases \citep{derrida2002large,bertini2015macroscopic}. Both datasets (100  and 3000 independent realizations) have periodic boundary conditions and are down-sampled from $d=256^2$ to $d=128^2$ pixels. They are augmented with rotations and spatial symmetries because their probability distribution is isotropic.

\begin{figure}
\centering
\includegraphics[width=0.95\textwidth]{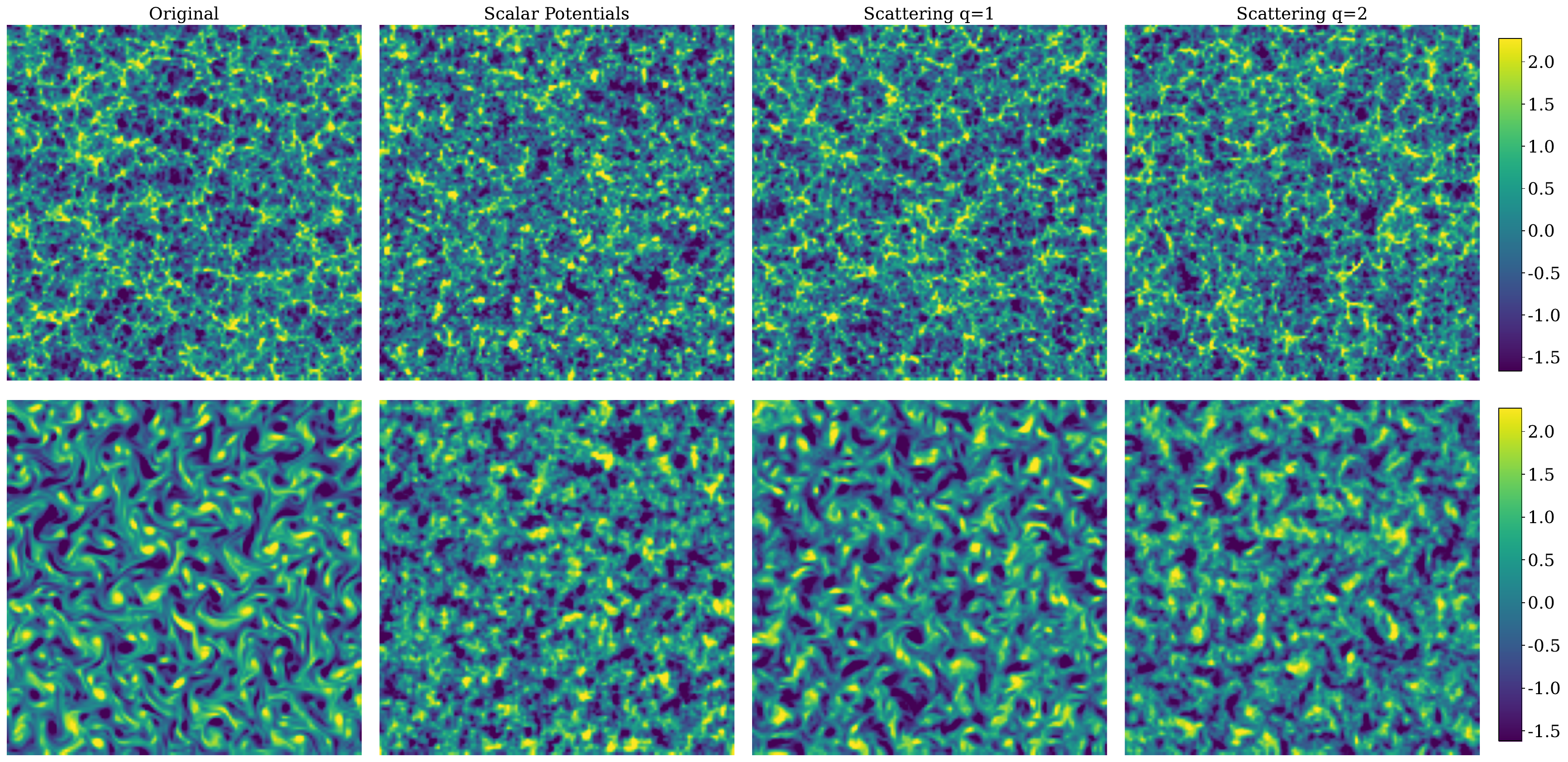}

~~(a)~~~~~~~~~~~~~~~~~~~~~~~~~~~~~~~~~~~~~~~(b)~~~~~~~~~~~~~~~~~~~~~~~~~~~~~~~~~~~~~(c)~~~~~~~~~~~~~~~~~~~~~~~~~~~~~~~~~~~~~~~(d)~~~~~~~~~~
\caption{(a): Top row: $2$D slice of a $3$D
  dark-matters density simulation \cite{Villaescusa_Navarro_2020}. Bottom row: vorticity fields of a $2$D periodic turbulent flow at a fixed time $t$, generated numerically with the 2D Navier Stokes equation from an initial Gaussian white noise \cite{schneider2006coherent}. (b): Samples of the hierarchical scalar potential model defined in (\ref{scalar-interaction}). (c): Samples of the hierarchic wavelet scattering model, with an exponent $q=1$. (d): Same as $(c)$ with an exponent $q=2$.}
  \label{plot:QuijoteTurbuSynth}
\end{figure}

\paragraph{Hierarchic model generations}
Figure \ref{plot:QuijoteTurbuSynth}  (b,c,d) shows images generated from hierarchical models, respectively with a scalar potential, robust wavelet scattering interactions with $q=1$ and higher order
wavelet scattering with $q=2$.
Over all scales, the scalar potential model 
and the wavelet scattering models have respectively 
about $300$ and $2500$ coupling parameters. There is no
scalar potential term in the scattering models. 

Generations from scalar potential models in Figure \ref{plot:QuijoteTurbuSynth}(b) do not restore random geometric structures appearing in the original fields. 
For dark matter fields, the wavelet scattering models in 
 Figure \ref{plot:QuijoteTurbuSynth}(c,d)
reproduce well the visual texture of these images.
For the vorticity fields of 2D turbulence, 
eddies and vorticity flows are well
reproduced only with $q=1$ but are degraded for $q=2$.
Indeed, as shown by
(\ref{theta-calcul}) and in Appendix \ref{app:score-matching}, computing the coupling flow parameters $\bar\theta_j$ of the hierarchical model requires inverting a matrix of empirical  moments estimated on the training dataset. 
Higher order polynomials amplify outliers. It increases 
the estimation variance of high order moments, and thus introduces more errors. As expected, the scattering model with $q=1$ is more robust than the high order model with $q=2$.
For turbulence images, estimation errors can introduce a divergence of the Langevin diffusion at the finest scale, where the spectrum has a fast decay. This is avoided by adding a small confinement term for large amplitude coefficients. It adds $\epsilon D_j\bar\varphi_j^4$ to the energy $\bar U_{\bar \theta_j}$. The value of $\epsilon$
is about $10$ times larger for $q=2$ than for $q=1$, which affects the generated image quality.

\begin{figure}
\centering
\includegraphics[width=0.45\textwidth]{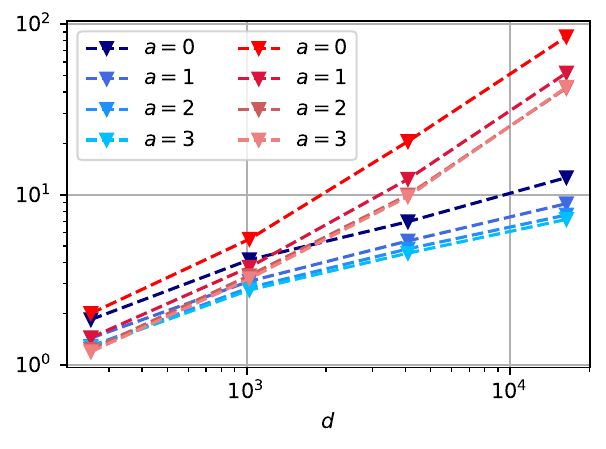}

  \caption{Langevin normalized auto-correlation relaxation time (\ref{eq:autocorrelation}) for scattering hierarchical models computed in wavelet or wavelet packet bases, with $a_j = \min(a, J-j)$, as a function of the image size $d$.
  Langevin relaxation time for dark matter, (computed with $q=2$, in a Symlet-3 basis) are in blue and for 2D turbulence fields  (computed with $q=1$, in a Symlet-4 basis) in red. Increasing the wavelet packet Fourier resolution with $a$ slightly reduces the relaxation time, but it does not avoid an exponential growth with $d$.}
\label{turbulence-mixing}
\end{figure}

\paragraph{Langevin relaxation time}
The renormalization aims at reducing the amplitude of normalized log-Sobolev constant to estimate
a model by score matching with as few samples as possible 
and to reduce the computational complexity 
of the Langevin diffusion. A necessary condition is to control
the largest eigenvalue of the covariance and its condition number. As explained in Section \ref{sec:basis-choice}, a
hierarchical model must use wavelets having enough vanishing moments and which are sufficiently regular. However, larger spatial support create longer range
spatial interactions between wavelet coefficients, as highlighted for $\vvarphi^4$ in section \ref{subsec:scalarpots}. Since the wavelet support increases with the number of vanishing moments,
we choose a wavelet with just enough vanishing moments.
For dark matter images and turbulence fields, we respectively use Symlet-$3$ and Symlet-$4$ wavelets.
These dark matter and 2D turbulence images have a
power spectrum which decays faster than a power law at the highest frequencies. To renormalize the smallest eigenvalues of the covariance at high frequencies, 
Section \ref{sec:basis-choice} also explains that $\bar \varphi_j$ can be represented in a wavelet packet basis having a sufficient frequency resolution. 
The 
wavelet packet filtering (\ref{waveletpacksn})
reduces the width of wavelet frequency support by a factor of $2^{a_j}$. In numerical experiments, we set
$a_j = \min(a,J-j)$ and we adjust $a$.
For $a=0$, the wavelet-packet basis is a wavelet basis.
Choosing a wavelet packet basis as opposed to a wavelet basis does not modify the hierarchic energy model.
It only modifies the coordinate system representing $\bar \varphi_j$ and hence the renormalization which modifies the log-Sobolev constant.

Figure \ref{turbulence-mixing} gives the
evolution of the Langevin hierarchic normalized auto-correlation relaxation time $\tau$ of the hierarchical model in (\ref{eq:autocorrelation}), as
a function of the system size $d$. 
For $a=0$ corresponding to wavelets, the relaxation time
$\tau$ increases with the system size $d$, both for dark-matter and turbulence fields. This is different from the 
relaxation time of hierarchical models of $\varphi^4$ shown in Figure \ref{fig:langevin-mixing}(a), which does not grow with $d$. However, after renormalization, the relaxation time for dark-matter and turbulence is $3$ orders of magnitude below the one for $\varphi^4$ without any renormalization. 
It shows that the hierarchical renormalization accelerates considerably the sampling, although computations increase with $d$.

Increasing the frequency resolution of wavelet packets with $a$ has a small effect on the Langevin relaxation time growth in Figure \ref{turbulence-mixing}. The remaining growth is due to the non-Gaussian multiscale behavior of the model. It is produced by scaling properties of third and fourth order energy components, which are different from second order scaling properties of the covariance, and are therefore not compensated by a renormalization based on the covariance. Such phenomena appear in multifractal random processes \cite{jaffard2004wavelet,abry2013scaling}.
The scattering model generated in Figure \ref{plot:QuijoteTurbuSynth} (b,c) 
are calculated by renormalizing wavelet coefficients
in a wavelet packet basis, with $a=3$ for turbulence images and with $a=2$ for dark matter fields.

\begin{figure}
\centering
\includegraphics[width=1\textwidth]{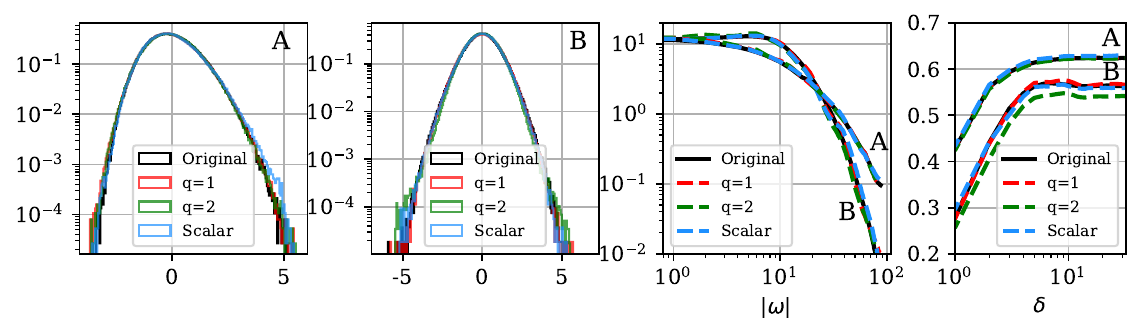}

~~~~~~~~~~~~~~~~~~(a)~~~~~~~~~~~~~~~~ ~~~~~~~~~~~~~~~~~~~~~~~~(b)~~~~~~~~~~~~~~~~~~~~ ~~~~~~~~~~~~~~~~~~~~~~(c)~~~~~~~~~~~~~~~~~~~~~~~~~~~~~~~~~~~~~~~~(d)

\includegraphics[width=0.975\textwidth]{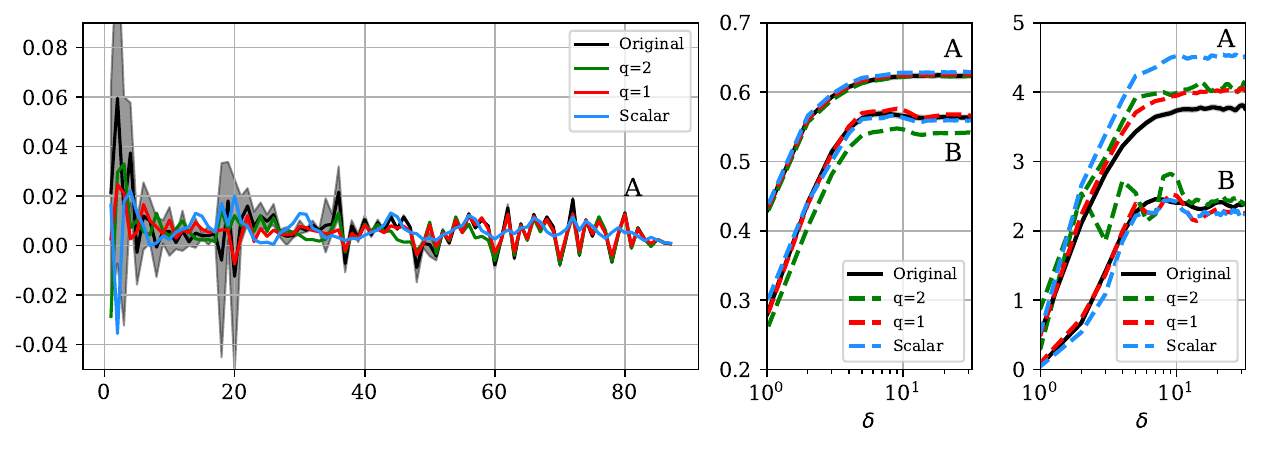}

~~~~~~~~~~~~~~~~~~~~~~~~~~~~~~~~~~~~~~~~~~~(e)~~~~~~~~~~~~~~~~~~~~~ ~~~~~~~~~~~~~~~~~~~~~~~~~~~~~~~~~~~~~~~~~~~~(f)~~~~~~~~~~~~~~~~~~~~~~~~~~~~~~~~~~~~(g)~~~~

\caption{Each graph shows a specific family of moments estimated on the
  training data (in black) and on samples of a hierarchical model (wavelet scattering with $q=1$ in red, with $q=2$ in green, and scalar potential model in blue). Models 
  are labeled (A) for dark-matter and (B) for 2D turbulence. Gray zones represent the estimation variance of these higher order moments over the training dataset. (a,b): marginal probability densities of $\data(n)$. (c): covariance eigenvalues (power-spectrum). (d): Structure function of order $m=1$ as a function of $\delta$. (e): third order bi-spectrum moments, ordered from low to high frequencies. Not shown for the 2D turbulence (B), because they all vanish. 
  (f): Structure function of order $m=4$. (g): Structure function of order $m=6$.}
\label{plot:QuijoteTurbuStats}
\end{figure}

\paragraph{Hierarchic model precision}
We have no energy model for the dark matter and turbulence fields used in numerical experiments. The model precision 
is thus evaluated numerically over a family of standard moments used in statistics and physics.
We test the distribution of point-wise marginals, the Fourier spectrum (second order moments), the bi-spectrum 
(third order moments) also calculated in the frequency domain, and structure functions. 
Each moment is estimated with a Monte Carlo sum over $m$ training samples of $p$. These moments are compared with moments of hierarchical models, also estimated with a Monte Carlo sum over enough model
samples generated by MALA sampling. 

Figure \ref{plot:QuijoteTurbuStats}(a,b) show that
scalar potential models generate
stationary fields $\varphi$, where  $\varphi(n)$ have a
marginal distribution nearly equal to the marginals of the original fields, since the model is optimized from these marginals. The same result is obtained for the wavelet
scattering models, although they do not incorporate a
scalar potential which imposes these marginal moments. 
All models reproduce  well the power spectrum, which is 
only specified in these models by a reduced diagonal matrix 
with less than $\log d$ wavelet second order moments. 

The bi-spectrum are third-order moments computed over
Fourier coefficients, which is zero for a Gaussian process. 
It is also zero if $p(\varphi) = p(-\varphi)$, 
which is the case for the 2D turbulence vorticity. This prior information allows us to 
reduce the model dimensionality by eliminating antisymmetric
terms in the model. The bi-spectrum is thus
only computed for the dark-matter density images.
The bi-spectrum calculation is explained in Appendix \ref{app:bi-tri}.
Figure \ref{plot:QuijoteTurbuStats}(e) shows that the bi-spectrum is well reproduced by the wavelet scattering
model with $q=2$, which incorporates
third order polynomial terms. It is also well reproduced with $q=1$. This robust model does not include third order terms but a lower order equivalent term which is also antisymmetric. On the opposite, the bi-spectrum of the scalar potential model is quite different from the bi-spectrum of the original
dark-matter field.  

Kolmogorov structure functions \cite{kolmogorov1941local}
are moments of order $m$ over the field of increments, indexed by the spatial lag $\delta$ 
\begin{equation}
    \text{SF}_m(\delta) = \expect[p]{|\data(n) - \data(n-\delta)|^m}.
\end{equation} 
Figures \ref{plot:QuijoteTurbuStats}(d,f) plot the structure
functions respectively for $m=1$ and $m=4$ since $2nd$ and $3rd$ order moments have already been evaluated. The only model providing an accurate approximation both for the turbulence and dark-matter fields is the robust wavelet scattering model for $q=1$. The error of the wavelet high-order model for $q=2$ over the turbulence field is due to the larger regularization needed to confine the Langevin diffusion. The reproduction of higher order moments slowly degrades when increasing the order of the moments. Figure \ref{plot:QuijoteTurbuStats}(g) shows that none of the models faithfully reproduces the structure function of order $6$ on dark matter, but it is well reproduced on the turbulence which is more Gaussian.

\begin{figure}
\centering
\includegraphics[width=0.317\textwidth]{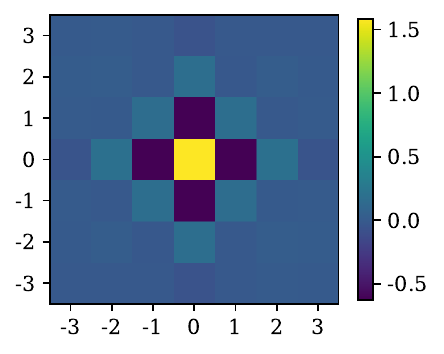}\includegraphics[width=0.3\textwidth]{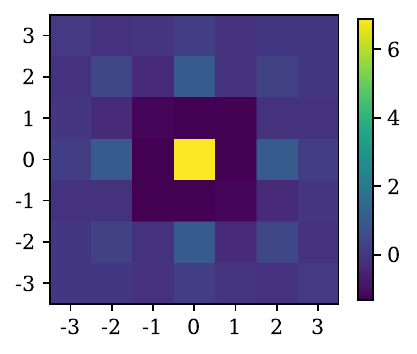}
\includegraphics[width=0.265\textwidth]{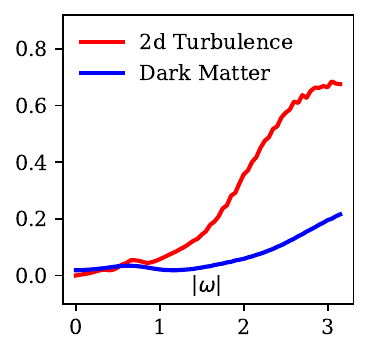}

~~~~~~~(a)~~~~~~~~~~~~~~~~~~~~~~~~~~~~~~~~~~~~~~~~~~~~~~~~~~~~~~~(b)~~~~~~~~~~~~~~~~~~~~~~~~~~~~~~~~~~~~~~~~~~~~~~~~~~~~~~~(c)~~

\caption{Estimated 2-point convolutional kernel $K$ of $\En_{\theta}$ defined in   \cref{scat-enn-mod}, for $q=1$. (a): kernel for dark matter images. (b): kernel for 2D turbulence. They are local and nearly rotation invariant. 
 (c): Fourier transform of each kernel as a function of $|\omega|$.}
\label{plot:QuijoteTurbuK}
\end{figure}

\paragraph{Fine scale Gibbs energy model}
Theorem \ref{prop:scat-covr-embeddge} proves that the
fine scale Gibbs scattering model is characterized by a parameter
vector $\theta$ composed of $3$ convolutional
operators: $K,L,M$ in (\ref{scat-enn-mod}).
They are progressively computed from the learned interaction parameters $\bar\theta_j$ at each scale, with the coupling flow equation (\ref{prop-eq3}).
Figure \ref{plot:QuijoteTurbuK}(a,b) shows the convolution kernel $K$ for turbulence and dark-matter datasets. Both kernels are local in space and nearly isotropic, which is not imposed by the model. They resemble a Laplacian. Figure \ref{plot:QuijoteTurbuK}(c) shows their Fourier transform
as a function of the frequency amplitude $|\omega|$. The estimated kernel for 2D turbulence has higher eigenvalues at high frequencies than the one for dark matter, which is coherent with the faster decrease of its power spectrum.

One can analyze some physical properties from the estimated
operators $K,L,M$ of $\theta$, but this energy
parametrization is
unstable. Sampling directly $p_\theta$ with a MALA algorithm
leads to diverging calculations which indicates that the 
computed Gibbs energy model is not integrable.
Trying to estimate directly the parameter vector $\theta$
by score matching does not improve the result.
We verified that sampling $p_\theta$ also leads to important errors,
producing non-plausible fields. 
For dark matter as well as turbulent fields, the quality of samples
in Figure \ref{plot:QuijoteTurbuSynth} can only be obtained with
the stabilization resulting from the hierarchical sampling algorithm.
It demonstrates the efficiency and stability
of Gibbs energy representations through
the wavelet renormalization group.

\section{Conclusion}

Hierarchic models provide an
inverse renormalization group decomposition of probability
distributions into conditional probabilities in wavelet
bases. It leads to low dimensional parametric models and the renormalization partly avoids the bad conditioning of learning and sampling algorithms,
by reducing log-Sobolev constants. In wavelet orthogonal bases, 
the renormalization group is computed over
dyadic scales, as in the Kadanoff scheme.
Continuous formulations in a Fourier basis would require to use
much higher dimensional parametric models.

For multiscale non-Gaussian processes such as dark matter distribution or 2D turbulence, we introduced a low-dimensional wavelet scattering model, based on robust multiscale high order interactions. Numerical tests show that they can provide an accurate model of complex fields beyond scalar potential models, although
such models are defined by unstable parametric Gibbs energies at a fine scale.
However, we remain far from satisfying models of turbulence.
More complex structures appear in 3D where the energy cascade is not inverted, and we do not consider the time evolution.

More sophisticated models of wavelet conditional probabilities can be
obtained with score diffusion algorithms \cite{guth_wavelet_2022},
and it leads to faster sampling algorithms. 
Deep neural networks have been used to define local
models of conditional probabilities, which have been applied to complex
data bases such as faces \cite{kadkhodaie-local-conditional-models}.
However, the resulting models incorporate much more parameters and are more difficult to analyze. Understanding the probability models calculated by
these deep networks is needed for physical interpretations.

\section*{Acknowledgments}
This work was supported by PR[AI]RIE-PSAI-ANR-23-IACL-0008 and the DRUIDS projet ANR-24-EXMA-0002. We thank Rudy Morel, Erwan Allys and Misaki Ozawa for providing the training datasets. We also thank Nathanaël Cuvelle-Magar, Jean Baptiste Himbert, Florentin Guth.

\clearpage

\appendix

\section{Wavelet Transforms}
\label{app:wavelets}
We review the properties of real and complex 
wavelet transforms computed by iterating over low-pass and band-pass filters, 
in one and two dimensions.

\paragraph{Conjugate mirror filters}
Orthogonal wavelets are computed with conjugate mirror filters.
In dimension $r = 1$, a pair of conjugate mirror filter $(g,\bar g)$ includes a low-pass filter $g$ whose Fourier transform $\hat{g}(\omega)=\sum_n g(n)e^{-in\omega}$ satisfies
\begin{equation}
    \label{eqn:filters}
             |\hat{g}(\omega)|^2 +\vert \hat{g}(\omega+\pi)\vert^2 = 2 ~~\mbox{and}~~            \hat{g}(0) = \sqrt{2}.
\end{equation}
The second filter $\bar g$ has a single high-pass filter defined by 
$\bar g(n) = (-1)^{1-n}\, g(1-n)$. One can verify that the resulting
convolution and subsampling operators 
\begin{equation}
\label{filtersGbarg}
G \varphi (n) = \varphi * g(2n)~~\mbox{and} ~~\bar G \varphi(n) = \varphi * \bar g(2n)
\end{equation}
define an orthogonal matrix $\left( {G} \atop {\bar G} \right)$ 
\cite{Mallat}. All convolutions are computed with periodic boundary conditions.

For images ($r = 2$), two-dimensional conjugate mirror filters 
are computed as separable products of one-dimensional conjugate mirror filters $(g,\bar g)$ \cite{Mallat}. For $n = (n_1,n_2)$ there is a single two-dimensional
separable low-pass filter $g(n_1, n_2) = g(n_1)g(n_2)$ and 
 $3$ high-pass filters in $\bar g = (\bar g_k )_{1 \leq k \leq 3}$,  with
\[            \bar g_1(n_1,n_2) = g(n_1)\bar g(n_2) ~,~
            \bar g_2(n_1,n_2) = \bar g(n_1)g(n_2) ~,~
            \bar g_3(n_1,n_2) = \bar g(n_1)\bar g(n_2).
\]
The convolutional operator $G$ and $\bar G$ in two dimensions are still defined
by (\ref{filtersGbarg}) with these two-dimensional separable filters, and $\bar G \varphi$ has $3$ output images.

\paragraph{A-trous filters}
Averaged coefficients and wavelet coefficients are 
iteratively computed for $j > 0$ with
\begin{equation}
\label{insdfdsf}
\varphi_{j}(n) = \varphi_{j-1} * g(2n)~~\mbox{and}~~\bar \varphi_j = \varphi_{j-1} * \bar g(2n) .
\end{equation}
These coefficients can be written as convolutions of the input $\varphi = \varphi_0$ with a scaling filter $\phi_j$ and wavelets $\psi_{j,k}$ subsampled by $2^j$:
\begin{equation}
\label{filters2}
\data_j = \big(\data * \phi_j (2^j n)\big)_{n}~~\mbox{and}~~
\bar \data_j = \big( \data * \psi_{j,k} (2^j n) \big)_{k, n}.
\end{equation}
These scaling filters and wavelets
satisfy recursive equations computed with "a-trous" filters resulting from the cascaded
subsampling. In one dimension, for any filter $h$ we write $h_j$ the a-trous filter such that $h_j (2^j n) = h(n)$ and $h_j (n) = 0$ if $2^{-j} n$ is not an integer in
one dimension or does not belong to 
the two-dimensional grid $\N^2$ in two dimensions. The a-trous filter $h_j$
is a dilation of $h$ by $2^j$, by setting intermediate coefficients to zero. 
One can derive from (\ref{insdfdsf})  that
\begin{equation}
\label{filters-scal}
\phi_{j} = \phi_{j-1} * g_{j-1} ~~\mbox{and}~~
\psi_{j,k} = \phi_{j-1} * \bar g_{j-1,k} .
\end{equation}
Scaling filters and wavelets are thus obtained with a cascade of a-trous filters.
One can verify that
\begin{equation}
\label{phi-rec}
\varphi_j * \phi_\ell (n) = \varphi_{j-1} * \phi_{\ell+1} (2n) ,
\end{equation}
and
\begin{equation}
\label{psi-rec}
\varphi_j * \psi_\ell (n) = \varphi_{j-1} * \psi_{\ell+1} (2n) .
\end{equation}

\paragraph{Asymptotic wavelet bases}
When $j$ goes to $\infty$, for appropriate filters $\bar g$ and low-pass filters $g$, one
can prove \cite{doi:10.1137/1.9781611970104} that $\phi_j$ and $\psi_{j,k}$ converge to  $\phi(x)$ and wavelets $\psi_k (x)$, up
to a dilation by $2^j$. These limit functions are square integrable, $\int |\phi(x)|^2 dx < \infty$ and 
$\int |\psi_k (x)|^2\, dx < \infty$. 
In $1$ dimension, conjugate mirror filters define a scaling function $\phi(x)$ and a single asymptotic wavelet $\psi(x)$ whose Fourier transforms 
$\hat \phi(\om)$ and $\hat \psi(\om)$ satisfy
\[
\hat \psi(\om) = \frac{1}{\sqrt{2}} \hat{\bar g}(2^{-1}\om )\, \hat \phi(2^{-1} \om)~~\mbox{with}~~\hat \phi(\om) = \prod_{q=1}^{+\infty}
\frac {\hat g(2^{-q} \om)} {\sqrt 2} .
\]
If we impose that $\hat g(\om) > 0$ for $\om \in [-\pi/2,\pi/2]$, then one can prove
\cite{Mallat} that $\{ 2^{-j/2} \psi(2^{-j} x -n) \}_{n \in \Z , j \in \Z}$ is an
orthonormal basis of $\Ld(\R)$. If we limit the maximum scale to $1$, with periodic boundary conditions, we obtain an orthonormal basis of $\Ld([0,1]^2)$. Wavelets with non-periodic boundary conditions may also be designed \cite{doi:10.1137/1.9781611970104}.

A Haar filter is the conjugate mirror filter having a minimum size spatial support: $g(n)=1$ if $n=0,1$ and $g(n) = 0$ otherwise.
It defines the Haar wavelet $\psi$ shown in Figure \ref{fig:wavelets}(a).
The Shannon low-pass filter has a Fourier transform having a minimum
size support: $\hat g = {1}_{[-\pi/2,\pi/2]}$. It defines a Shannon wavelet, shown in figure \ref{fig:wavelets}(c).
A Daubechies Symlet filter $g$ of order $m$ 
\cite{doi:10.1137/1.9781611970104} has a support of size $2m-1$ and defines a compactly supported wavelet having $m$ vanishing moments:
\[
\forall 0 \leq k < m~,~\int x^k\, \psi(x)\, dx = 0 .
\]
These integrals imply that $\hat \psi$ and its $m-1$ first derivatives vanish
at the frequency $\om = 0$ and hence that $|\hat \psi(\om)| = O(|\om|^m)$. 
A Symlet filter is as symmetric as possible \cite{doi:10.1137/1.9781611970104}.
The regularity of
$\psi$ also increases with $m$.
The Symlet-$4$ wavelet is shown in figure \ref{fig:wavelets}(b).

\begin{figure}
\centering
\includegraphics[width=0.5\textwidth]{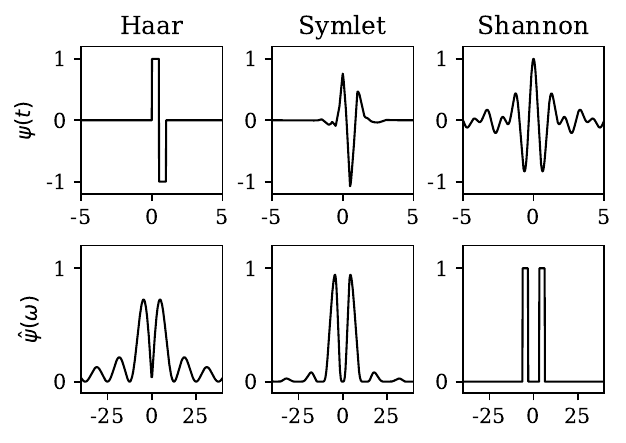}

~~~~~~~~~~~(a)~~~~~~~~~~~~~~~~~~~~~~~ (b)~~~~~~~~~~~~~~~~~~~~~~~~(c)
\caption{Top row: Graphs of different one-dimensional wavelets $\psi(t)$. Bottom row: Graphs of their Fourier transform modulus $|\hat\psi (\om)|$. (a) Haar (1 vanishing moment), (b) Daubechies Symlet (4 vanishing moments), (c) Shannon Wavelet.}
\label{fig:wavelets}
\end{figure}

\paragraph{Complex Morlet filters}
Complex wavelets $\tilde \psi_{j,k}$ are defined by replacing the real filters $\bar g = (\bar g_k )_k$
by a family of complex filters $\tilde g = (\tilde g_k )_k$.
Similarly to (\ref{filters-scal})
\begin{equation}
\label{filters-scal2}
\phi_{j} = \phi_{j-1} * g_{j-1} ~~\mbox{and}~~
\tilde \psi_{j,k} = \phi_{j-1} * \tilde g_{j-1,k}~,
\end{equation}
where $\tilde g_{j,k}$ is the a-trous filter defined from $\tilde g_k$.

In numerical applications we use $Q=4$ complex Morlet filters $\tilde g = \{\tilde g_{k} \}_{k \leq Q}$ defined by
\begin{equation}
  \label{Morlet-filt}
\tilde g_{k} (n_1,n_2) = \gamma\, e^{- \eta (n_1^2 + n_2^2)}\big(e^{i \xi (n_1 \cos\alpha_k + n_2 \sin \alpha_k)} - \beta_k\big),
\end{equation}
$\alpha_k = k \pi/Q$ and each $\beta_k$ is adjusted so that 
$\sum_{n_1,n_2} \tilde g_{k} (n_1,n_2) = 0.$
We choose $\eta = 1.67$ and $\xi = \pi$.
Figure \ref{fig:complexfilter} shows the two-dimensional complex wavelets $\tilde \psi_{j,k}$, computed with the 2D separable 
low-pass Symlet-4 conjugate mirror filter $g$ and these complex Morlet filters $\tilde g$.

 At the very large scales close to the image size, to avoid issues created by periodic boundary conditions, 
 the filters $\tilde g_k$ are modified. They are constructed from the analytic part of the conjugate
 filters $\{ \bar g_k \}_{k \leq 3}$ associated to $g$. Computing the analytic part amounts to restrict the Fourier transform over half of the Fourier plane \cite{Mallat}, which defines a complex filter. The third diagonal filter is divided into two analytic filters located in the two diagonal directions.

\paragraph{A-trous wavelet transform}
For any $j  > 0$, $\varphi_j(n) = \varphi * \phi_j (2^j n)$.
Moreover, (\ref{filters-scal2}) implies that, for $j\leq j'$,
\[
\tilde \psi_{j',k} = \phi_{j-1} * g_{j-1} * ... * g_{j'-2} * \tilde g_{j'-1} .
\]
One can thus verify that $\varphi * \tilde \psi_{j',k} (2^{j-1} n)$ is calculated
from $\varphi_{j-1}$ by an a-trous algorithm which cascades convolutions
with dilated low-pass filters
$g$ and the complex Morlet wavelet filter $\tilde g_k$
\begin{equation}
\label{psi-rec22}
\varphi * \tilde \psi_{j',k} (2^{j-1} n) = \varphi_{j-1} * g * ... * g_{j'-j -1} * \tilde g_{j'-j,k} (n) .
\end{equation}

The complex wavelet transform $W$ in (\ref{compln-wav-tn}) is computed by convolutions with $\phi_j$ and $\tilde \psi_j$ without subsampling. Since these filters are cascade of a-trous filters
$g_j$ and $\bar g_j$ in (\ref{filters-scal2}), these convolutions can be calculated as a cascade of a-trous filterings, from $\varphi = \varphi_0$:
\begin{equation}
\label{psi-rec222}
\varphi * \tilde \psi_{j,k} (n) = \varphi *  g * .. * g_{j-2} * \tilde g_{j-1}(n)~.
\end{equation}

\section{Scattering Spectrum Interactions}
\label{app:scattering}
This appendix specifies scattering spectrum interaction matrices 
$K_{j,j'}$, for the wavelet scattering model (\ref{scalar-interaction3}), with
\[
S_{j'} (\varphi_{j-1}) = \big(\varphi * \tilde \psi_{j',k} (2^{j-1} n) \,,\,|\varphi * \tilde \psi_{j',k'}|^q * \tilde \psi_{\ell,k} (2^{j -1} n)\big)_{k,k' \leq Q, \ell \geq j',n} .
\]

A scattering spectrum model for stationary processes retains interactions between coefficients of $S_j$ and $S_{j'}$ in (\ref{scalar-interaction3}) computed with the same wavelets at the same spatial position.
The pairs of
coefficients of $S_j$ and $S_{j'}$ which interact are thus
\begin{equation}
\label{Lambda20}
\data * \tilde \psi_{j,k} (2^{j-1}n)~~\mbox{and}~~\data * \tilde \psi_{j',k} (2^{j'-1}n)~~\mbox{for } j'=j,~ 1 \leq k \leq Q
\end{equation}

\begin{equation}
\label{Lambda30}
|\data * \tilde \psi_{j,k'}|^q * \tilde \psi_{\ell,k} (2^{j-1} n)~~\mbox{and}~~\data * \tilde \psi_{j',k} (2^{j-1} n)~~\mbox{for }j' = \ell, ~1 \leq k',k \leq Q,
\end{equation}
\begin{equation}
\label{Lambda40}
|\data * \tilde \psi_{j,k'}|^q * \tilde \psi_{\ell,k} (2^{j-1} n)~~\mbox{and}~~|\data * \tilde \psi_{j',k''}|^q * \psi_{\ell,k} (2^{j-1})~~\mbox{for } \ell \geq j' \geq j, ~1 \leq k,k',k'' \leq Q.
\end{equation}

For translation invariant energies, the translation invariant version of the potential $\Psi_j$ in (\ref{scalar-interaction30}) has 
four interaction potentials:
\begin{equation}
\label{scat-eqgenera}
 \Psi_j  = \Big(\Gamma \,,\, 
 \Lambda_{2,j}\,,\,
  \Lambda_{3,j} \,,\,
   \Lambda_{4,j}  \Big) .
\end{equation}
The potential vector $\Gamma$ defined in (\ref{scalar-pot-vec})  provides an approximation of scalar  potentials.
For the other $3$ potentials, 
we only keep interactions between coefficients computed with a
same wavelet at the same position. 
The interactions (\ref{Lambda20})
give the average of squared wavelet coefficients in each direction $k$:
\begin{equation}
\label{Lambda2}
\Lambda_{2,j} (\data_{j-1}) = 
\Big(\sum_n |\data * \tilde \psi_{j,k}(2^{j-1} n)|^2 \Big)_{k \leq Q} .
\end{equation}
They provide an estimation of the power spectrum  over the frequency support where
$\hat \psi_{j,k}$ is concentrated. For processes with power spectrums that decrease faster than a power law, we as well include in $\Lambda_{2,j}$ shifted two-points interactions, of the form $\sum_n \data *\tilde \psi_{j,k}(2^{j-1} n)\data *\tilde \psi_{j,k}^*(2^{j-1}(n+\tau))$, with small $\tau$.

The interaction terms (\ref{Lambda30}) capture the interaction between 
wavelet coefficients at the scale $2^{j'}$ and scattering coefficients at a
same scale $2^{\ell} = 2^{j'}$
\begin{equation}
\label{Lambda3}
\Lambda_{3,j} (\data_{j-1}) =\Big( \sum_n |\data * \psi_{j,m}|^q * \psi_{j',k}(2^{j-1}n)~ \data * \psi_{j',k}^*(2^{j-1}n) \Big)_{k,m \leq Q,j'\geq j }.
\end{equation}
It is antisymmetric in $\varphi$ and it captures phase alignment properties across scales. 
If $q=2$, then they are third order polynomial moments in $\varphi$. When $q=1$, these coefficients capture similar properties \cite{cheng2023scattering}.
The last interaction terms (\ref{Lambda40}) give the interactions between wavelet modulus coefficients at
different wavelet scales $2^j$ and $2^{j'}$ but at the
same scattering scales $2^{\ell}$ 
\begin{equation}
\label{Lambda4}
\Lambda_{4,j}(\data_{j-1}) = \Big(\sum_n |\data * \tilde \psi_{j,k'}|^q * \tilde  \psi_{\ell,k}(2^{j-1} n)~ |\data * \tilde \psi_{j',k''}|^q * \tilde \psi_{\ell,k}^*(2^{j-1} n) \Big)_{k,k',k''\leq Q, \ell\geq j' \geq j}.
\end{equation}
If $q=2$ then they define fourth order moments. When $q=1$ 
they have similar numerical properties \cite{cheng2023scattering}. 
The largest number of coefficients is within $\Lambda_{4,j}$.
The dimension of $\Psi_j$ is of the order of $Q^3 (j-J)^2$.

\paragraph{Further dimensionality reduction}

If $\varphi$ has a probability distribution which is invariant to rotations, then 
the energy remains invariant to the flipping operator which transforms $\varphi(n)$ into $\varphi(-n)$.
Since $\tilde \psi_{j,k}(-n) = \tilde \psi_{j,k}^*(n)$, one can verify, similarly to \cite{mallat2020phase,morel2023scale}, that
the imaginary parts of $\Lambda_{3,j}$ and $\Lambda_{4,j}$ change sign
when $\varphi(n)$ is transformed into $\varphi(-n)$.
It results that the interaction
coefficients of the energy $U_\theta$ over these imaginary parts are zero.
We impose this property 
in all numerical examples shown in this paper by eliminating the imaginary
parts of $\Lambda_{3,j}$ and $\Lambda_{4,j}$ from $\Psi_j$.
We also point out that for a symmetrical process, such that $\varphi$ and $-\varphi$ have the same distribution, one can as well discard $\Lambda_{3,j}$.

\section{Estimation of Conditional Energy parameters}
\label{app:est-cond-energy}

\subsection{Interaction energy estimation by score matching}
\label{app:score-matching}

Instead of minimizing directly each term 
$\expect[p_j]{KL (\bar p_{j} , \bar p_{\bar \theta_{j}})} $, and applying the algorithm 
in (\ref{likelihood-grad}), which
requires heavy calculations, we use a score matching algorithm which minimizes
a relative Fisher information.
It is calculated with a gradient relative to $\bar \data_{j}$, which eliminates the free energy $F_j (\data_j)$. The corresponding Fisher information is averaged with $p_j$:
\[
\expect[p_j] {\I (\bar p_{\bar \theta_j} , \bar p_j  )} =  \E_{p_{j}} \E_{\bar p_j}
\Big(\|\nabla_{\bar \data_{j}} \log \bar p_{j} ( \bar \data_{j}|\data_{j}) - \nabla_{\bar \data_{j}} \log \bar p_{\bar \theta_{j}}
( \bar \data_{j} | \data_{j})\|^2 \Big).
\]
According to \cite{hyvarinen2005estimation} and similarly to (\ref{eq:loss-score}),
since $\nabla_{\bar \data_{j}} \log \bar p_{\bar \theta_{j}} = \bar \theta_j\trans \nabla \Psi_j$, one can verify that it is equivalent to minimize
\begin{equation}
\label{eq:score}
\ell(\bar \theta_{j}) = \E_{p_{j-1}}
\Big(\frac 1 2 \|  \bar \theta_{j}\trans \nabla_{\bar \data_{j}} \Psi_j(\data_{j-1})
  \|^2 - 
 \bar \theta_{j}\trans\Delta_{\bar \data_{j}}  \Psi_j (\data_{j-1}) \Big) .
\end{equation}

This calculation can be done in parallel for all $j$.

Like \cref{theta-calcul}, a closed form for the minimizing $\bar\theta_j$ is given by 
\begin{equation}
\label{eq:exact_theta_bar}
\bar\theta_j = \bar M_j^{-1}\expect[p_{j-1}]{\Delta_{\bar\data_j} \Psi_j(\data_{j-1})} ~~ \text{with}~~ \bar M_j = \expect[p_{j-1}]{\nabla_{\bar\data_j} \Psi_j(\data_{j-1})\nabla_{\bar\data_j }\Psi_j(\data_{j-1})\trans} 
\end{equation}
For considered datasets, with the potential $\Psi_j$ from \cref{scalar-interaction3}, (\ref{eq:score}) is an ill-conditioned quadratic learning problem. We compute the matrix $\bar M_j$ from \cref{eq:exact_theta_bar}, and precondition $\theta$, by defining $\tilde\theta_j = (\bar M_j+\epsilon\Id)^{1/2}\bar\theta_j$, whose optimal value can be computed by minimizing the quadratic loss $$\tilde\ell(\tilde\theta) = 
\frac 1 2\tilde\theta_j\tilde\theta_j\trans
 -\tilde \theta_{j}\trans\expect[p_{j-1}]{(\bar M_j+\epsilon\Id)^{-1/2\text{T}}\Delta_{\bar \data_{j}}  \Psi_j (\data_{j-1})}~,$$ 
which has a condition number of $1$. This loss function is minimized using a single batch gradient descent. 
The same procedure applies to $\theta_J$, at the coarsest scale, for the quadratic loss defined in \cref{eq:loss-score,theta-calcul}. 
Due to the finite size of the datasets, the expectancies are replaced with their empirical estimations.
For applications in this paper, we used, for $\epsilon$, values not bigger than $10\%$ of the eigenvalues of $\bar M_j$.

\subsection{Free energy calculation}
\label{app:free}
To derive the Gibbs energy $ U_\theta(\varphi)$ in equation (\ref{modelwithfree}), we define and optimize a linear approximation $F_{\alpha_j}  = \alpha_j\trans\Phi_j$ of the
free energy $F_j$ defined by the normalization integral (\ref{normalisat-eq}) of $\bar p_{\bar \theta_{j-1}}$. Taking the derivative with respect to $\varphi_j$ in (\ref{normalisat-eq}) gives

\begin{equation}
     \nabla_{\varphi_j}F_j(\varphi_j) -  \expect[\bar p_{\bar \theta_{j}}(\bar \data_j|\data_j)]{ \bar \theta_{j}\trans \nabla_{\data_j}\Psi_j (\data_{j-1})} = 0 .
\end{equation}

Using the previous equation, it can be proven that the free energy $F_j$ is a minimizer of the quadratic function 
$f \mapsto \expect[\bar p_{\bar \theta_{j}}(\bar \data_j|\data_j) p_j(\data_j)]{\norm{
 \nabla_{\data_j} f (\data_j)  -  { \bar \theta_{j}\trans \nabla_{\data_j}\Psi_j (\data_{j-1})} }^2}$. The parameters $\alpha_j$ are optimized with the quadratic loss 
\begin{equation}
\label{cost-free}
\ell( \alpha_{j}) = \expect[\bar p_{\bar \theta_{j}}(\bar \data_j|\data_j) p_j(\data_j)]{\norm{
 \alpha_{j}\trans    \nabla_{\data_j} \Phi_j (\data_j)  -  { \bar \theta_{j}\trans \nabla_{\data_j}\Psi_j (\data_{j-1})} }^2}.
\end{equation}

This expected value is calculated with an empirical average over the $m$ samples $\data_j^{(i)}$ of $p_j$. For each $\data_j^{(i)}$ we know a sample
$\bar \data_j^{(i)}$ of $\bar p_{j} (\bar \data_j | \data^{(i)}_j )$. We modify this
sample with a MALA algorithm to obtain a
sample of $\bar p_{\bar\theta_{j}}(\bar \data_j|\data^{(i)}_j)$, which approximates
$\bar p_{j} (\bar \data_j | \data^{(i)}_j )$.

\section{Moment errors on $\vvarphi^4$}
\label{app:phi4}

\paragraph{Metric on moments}
\label{app:phi4-metric}
We evaluate 
the model errors from generated samples over a
sufficient set of moments.
For $\vvarphi^4$, these statistics are second order moments and the marginal distribution of $\data(n).$ This appendix defines a Kullback-divergence
error from these moments.

\begin{figure}
\centering
\includegraphics[width=0.65\textwidth]{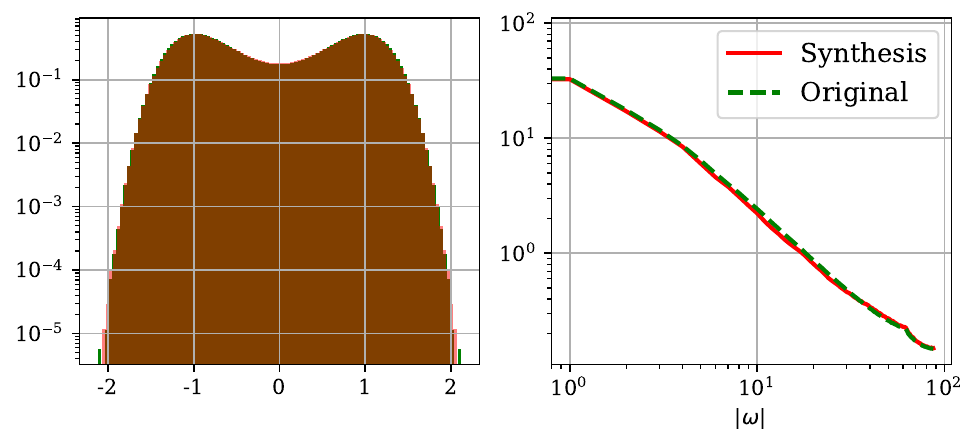}

~~~~~~~~(a)~~~~~~~~~~~~~~~~~~~~~~~~~~~~~~~~~~~ ~~~~~~~~~~~~~~~~~~(b)
\caption{(a): The marginal probability distribution of each $\data(n)$ is computed
  from samples of the $\vvarphi^4$ model at the phase transition. It is superimposed with the marginal probability distribution of samples of the hierarchic model in a Haar wavelet basis.  
  (b): Eigenvalues of 
  covariance matrices in the Fourier basis. They are computed from samples of the $\vvarphi^4$ model and with the Haar  hierarchic model, for images of size $d=128^2$.  The hierarchic model recovers precisely both types of moments in (a) and (b), which are sufficient statistics for $\vvarphi^4$.}
\label{moment-verification}
\end{figure}

The $\varphi^4$ model has a Gibbs energy
$\En (\data) = \theta\trans \Phi(\data)$,
where $\Phi$ is defined in (\ref{potential-equations}), with a single scalar potential $\Gamma(\data*\phi_\ell) = \Gamma(\data)$ for $\ell=0$. Sections \ref{sec:estim-gibbs} shows that $p$ is 
equal to the maximum entropy distribution $p_\theta$ such that
$\E_{p_{\theta}} (\Phi) =  \E_{p} (\Phi)$. These moments 
specified by 
the covariance $C$ of $p(\data)$ and by
the marginal distribution $\tilde p (t)$ of field values $t = \data(n)$. 

To evaluate the precision of a hierarchic model $p_\theta$, we
evaluate the covariance
$C_{\theta}$ and the marginal distribution $\tilde p_{\theta}(t)$, with a Monte Carlo average
over samples generated by this model. 
We compute the KL divergence between the maximum entropy
distributions having covariances a $C$ and a $C_\theta$.
Both distributions are gaussians and their KL divergence is
\[
\frac{1}{2 }\log(|C_{\theta}C^{-1}|)-\frac{d} {2} +\frac{1}{2}\Tr{C\,C_{\theta}^{-1}}  .
\]
The maximum entropy distributions having marginals $\tilde p$ and $\tilde p_{\theta}$ are random vectors with $d$ independent coordinates. Their KL divergence is
thus $d\, KL(\tilde p , \tilde p_\theta)$. 
Adding and renormalizing these KL divergences by the dimension $d$ gives an error:
\begin{equation}
  \label{moment-error}
e(p,p_\theta) = \frac{1}{2d }\log(|C_{\theta}C^{-1}|)-\frac{1} {2} +\frac{1}{2d}\Tr{C\,C_{\theta}^{-1}} + KL(\tilde p , \tilde p_\theta) . 
\end{equation}
Figure \ref{moment-verification} shows that the estimated marginal distributions $\tilde p$ 
and the eigenvalues of the covariance $C$ of $p$ are precisely approximated by the 
marginal $\tilde p_\theta$ and the covariance $C_\theta$
of the hierarchical model $p_\theta$ computed in a Haar basis.

\section{Proofs of Propositions and Theorems}

\subsection{Proof of Proposition \protect\ref{prop:covar-lower}}
\label{app:pointcarre}

The log-Sobolev equation (\ref{log-sobolev}) implies a Poincaré inequality 
\cite{ledoux2000geometry}
\[
\E [|f(\data) - \E(f(\data))|^2] \leq 2 c(p)\, \int \|\nabla f(\data)\|^2\, p(\varphi)d \data .
\]
This is proven, for regular enough functions $f$ , by applying the log-Sobolev inequality to the density $q = p + \epsilon (pf-\expect[p]{f})$, and letting $\epsilon$ go to zero.
Let $e_{\max}$ be a normalized eigenvector of the covariance of $p$ corresponding to the maximum eigenvalue $\mu_{\max}$ and $f(\data) = \langle e_{\max} , \data \rangle$.
The Poincaré inequality applied to $f$ gives
\[
\mu_{\max} \leq 2 c(p) \, \int \|e_{\max} \|^2\, p(\varphi)d \data = 2 c(p)~,
\]
which proves that $c(p) \geq \mu_{\max} / 2.$

\subsection{Proof of Proposition \protect\ref{th:translat}}
\label{app:translat}

To prove that $KL(p , \tilde q) \leq KL(p , q)$, we decompose
\[
KL(p , \tilde q) = \int p (\data) \log p (\data) d \data - \int p (\data) \log \tilde q (\data) d \data ~,
\]
with $\tilde q = \tilde {\cal Z}^{-1} e^{- \Ave_{j-1} \En}$.
Let us verify that the second term increases when $\tilde q$ is replaced by 
$q=  {\cal Z}^{-1} e^{- U}$. 

Since $p$ is translation invariant, a change of variable proves that
\begin{align}
- \int p (\data) \log \tilde q (\data) d \data = &|{\cal G}_{j-1}|^{-1} \int p (\data)  \sum_{\tau \in {\cal G}_{j-1}} \En ( T_\tau \data) d \data + \log \tilde {\cal Z} 
\nonumber\\
 & = \int p (\data)  \En(\data) d \data + \log \tilde {\cal Z} ~, 
 \label{intasdf}
\end{align}
where
\[
\tilde {\cal Z} = \int e^{- \Ave_{j-1} \En (\data)} d \data = \int \prod_{\tau \in {\cal G}_{j-1} } \big(e^{ -\En (T_\tau \data)}\big)^{|{\cal G}_{j-1}|^{-1}}
 \, d \data .
\]
We now prove that 
\begin{equation}
\label{cons-ineq}
\tilde {\cal Z} \leq {\cal Z} = \int 
e^{ - \En(\data)}
\, d \data ~,
\end{equation}
so that we can show with (\ref{intasdf}) that
\[
- \int p (\data) \log \tilde q (\data) d \data \leq - \int p (\data) q (\data) d \data ,
\]
which proves the theorem result (\ref{KL-diff}).

We prove (\ref{cons-ineq}) by iterating on the H\"older inequality 
which proves that
\begin{equation*}
\label{Holder}
 \left( \int \prod_{\tau \in {\cal G}_{j-1}} \big(e^{ -\En (T_\tau \data)}\big)^{|{\cal G}_{j-1}|^{-1}}
 \, d \data \right)^{|{\cal G}_{j-1}|} 
 \leq \left(
\int \prod_{\tau \in {\cal G}_{j-1} - \{\tau_1\}} \big(e^{ -\En( T_\tau \data)}\big)^{(|{\cal G}_{j-1}|-1)^{-1}}\, d \data \right)^{|{\cal G}_{j-1}|-1}
~ \int 
e^{ -\En (T_{\tau_1} \data)}
\, d \data  .
\end{equation*}
But $\int 
e^{ -\En (T_\tau \data)}
\, d \data = {\cal Z}$ so
reapplying $|{\cal G}_{j-1}| -1$ times the H\"older inequality to the integral to the power $|{\cal G}_{j-1}| -1$ proves that
$ (\tilde{\cal Z}) ^{|{\cal G}_{j-1}|} \leq ({\cal Z}) ^{|{\cal G}_{j-1}|}$ and hence
${\tilde{\cal Z}} \leq {\cal Z}$, which finishes the proof. $\Box$

\subsection{Proof of Theorem \protect\ref{pro:scalar-embedded-potential}}
\label{app:pro:scalar-embedded-potential}

At the coarsest scale $2^J$,  $\Phi_J$ is defined in (\ref{inter-param})  
by a convolutional operator whose kernel is written $K_J$ and a scalar
potential defined in (\ref{scalar-pot-vec}). One can thus
verify that 
\begin{equation}
\label{scalar-interaction00}
\theta_J = (K_J / 2 ,  \gamma_J )
~~
\mbox{and}~~
\Phi_J (\data_{J}) = 
\big(\data_J *  \data_J\trans ~,~
\Gamma(\data_{J}) \big).
\end{equation}
The potential $\Phi_J$ thus
satisfies (\ref{potential-equations}) for $j = J$.
 
To prove that the $\Phi_j$ define a hierarchic potential, we shall prove 
that for any $j \geq J$, if
\[
\Phi_j (\data_j) = \big(\data_j* \data_j\trans , \Gamma ( \data_j * \phi_\ell ) \big)_{J- j \geq \ell \geq 0}
\]
and if $\Psi_j$ is defined by (\ref{scalar-interaction}) and hence
\[
\Psi_{j}(\data_{j-1}) = \big( \bar \data_{j}  \bar \data_{j}\trans~,~
\bar \data_{j}  \data_{j}\trans~,~\Gamma (\data_{j-1}) \big) 
\]
then $\Ave_{j-1} (\Phi_j , \Psi_j)$ is a linear function 
of $\Phi_{j-1}$, with 
\[
\Phi_{j-1} (\data_{j-1}) = \big(\data_{j-1}* \data_{j-1}\trans , \Gamma ( \data_{j-1} * \phi_\ell ) \big)_{J- j+1 \geq \ell \geq 0} .
\]

Let us first show that $\Ave_{j-1} (\Psi_j)$
is a linear function of $\Phi_{j-1}$. 
We consider the first two
terms of $\Psi_j$.
Since $\bar \varphi_{j} = \bar G \varphi_{j-1}$ and
$\varphi_{j} =  G \varphi_{j-1}$, it results that
$\bar \varphi_{j} \bar \varphi_j\trans$ 
and
$\bar \varphi_{j} \varphi_j\trans$ 
are linear functions of
 $\varphi_{j-1} \varphi_{j-1}\trans$.
Since $\varphi_{j-1} \varphi_{j-1}\trans = (\varphi_{j-1} (n) \varphi_{j-1}(n') )_{n,n'}$ it results that 
\begin{align*}
\Ave_{j-1} (\varphi_{j-1} \varphi_{j-1}\trans) &= 
|{\cal G}_{j-1}|^{-1}
\Big(\sum_{\tau \in {\cal G}_{j-1}} \varphi_{j-1}(n-\tau) \varphi_{j-1}(n'-\tau) \Big)_{n,n'} \\
&= |{\cal G}_{j-1}|^{-1}\big(\varphi_{j-1} * \varphi_{j-1}\trans (n'-n) \big)_{n,n'}.
\end{align*}
It results that applying $\Ave_{j-1}$ to the first two terms
of $\Psi_{j-1}$ is a linear function of $\Phi_{j-1}$.
For the third term of $\Ave_{j-1}(\Psi_j)$ we have
\[
\Ave_{j-1} \Gamma (\varphi_{j-1}) = \Gamma(\varphi_{j-1}) = 
\Gamma(\varphi_{j-1} * \phi_0) ~,
\]
because $\phi_0 = \delta$, which a linear function of the
second term of $\Psi_j$ which includes this term. It concludes the
proof that $\Ave_{j-1} (\Psi_j)$
is a linear function of $\Phi_{j-1}$.

Let us now prove that $\Ave_{j-1} (\Phi_j)$
is a linear function of $\Phi_{j-1}$. For the first term in  $\Ave_{j-1} (\Phi_j)$,  equation (\ref{insdfdsf}), which states that $\data_j(n)= \data_{j-1} * g(2n)$, gives 
\begin{align*}
    \data_j* \data_j\trans &= \Big(\sum_{\tau\in{\cal G}_j} \varphi_j(-\tau)\varphi_j(n-\tau) \Big)_n \\ &=  \Big(\sum_{\tau\in{\cal G}_j} (\varphi_{j-1}*g)(-2\tau)(\varphi_{j-1}*g)(2n-2\tau) \Big)_n . 
\end{align*}
Averaging over all the translations on the grid ${\cal G}_{j-1}$ eliminates the subsampling in the previous sum, and gives

\begin{align*}
\Ave_{j-1} (\data_j * \data_j\trans)  &= 
\frac 1 4 \big(( \data_{j-1} * g )*( \data_{j-1}\trans * g \trans )(2n)\big)_{n}\\ &= \frac 1 4 
\big((g * g\trans) * (\data_{j-1} * \data_{j-1} \trans)(2n)\big)_{n} ~,
\end{align*}
where $g\trans(n) = g(-n)$. $\Ave_{j-1} (\data_j * \data_j\trans)$ is thus a linear transformation of the
first term of $\Phi_{j-1}(\varphi_{j-1})$.

To compute the second terms of $\Ave_{j-1}(\Phi_j)$, we use
(\ref{phi-rec}) which shows that
\begin{equation}
\label{scalin-relasn}
\varphi_j * \phi_\ell (n) = \varphi_{j-1} * \phi_{\ell+1} (2n) .
\end{equation}
Moreover $\Gamma$
computes a sum on the grid ${\cal G}_j$ of pointwise transformations $\rho_k$ which commute with
translations. It results that the averaging of all translations 
on the grid ${\cal G}_{j-1}$ gives
\begin{align*}
\Ave_{j-1}(\Gamma(\data_j * \phi_\ell) ) &= |{\cal G}_{j-1}|^{-1} 
\Big( \sum_{\tau \in {\cal G}_{j-1} , n \in {\cal G}_{j}} \rho_k ( \varphi_{j-1} * \phi_{\ell+1} (2n-\tau)) \Big)_k \\ &=
\frac 1 4 
\Big( \sum_{n \in {\cal G}_{j-1}} \rho_k ( \varphi_{j-1} * \phi_{\ell+1} (n) \Big)_k\\ &= \frac 1 4 \Gamma(\data_{j-1} * \phi_{\ell+1}) . 
\end{align*}
This proves that the second terms of $\Ave_{j-1}(\Phi_j)$ is a linear
function of the second term of $\Phi_{j-1}$.
This finishes the proof that the $\Phi_j$ define a 
stationary hierarchic potential.

Observe that the scalar potential of each interaction energy $\Psi_j$ creates
a scalar potential at each scale, which appears in (\ref{potential-equations}). 
If we approximate the free energies $F_j$
by $\alpha_j\trans \Phi_j$ then Proposition \ref{prop:couplingflow} 
defines stationary energy models $U_{\theta_j} = \theta_j\trans \Phi_j$ at each scale $2^j$.
The coupling vector $\theta_{j-1}$ is 
calculated from $\theta_j$ and $(\bar \theta_j,\alpha_j)$ with the
operator $Q_j$.
It defines a fine scale stationary energy model 
\[
U_\theta (\varphi) = \theta_0\trans \Phi_0 (\varphi) =
(K, \gamma_j)_{j \geq J}\trans \, \big(\varphi * \varphi\trans , \Gamma ( \data * \phi_ j) \big)_{J \geq j\geq 0} ~,
\]
and hence
\[
U_\theta (\varphi) = \frac 1 2 \varphi\trans K \varphi + \sum_{j=0}^J \gamma_j\trans \Gamma (\varphi * \varphi_j)~,
\]
which finishes the theorem proof. 

\subsection{Proof of Theorem \protect \ref{prop:scat-covr-embeddge}}
\label{app:th:scat-covr-embeddge}

To verify that the $\Phi_j$ define a hierarchic potential we shall prove 
that for any $j \geq J$, if
\[
\Phi_j (\data_j) = \big(R(\varphi_j)* R(\varphi_j)\trans , \Gamma ( \data_j * \phi_\ell ) \big)_{J- j \geq \ell \geq 0} ~~\mbox{with}~~  R(\varphi_j) = (\varphi_j , |\varphi * \tilde \psi_{j',k'}(2^{j} n)|^q)_{j' > j,k'} ~,
\]
and
\[
\Psi_j = \big(S_{j} S_{j'}\trans, \Gamma \big)_{J+1 \geq j' \geq j} ~~\mbox{with}~~S_{j'}(\varphi_{j-1}) = \Big( \varphi * \tilde \psi_{j',k'}(2^{j-1} n)   \,,\,  |  \data * \tilde \psi_{j',k'}|^q * \psi_{\ell,k} (2^{j-1} n) \Big)_{\ell \geq j' , k,k'\leq Q,n} ~,
\]
then $\Ave_{j-1} (\Phi_j )$ and $\Ave_{j-1}(\Psi_j)$ are linear functions 
of $\Phi_{j-1}$, for
\[
\Phi_{j-1} (\data_{j-1}) = \big(R(\varphi_{j-1}) * R(\varphi_{j-1})\trans , \Gamma ( \data_{j-1} * \phi_\ell ) \big)_{J- j+1 \geq \ell \geq 0} .
\]

From the proof of Theorem \ref{pro:scalar-embedded-potential},
the $\Gamma$-terms of $\Ave_{j-1} (\Phi_j )$ and $\Ave_{j-1}(\Psi_j)$ are linear functions 
of the $\Gamma$-terms of $\Phi_{j-1}$. We concentrate the first terms in $\Ave_{j-1}(\Psi_j)$ and $\Ave_{j-1}(\Phi_j)$.

For $\Ave_{j-1}(\Psi_j)$, consider $\Ave_{j-1}(S_{j}  S_{j'}\trans)$. For any $j'\geq j$, the scattering vector $S_{j'}$ is a linear function of $R(\varphi_{j-1}) = (\varphi_{j-1} , |\varphi * \tilde \psi_{j',k'} (2^{j-1} n)|^q)_{j' > 0,k'\leq Q}$. Nonlinear terms in $S_j'$ are included in $R(\varphi_{j-1})$, and, from (\ref{psi-rec22}), $\varphi * \tilde \psi_{j',k'}(2^{j-1}n)$, is a linear function of $\varphi_{j-1}$. 

It results that the terms
$\Ave_{j-1}(S_{j}  S_{j'}\trans)$ 
are linear functions of 
$\Ave_{j-1}(R_{j-1} R_{j-1}\trans)$. In the same way that we proved for Theorem \ref{pro:scalar-embedded-potential} that 
$\Ave_{j-1}(\varphi_{j-1} \varphi_{j-1}\trans)$ is a linear function
of $\varphi_{j-1} * \varphi_{j-1}\trans$, we prove that
$\Ave_{j-1}(R(\varphi_{j-1}) R(\varphi_{j-1})\trans)$ is a linear function
of $R(\varphi_{j-1}) * R(\varphi_{j-1})\trans$. The term $\Ave_{j-1}(S_{j}  S_{j'}\trans)$ is therefore a linear function of the first
term of $\Phi_{j-1}$.

Let us now prove that the first term of
$\Ave_{j-1} (\Phi_j)$, and hence $\Ave_{j-1} (R(\varphi_j) * R(\varphi_j)\trans)$,
is a linear function of $R(\varphi_{j-1}) * R(\varphi_{j-1})\trans$. 

$R(\varphi_j)$ is deduced from $R(\varphi_{j-1})$ using, the filter $G$ on $\varphi_{j-1}$, and by subsampling (or discarding) the nonlinear terms. 

$\Ave_{j-1} (R(\varphi_j) R(\varphi_j)\trans)$,
is a linear function of $\Ave_{j-1} (R(\varphi_{j-1}) R(\varphi_{j-1})\trans)$, which is itself a linear transform of $R(\varphi_{j-1})*R(\varphi_{j-1})\trans$.

This finishes the proof that the $\Phi_j$ define a 
stationary hierarchic potential.

Proposition \ref{prop:couplingflow} implies that approximating $F_j$
by $\alpha_j\trans \Phi_j$ yields a coupling flow equation (\ref{prop-eq3}). It computes 
$\theta_{j-1}$ from $\theta_j$ and $(\bar \theta_j,\alpha_j)$ with the linear
operator $Q_j$. It defines
\[
U_\theta (\varphi) = \theta_0\trans \Phi_0 (\varphi) =
\theta_0\trans \big(R(\varphi) * R(\varphi)\trans , \Gamma ( \data * \phi_\ell ) \big)_{J\geq j \geq 0} .
\]
Since
$R(\varphi) = (\varphi, |W \varphi|^q )$, 
for $\theta_0 = (K,L,M, \gamma_j )_{j \geq J}$ 
where $(K,L,M)$ are convolutional operators, this can be rewritten
\[
U_\theta (\varphi) = 
\frac 1 2 \data\trans K \data + 
\data\trans L\, (|W \data|^q) + 
\frac 1 2 (|W \data|^q)\trans M \,(|W \data|^q)  + 
 \sum_{j=0}^J   \gamma_j\trans \Gamma(\data * \phi_j),
 \]
which proves (\ref{scat-enn-mod}).

\section{Estimation Algorithms}

\subsection{Normalized Autocorrelation Relaxation Time of a Langevin Diffusion}
\label{sec:mixingtimelangevin}

This appendix specifies the calculation of normalized
relaxation time of Langevin diffusions in hierarchic models.

\paragraph{Conditional auto-correlation relaxation time} As in \citep{marchand_wavelet_2022,guth2023conditionally},
we compute the auto-correlation relaxation time (\ref{relax-time}) to sample the conditional probabilities of a hierarchic model with an unadjusted Langevin diffusion.
This conditional probability has an energy $\bar U_{\bar \theta_j}$.
The unadjusted Langevin diffusion of $\bar \varphi_j$
is numerically computed
with an Euler-Maruyama discretization with a time interval
$\delta_j$. This time interval must be smaller than the
inverse of the Lipschitz constant of $\nabla_{\bar\data_j}\bar U_{\bar\theta_j}$ \cite{vempala2022rapid}. This Lipschitz constant is evaluated by estimating the supremum of the eigenvalues of the Hessian $\nabla_{\bar\data_j}^2\bar U_{\bar\theta_j} (\varphi_{j-1})$ over typical realizations
$\varphi_{j-1}$. We set $\delta_j$ to be a fixed fraction of the inverse of this supremum.

Let $\bar \varphi_{j,\alpha}$ be the wavelet coefficients of obtained by the Langevin diffusion, after $\alpha$ steps at intervals $\delta_j$, thus corresponding to a time $t = \alpha\, \delta_j$. 
A normalized conditional auto-correlation at $\alpha$ is
\begin{equation}
    \bar A_{j}(\alpha| \data_j) = \frac{\expect[\bar p_{\bar\theta_j}(\cdot | \data_j)]{\big(\bar\data_{j,\alpha}-\expect[\bar p_{\bar\theta_j}(\cdot | \data_j)]{\bar\data_{j}}\big)\big( \bar\data_{j,0}-\expect[\bar p_{\bar\theta_j}(\cdot | \data_j)]{\bar\data_{j}}  \big) } }{\expect[\bar p_{\bar\theta_j}(\cdot | \data_j)]{\big(\bar\data_{j}-\expect[\bar p_{\bar\theta_j}(\cdot | \data_j)]{\bar\data_{j}}\big)^2 }}~,
\end{equation}
This auto-correlation is averaged over the distribution of $\varphi_j$
\begin{equation}
    A_{j}(\alpha) = \expect[p_j]{\bar A_j (\alpha | \data_j)}.
\end{equation}
The exponential decay $A_j(\alpha)$ is measured
by the normalized autocorrelation relaxation time $\bar \tau_j$, 
\begin{equation}
    A_{j}(\alpha) \approx A_j(0) \exp\paren{-\frac{\alpha}{\bar \tau_j}}.
\end{equation}
This normalized relaxation time is the number of Langevin iterations needed to reach a fixed precision.

\paragraph{Comments on MALA}
The rejection step of the MALA algorithm modifies the value
of the relaxation time. Whereas Langevin and MCMC algorithm have hierarchical autocorrelation relaxation times which do not depend upon the system size $d$, the MALA algorithm produces a  hierarchical relaxation
time which has a slow growth as a function of $d$,
highlighted in \citep{guth2023conditionally}. 
Indeed, MALA relies on a global acceptance step, which does not take advantage of the locality of interactions, as opposed to an MCMC sampling or an unadjusted Langevin diffusion
\cite{marchand_wavelet_2022,montanari2006rigorous}.
For example, MALA has a mixing time that grows with the dimension $d$, for strongly log-concave distributions having a
log-Sobolev constant which does not depend on $d$ \citep{chewi2021analysis,li2022sqrtd,wu2022minimax}, whereas
it is not the case for an MCMC sampling or a 
Langevin unadjusted diffusions.
However, MALA is the fastest algorithm to compute the
hierarchic sampling for the image sizes considered in this paper \cite{guth2023conditionally}. It is used for the generation results but not to evaluate the relaxation time in order to avoid issues related to the rejection-approval step.

\subsection{Estimation of the Bi-spectrum}
\label{app:bi-tri}
We describe the calculation of the bi-spectrum in
Figure \ref{plot:QuijoteTurbuStats}. The bi-spectrum is defined as the expected value Fourier transform of the 3 points correlation function. We use a regularized estimation of these statistics for rotationally invariant probability 
distributions, described in \citep{cheng2023scattering}.
It performs a frequency averaging in thin frequency annuli, evenly spaced in frequency log-scale, that we write ${\cal A}_k$. For the Figure \ref{plot:QuijoteTurbuStats}(e), for images of size $d=128^2$, the frequency plane is decomposed into $9$ frequency annuli ${\cal A}_k$, which select frequencies $\omega$ such that $ 2^{-7k/9}\pi \leq |\omega| \leq 2^{-7(k+1)/9}\pi $.

We consider probability distributions which are invariant to rotations. The power spectrum is thus estimated with an average over frequencies $\om$ in each annulus ${\cal A}_k$ and over multiple samples $\varphi^{(i)}$:
\begin{equation}
    P(k) = \Ave_{i, \omega\in {\cal A}_k} |\hat \data^{(i)}(\omega)|^2 .
\end{equation}

For $(k_1,k_2,k_3)$, 
a regularized normalized bi-spectrum is defined as a sum
over all $(\omega_1,\omega_2,\omega_3)$ 
from 3 frequencies $\omega_j \in {\cal A}_{k_j}$ which sum to zero $\sum_{j=1}^3 \om_j = 0$, 
by multiplying the Fourier transforms
at these frequencies of multiple samples $\varphi^{(i)}$:

\begin{equation}
B(k_1,k_2,k_3) = \Ave_{i, \omega_{j}\in {\cal A}_{k_j},\sum_{j=1}^3\omega_{j} = 0}\Big( \frac{\hat\data(\omega^{(i)}_{1})\,\hat\data(\omega^{(i)}_{2})\,\hat\data^{(i)}(\omega_{3})} {\sqrt{P(k_i)\,P(k_2)\,P(k_3)}}\Big).
\end{equation}

\bibliography{References}

\end{document}